\documentclass[twoside,11pt]{article}
\usepackage{jair, theapa, rawfonts}

\usepackage{latexsym, amsmath, amssymb, amsfonts, mathrsfs}
\usepackage{xspace}
\usepackage[T1]{fontenc}
\usepackage{float}
\usepackage{wrapfig}
\usepackage{color}
\usepackage[pdftex]{graphicx}
\usepackage{tikz}
\usepackage{enumitem}
\usepackage[ruled,linesnumbered,noresetcount,vlined]{algorithm2e}
\usepackage{todonotes}
\usepackage{picinpar}
\usepackage{multirow}
\usepackage{subcaption}
\usepackage[normalem]{ulem}
\usepackage{amsthm}

\usepackage{listings}
\lstdefinelanguage{PDDL}
{
  sensitive=false,    
  morecomment=[l]{;}, 
  alsoletter={:,-},   
  morekeywords={
    define,domain,problem,not,and,or,when,forall,exists,either,
    :domain,:requirements,:types,:objects,:constants,
    :predicates,:action,:parameters,:precondition,:effect,
    :fluents,:primary-effect,:side-effect,:init,:goal,
    :strips,:adl,:equality,:typing,:conditional-effects,
    :negative-preconditions,:disjunctive-preconditions,
    :existential-preconditions,:universal-preconditions,:quantified-preconditions,
    :functions,assign,increase,decrease,scale-up,scale-down,
    :metric,minimize,maximize,
    :durative-actions,:duration-inequalities,:continuous-effects,
    :durative-action,:duration,:condition,:probabilistic-effects
  }
}

\allowdisplaybreaks

\newcommand{\M}{\mathcal{M}}

\newcommand{\C}{\mathcal{C}}
 \newcommand{\B}{\mathcal{B}}
\newcommand{\Bh}{\mathcal{B}_h}

\newcommand{\fml}[1]{\mathcal{#1}}

\newcommand{\Ba}{\mathcal{B}_\alpha}

\newcommand{\KB}{\mathrm{KB}}
\newcommand{\EP}{\mathrm{EP}}
\newcommand{\KBa}{\mathrm{KB}_\alpha}

\newcommand{\I}{\mathbb{I}}
\newcommand{\R}{\mathbb{R}}

\newcommand{\pe}{\tilde{\epsilon}}
\newcommand{\pE}{\tilde{\fml{E}}}
\newcommand{\e}{\epsilon}

\newcommand{\beitemize}{\begin{list}{$\bullet$}{\topsep=1.5pt \parsep=0pt \itemsep=1pt \leftmargin=1em }} 
\newcommand{\enitemize}{\end{list}}

\newcommand{\beenumerate}{\hspace{-0.5in} \begin{enumerate}\topsep=1pt \parsep=0pt \itemsep=-3pt} \newcommand{\enenumerate}{\end{enumerate}}

\newcommand{\belist}{\begin{list}{$\bullet$}{\topsep=1.5pt \parsep=0.5pt \itemsep=1pt \leftmargin=2.25em \labelwidth=1.0em \labelsep=0.5em \partopsep=1.5pt}} 
\newcommand{\enlist}{\end{list}}

\newtheorem{theorem}{{\bf Theorem}}
\newtheorem{lemma}{{\bf Lemma}}
\newtheorem{corollary}{{\bf Corollary}}
\newtheorem{definition}{{\bf Definition}}

\newtheorem{example}{{\bf Example}}
\newtheorem{proposition}{{\bf Proposition}}

\newboolean{includeMemo}
\setboolean{includeMemo}{true} 

\newcommand{\memoside}[1]{\ifthenelse{\boolean{includeMemo}}{\todo[caption={},color=green!20!]{{\footnotesize #1}}}}
\newcommand{\memo}[1]{\ifthenelse{\boolean{includeMemo}}{\todo[inline,caption={},color=green!20!]{#1}}}
\newcommand{\memob}[1]{\ifthenelse{\boolean{includeMemo}}{\todo[inline,caption={},color=blue!20!]{#1}}}

\newcommand{\xhdr}[1]{\vspace{5pt}\noindent\textbf{#1 }}
\newcommand{\ignore}[1]{}

\newcommand{\squishlist}{
\begin{list}{{{\small{$\bullet$}}}}
{\setlength{\itemsep}{3pt}      
\setlength{\parsep}{3pt}
\setlength{\topsep}{3pt}       
\setlength{\partopsep}{3pt}
\setlength{\leftmargin}{1em} 
\setlength{\labelwidth}{1em}
\setlength{\labelsep}{0.5em} } }
\newcommand{\squishend}{  \end{list}}

\newcommand{\squishenum}{
\begin{list}{$\bullet$}{ 
    \setlength{\itemsep}{1pt}
    \setlength{\parsep}{0pt}
    \setlength{\topsep}{1.5pt}
    \setlength{\partopsep}{0pt}
    \setlength{\leftmargin}{2em}
    \setlength{\labelwidth}{1.5em}
    \setlength{\labelsep}{0.5em} } }

\newcommand{\citet}[1]{\citeauthor{#1}~\citeyear{#1}}
\newcommand{\citep}{\cite}

\usepackage{ulem}
\renewcommand{\underline}{\uline}

\usepackage[hyphens]{url}  
\ShortHeadings{}
{Vasileiou, Yeoh, Previti, \& Son}

\begin{document}

\title{On Generating Monolithic and Model Reconciling Explanations in Probabilistic Scenarios\thanks{Paper has been accepted at JAIR.}}


\author{\name Stylianos Loukas Vasileiou \email stelios@nmsu.edu \\ 
       \addr New Mexico State University \\
       Las Cruces, United States 
       \AND
       \name William Yeoh \email wyeoh@wustl.edu \\
       \addr Washington University in St.~Louis \\
       Saint Louis, United States
       \AND
       \name Alessandro Previti \email alessandro.previti@ericsson.com \\
       \addr Ericsson Research \\ 
       Stockholm, Sweden
       \AND
       \name Tran Cao Son \email stran@nmsu.edu \\
       \addr New Mexico State University \\ 
       Las Cruces, United States
       }


\maketitle

\begin{abstract}

Explanation generation frameworks aim to make AI systems' decisions transparent and understandable to human users. However, generating explanations in uncertain environments characterized by incomplete information and probabilistic models remains a significant challenge. In this paper, we propose a novel framework for generating \textit{probabilistic monolithic explanations} and \textit{model reconciling explanations}. Monolithic explanations provide self-contained reasons for an explanandum without considering the agent receiving the explanation, while model reconciling explanations account for the knowledge of the agent receiving the explanation. For monolithic explanations, our approach integrates uncertainty by utilizing probabilistic logic to increase the probability of the explanandum. For model reconciling explanations, we propose a framework that extends the logic-based variant of the model reconciliation problem to account for probabilistic human models, where the goal is to find explanations that increase the probability of the explanandum while minimizing conflicts between the explanation and the probabilistic human model. We introduce \textit{explanatory gain} and \textit{explanatory power} as quantitative metrics to assess the quality of these explanations. Further, we present  algorithms that exploit the duality between minimal correction sets and minimal unsatisfiable sets to efficiently compute both types of explanations in probabilistic contexts. Extensive experimental evaluations on various benchmarks demonstrate the effectiveness and scalability of our approach in generating explanations under uncertainty.

\end{abstract}

\section{Introduction}
\label{sec:intro}

The rapid integration of artificial intelligence (AI) into critical and everyday applications has magnified the importance of not just achieving high-performance metrics but also ensuring that AI decisions are transparent, interpretable, and, above all, trustworthy. This imperative has given rise to the field of explainable AI (XAI), which seeks to make AI systems' workings comprehensible to their human users \cite{gunning2019xai}. XAI endeavors to demystify the often opaque processes of AI, providing insights into the reasoning behind decisions and actions. This transparency is not just a matter of ethical AI design but a practical necessity for enhancing user trust, facilitating user decision-making, and ensuring the accountability of AI systems.

In the domain of machine learning (ML), significant strides have been made towards enhancing the explainability of algorithms. Researchers have sought to categorize ML algorithms according to various dimensions of explainability \cite{Guidotti:2018:SME:3271482.3236009}, improve the transparency of existing algorithms \cite{NIPS2018_8003,petkovic2018improving}, and even propose new algorithms that balance accuracy with increased explainability \cite{dong2017improving,gilpin2018explaining}. These efforts underscore a growing recognition within the ML community of the critical role that explainability plays in the deployment of AI systems \cite{belle2021principles}.

Parallel to advancements in ML, the automated planning community has adopted a focused approach to generating explanations for plans produced by AI (planning) agents, which led to the inception of explainable AI planning (XAIP) \cite{Fox2017ExplainableP}. Predominantly, XAIP research focuses on the explanation generation problem, which involves identifying explanations for plans that, when conveyed to human users, help them understand and accept the agent's proposed actions \cite{kambhampati1990classification,langley2016explainable}. A noteworthy direction within this space is the \textit{model reconciliation problem} (MRP) \cite{chakraborti2017plan,sreedharan2018handling,nguyen2020explainable,vas21,vasLMRP}, aimed at aligning a user's model with that of an AI agent through the provision of explanations, especially when discrepancies in their understanding of a planning problem lead to confusion or misinterpretation of the agent's decisions. We will refer to these explanations as \textit{model reconciling explanations}. In contrast, we refer to explanations that do not account for the user's model as \textit{monolithic explanations}.

Although MRP tackles essential facets of explainability, it often operates under the presumption that the AI agent has a deterministic grasp of the human model---a scenario that may not always align with the complexities of real-world interactions characterized by uncertainty about human knowledge. Indeed, this gap highlights a general challenge in explanation generation from AI agents: Their operation within realms of incomplete information and probabilistic decision-making models. Traditional explanation methods, which rely on deterministic knowledge, falter under these conditions, unable to represent the uncertain nature of the agent's knowledge adequately.

To bridge this gap in explainability, the first part of the paper presents a framework for explanation generation under uncertainty. In particular, we use (propositional) probabilistic logic as our underlying mechanism for modeling uncertainty, and introduce the notions of \textit{probabilistic monolithic explanation} and \textit{probabilistic model reconciling explanation}. For the former, given an agent's belief base $\Ba$ (a weighted knowledge base\footnote{A weighted knowledge base, essentially, induces a probability distribution over the variables of the agent's language \cite{richardson2006markov}.}), and an explanandum $\varphi$ (a formula to be explained), the goal is to find a \textit{probabilistic monolithic explanation} such that the probability of the explanandum being true is increased. Moreover, we extend our work on \textit{logic-based model reconciliation problems} (L-MRPs) \cite{son2021model,vas21,vasLMRP} to scenarios in which the human user's model is uncertain. Specifically, given a knowledge base $\KBa$ of an agent, an explanandum $\varphi$ entailed by $\KBa$, and a human belief base $\Bh$, the goal is to find a \textit{probabilistic model reconciling explanation} such that the explanandum's probability is increased while the probability of conflicts between the explanation and $\Bh$ is decreased. To measure the quality of such explanations, we define the concepts of \textit{explanatory gain} and \textit{explanatory power}, aimed at quantifying the effectiveness of these explanations with respect to the explanandum.

In the second part of the paper, we describe algorithms for computing both types of explanations, where we leverage the duality between \textit{minimal correction sets} (MCSes) and \textit{minimal unsatisfiable sets} (MUSes). These algorithms, adapted from our prior work \cite{vas21}, integrate a \textit{weighted maximum satisfiability} procedure for computing probabilities. Our experimental evaluation across various benchmarks highlights the practicality and applicability of our algorithms, and demonstrate their capabilities to generate probabilistic explanations within a propositional logic framework.

In summary, our contributions are the following:

\squishlist
    \item We propose a novel framework for generating probabilistic monolithic explanations and probabilistic model reconciling explanations. Central to our framework are the concepts of \textit{explanatory gain} and \textit{explanatory power}, metrics designed to quantitatively assess the quality and effectiveness of probabilistic explanations. 


    \item We describe algorithms for computing both types of probabilistic explanations, using the duality of MCSes and MUSes together with a weighted maximum satisfiability process. Through a series of benchmark evaluations, we demonstrate the efficacy of our proposed algorithms in generating explanations.

\squishend

The paper is structured as follows. In Section~\ref{sec:motivation}, we motivate the use and applicability of logic in explainability, and discuss possible sources of uncertainty in explanation generation. In Section~\ref{sec:back}, we provide the background knowledge needed, and in Section~\ref{sec:example} we describe a motivating application that serves as a running example. In Section~\ref{sec:pexpl-gen}, we present our explanation generation framework for monolithic and model reconciliation probabilistic explanations.~We present algorithms for computing explanations in Section~\ref{sec:comp}, and experimentally evaluate them on a set of benchmarks in Section~\ref{sec:exper}.~We then discuss related work in Section~\ref{sec:related}, address the assumptions, limitations and future extensions of our framework in Section~\ref{sec:discussion}, and finally conclude the paper in Section~\ref{sec:conclusion}.

\section{A Logic-based Perspective on Explainable Decision-Making}
\label{sec:motivation}

In this work, we take the following perspective: \textit{Logic-based frameworks can serve as an explanatory representational layer for AI systems, enabling the generation of rigorous and flexible explanations across diverse problem domains by capturing the system's decisions in a formal logical language that supports inference and reasoning.}

According to this perspective, formal logic provides a robust foundation for creating explanatory mechanisms by serving as an intermediate representational layer between AI systems and explanation generation processes. Some of the reasons that justify this perspective are as follows:

\squishlist
\item \textbf{Structured Semantics}: Logic has well-defined compositional semantic functions that compute the meaning of a compound as a function of its constituents' meanings. This composition is invertible, allowing us to trace back from conclusions to premises—essentially performing inference in reverse. This property is fundamental to explanation, as it enables us to identify precisely which components contributed to a particular decision or output.

\item \textbf{Expressivity \& Scrutability}: Logic can represent complex knowledge and reasoning processes in a form that can be traced and scrutinized. Depending on the formal language of choice, it allows for encoding rich, relational information about the world, including classes, hierarchies, causal relationships, and quantified statements. For instance, in planning domains, logical representations can express not only states and actions but also the causal links between them, making it possible to explain why certain actions were selected in a plan. Moreover, logic-based systems enable us to examine their internal properties, both through internal verification techniques and external dialogues. This scrutability is crucial for building trust in AI systems, as it allows for thorough validation of the reasoning process, ensuring that explanations are not just plausible but provably correct within the system's logical framework.

\item \textbf{Adaptability}: Logic-based systems can be extended with new knowledge and can integrate with various AI paradigms. This means that explanatory frameworks built on logic can evolve over time, incorporating new concepts, rules, or meta-level reasoning principles as needed. This adaptability is important for explanation systems that must operate across diverse domains or interact with different types of users.

\item \textbf{Uncertainty Representation}: While classical logic is deterministic, probabilistic extensions (which we use in this work) allow for representing and reasoning about uncertainty. This is important in real-world scenarios where AI systems must operate with incomplete information or where multiple interpretations of data are possible. Probabilistic logic enables explanations that convey not just what the system ``believes'', but also its degree of confidence in those beliefs.
\squishend

For more about logic and its suitability in explainability as well as the current AI landscape, please see some good arguments presented by \citet{belle2017logic}, \citet{mocanu2023knowledge}, and \citet{bellerelevance}.

\begin{figure}[t]
    \centering
    \includegraphics[width=.85\textwidth]{figures/explainability-layer.pdf}
    \caption[The logic-based explainability layer]{The Logic-based Explainability Layer. Given a decision $d$ from an AI model, a \textit{monolithic explanation} is generated with respect to the AI system's knowledge base ($\KBa$), while a \textit{model reconciling explanation} is generated with respect to the AI system's ($\KBa$) as well as the human user's knowledge base ($\KB_h$).}
    \label{fig:concept}
\end{figure}

Now, as illustrated in Figure~\ref{fig:concept}, our approach positions logic as an intermediate explanatory layer between AI systems and human users. The logic-based explainability layer functions as an abstraction mechanism that captures the essential reasoning behind an AI system's decisions in a formal logical representation, regardless of the underlying implementation details of the AI system itself. This layer serves two critical functions. First, it translates the internal decision processes of potentially opaque AI systems into a logical formalism that can be systematically analyzed. Second, it enables the generation of explanations through formal reasoning over this logical representation.

We will distinguish between two notions of explanation: \textit{monolithic explanations}, which are generated with respect to the AI system's knowledge base ($\KBa$) alone, focusing on providing a self-contained justification for the system's decision; and \textit{model reconciling explanations}, which are generated with respect to both the AI system's knowledge base ($\KBa$) and the human user's knowledge base ($\KB_h$), explicitly addressing the discrepancies between the two models that may lead to confusion or misunderstanding. The explainability layer thus provides a unified framework for generating different types of explanations while maintaining a clear separation between the underlying AI system and the explanation mechanism.

The effectiveness of the logic-based explainability framework depends crucially on the ability to encode the problem domains of interest in logical formalisms. In general, it can be used for problems that can be represented in a logical language for which satisfiability of sets is feasible. Fortunately, many problem domains admit logical encodings, such as planning problems \cite{kautz1996encoding,cashmore2012planning,cashmore2020planning}, scheduling problems~\cite{crawford1994experimental,ansotegui2011satisfiability,demirovic2019modeling}, argumentation problems~\cite{besnard2001logic,prakken2006formal,besnard2010mus}, as well as machine learning problems~\cite{shrotri2022constraint,marques2022delivering,izza2023computing}. Indeed, logic-based explanation generation approaches have been successfully employed in numerous applications~\cite{vas21,vasLMRP,vasileioua2023lasp,rago2023interactive,marques2022delivering}.

\subsection{The Source of Uncertainty in Explanations}

Traditional logic-based approaches have typically operated within deterministic settings. They assume that both the AI system and the human user have definite, certain knowledge about the domain. This assumption, however, fails to capture the reality of many real-world scenarios where uncertainty plays a crucial role. Succinctly,

\squishlist
\item \textit{Incomplete Information}: AI systems often operate with partial information about the environment, meaning they operate with probabilistic rather than deterministic knowledge.
\item \textit{Ambiguous Evidence}: In domains like medical diagnosis or law, the available evidence may support multiple hypotheses to varying degrees, requiring probabilistic reasoning.
\item \textit{Uncertain Human Models}: Human users may hold beliefs with varying degrees of confidence, and AI systems may be uncertain about what humans know or believe.
\item \textit{Probabilistic Effects}: In more realistic, dynamic environments, actions may have probabilistic outcomes, requiring explanations that account for these uncertainties.
\squishend

\xhdr{Motivating Example.}To ground our discussion, consider a service robot operating in an office building with some uncertain information about its environment.~The robot must navigate from its current location to deliver a package to a specific office, choosing among multiple possible routes.~In this scenario, uncertainty may arise due to several reasons, such as environmental uncertainty, e.g.,~the robot has limited information about the current state of different corridors (e.g.,~crowded or clear), outcome uncertainty, e.g.,~even with a chosen path, the robot may not be able to perfectly predict travel times due to potential obstructions or changing conditions, and uncertainty about the beliefs of humans, i.e.,~when explaining its decisions to building staff, the robot may not know exactly what they believe about current building conditions.

Now, assume that the robot chooses a longer route through corridor \texttt{B} instead of a shorter path through corridor \texttt{A}. When asked why it did not take the more direct route through corridor \texttt{A}, the robot could provide different (monolithic) explanations. For instance, ``Corridor \texttt{A} likely has high foot traffic at this time of day, which would impede movement and potentially require frequent stops'', or ``There is an ongoing maintenance operation in corridor \texttt{A} that creates an obstruction risk and would likely result in a significant detour mid-route.'' Both explanations may increase our understanding of why the robot avoided corridor \texttt{A}, but they focus on different aspects of the robot's belief state and might have different explanatory power.

For the model reconciliation case, consider a human supervisor who expected the robot to take corridor \texttt{A} because they believe the corridor should be clear at this time. This discrepancy arises because the supervisor and the robot have different beliefs about the current state of the building. Then, a model reconciling explanation would need to address these specific belief differences and help the supervisor understand why the robot's decision was reasonable given its information.

This example illustrates how uncertainty naturally arises in practical scenarios and how multiple potential explanations must be compared to identify those with the highest explanatory power.

\section{Background}
\label{sec:back}

In this section, we provide some background for propositional logic, the hitting set duality between minimal unsatisfiable and minimal correction sets, modeling uncertainty in propositional logic, and the model reconciliation problem.

\subsection{Propositional Logic}
\label{sec:proplog}

Let $\fml{L}$ be a propositional language built from a finite set of atomic variables $\fml{V} = \{a,b,c,\ldots \}$. A \textit{possible world} is a truth-value assignment to each variable $\omega: \fml{V} \mapsto \{T, F\}$, where $T$ and $F$ denote truth and falsity respectively. The set of all possible worlds of $\fml{L}$ is denoted by $\Omega$. The simplest formulae in $\fml{L}$ are atoms: Individual variables that may be true or false in a given possible world. More complex formulae are recursively constructed from atoms using the classical logical connectives. A \textit{model} of a formula is a possible world in which the formula is satisfied (i.e.,~evaluates to true). A knowledge base $\KB$ is a set of formulae. If there exists at least one possible world $\omega$ that satisfies all formulae in $\KB$, then $\KB$ is \textit{consistent}, otherwise we say that $\KB$ is \textit{inconsistent}. We use $\models$ to denote the classical entailment relation and say that a (consistent) $\KB$ entails a formula $\varphi$, expressed as $\KB \models \varphi$, if and only if every model of $\KB$ is also a model of $\varphi$, or equivalently, if $\KB \cup \{\lnot \varphi \}$ is inconsistent. Unless stated otherwise, it is assumed that all formulae are expressed in \textit{conjunctive normal form} (CNF).\footnote{A CNF formula is a conjunction of clauses, where each clause is a disjunction of literals. A literal is either an atom or its negation. This is not a restrictive requirement, since any propositional formula can be transformed into a CNF representation.}

Given a knowledge base $\KB$ and a formula $\varphi$, called the \textit{explanandum} such that $\KB \models \varphi$, we define a \textit{monolithic explanation} for $\varphi$ from $\KB$ as a minimal set of formulae that entails~$\varphi$:

\begin{definition}(Monolithic Explanation)
Let $\KB$ be a knowledge base and $\varphi$ an explanandum such that $\KB \models \varphi$. We say that $\epsilon \subseteq \KB$ is a monolithic explanation for $\varphi$ from $\KB$ if and only if: (i) $\epsilon \models \varphi$; and (ii) $\nexists \epsilon' \subset \epsilon$ such that $\epsilon' \models \varphi$.   
\label{def:expl}
\end{definition}

\begin{example}
\label{back:expl1}
    Consider the knowledge base $\KB = \{p, \lnot p \vee q, \lnot p \vee r \}$ built up from $\fml{V} = \{p,q,r\}$. Notice that $\KB \models q$. Then, $\epsilon = \{p, \lnot p \vee q \}$ is a monolithic explanation for $q$ from $\KB$.
\end{example}

\subsubsection{Duality of Minimal Unsatisfiable and Minimal Corrections Sets}
\label{mus:mcs}

\begin{definition}
[Minimal Unsatisfiable Set (MUS)]
	Given an inconsistent knowledge base $\KB$, a subset $\M \subseteq \KB$ is an MUS if $\M$ is inconsistent and $\forall \M' \subset \M$, $\M'$ is consistent. 
\end{definition}

\begin{definition}
[Minimal Correction Set (MCS)]
	Given an inconsistent  knowledge base $\KB$, a subset $\C \subseteq \KB$ is an MCS if $\KB \setminus \C$ is consistent and $\forall \C' \subset \C$, $\KB \setminus \C'$ is inconsistent.
\end{definition}

By definition, every inconsistent $\KB$ contains at least one MUS. 

\begin{definition}[Partial MUS]
	A set of formulae $\Phi$ is a partial MUS of an inconsistent knowledge base $\KB$ if there exists at least one MUS $\M \subseteq \KB$ such that $\Phi \subseteq \M$. 
\end{definition}

Partial MUSes in an inconsistent knowledge base $\KB$ appear when a subset of formulae is set as \emph{hard}, that is, formulae that must always be satisfied in a solution. Conversely, \emph{soft} formulae may not always be satisfied. Given a formula $\varphi$, we will write $\varphi^*$ with $* \in \{s, h\}$ to denote it as \emph{soft} and \emph{hard}, respectively. 

MUSes and MCSes are related by the concept of \emph{hitting set}:

\begin{definition}
[Hitting Set]
	Given a collection $\Gamma$ of sets from a universe $U$, a hitting set for $\Gamma$ is a set $H \subseteq U$ such that $\forall S \in \Gamma$, $H \cap S \neq \emptyset$. A hitting set $H$ is \textit{minimal} if $\nexists H'\subset H$ such that $H'$ is a hitting set.
\end{definition}

The relationship between MUSes and MCSes is discussed by \citet{kas-jar08} and \citet{liffiton-pmm16}, and it was first presented by \citet{rei87}, where MUSes and MCSes are referred to as (minimal) conflicts and diagnoses, respectively.

\begin{proposition}\label{prop:duality}
	A subset of an inconsistent knowledge base $\KB$ is an MUS (resp.~MCS) if and only if it is a minimal hitting set of the collection of all MCSes (resp.~MUSes) of $\KB$. 
\end{proposition}

It follows from the above proposition that a cardinality minimal MUS (resp.~MCS) is a minimal hitting set. \emph{Cardinality minimal MUS} are referred to as SMUS, whereas a \emph{cardinality minimal MCS} corresponds to the complement of a MaxSAT solution~\cite{li2009maxsat}. We may refer to a cardinality minimal set as a minimum or smallest set.

\begin{lemma}\label{lem:smus}
	Given a subset $\fml{H}$ of all the MCSes of knowledge base $\KB$, a hitting set is an SMUS if:
(1)~It is a minimal hitting set $h$ of $\fml{H}$, and 
(2)~The subformula induced by $h$ is inconsistent. 

\end{lemma}

\noindent See the work by \citet{ignatiev-plm15} for a proof.

Proposition \ref{prop:duality} and Lemma \ref{lem:smus} naturally extend to the case of partial MUS. Note that when some formulae are set as hard in an inconsistent knowledge base, the set of all MCSes is a subset of the soft formulae. In this case, every minimal hitting set on the set of all MCSes is a partial MUS.

Finally, MUSes and monolithic explanations are related by the following:

\begin{proposition}\label{prop:support-mus}
Given a knowledge base $\KB$, a consistent set of formulae $\epsilon \subseteq \KB$ is a monolithic explanation for $\varphi$ from $\KB$ (Definition~\ref{def:expl}) if and only if $\epsilon$ is a partial MUS of $\epsilon \cup \{ \lnot \varphi \}$.
\end{proposition}

\begin{example}
    Let $\KB = \{p, \lnot p \vee q, \lnot p \vee r \}$ and $\epsilon = \{p, \lnot p \vee q \}$ from Example~\ref{back:expl1}. Notice how $\M = \{p, \lnot p \vee q, \lnot q \}$ is an MUS of $\KB \cup \{\lnot q \}$. Then, it is easy to see that $\epsilon$ is a partial MUS of $\M$.
\end{example}

\subsection{Modeling Uncertainty in Propositional Logic}

Building on a propositional language $\fml{L}$, we can model the uncertainty of propositional formulae using a \textit{probability distribution} over the possible worlds $\Omega$ of $\fml{L}$. Formally,

\begin{definition}[Probability Distribution]
Let $\Omega$ be the set of possible worlds of the language $\fml{L}$. A probability distribution $P$ on $\Omega$ is a function $P: \Omega \mapsto [0,1]$ such that $\underset{\omega \in \Omega}{{\mathlarger{\sum}}}P(\omega) = 1$.
\end{definition}

In essence, a probability distribution over possible worlds creates a \textit{ranking} between those worlds with respect to how likely they are to be true. This then allows us to quantify the uncertainty in a formula as follows:

\begin{definition}[Degree of Belief]
Let $\Omega$ be the set of possible worlds and $P$ a probability distribution over $\Omega$. The degree of belief in a formula $\varphi \in \fml{L}$ is $P(\varphi) = \underset{\omega \models \varphi}{{\mathlarger{\sum}}}P(\omega)$.\footnote{Note that, with a slight abuse of notation, we use $\omega \models \varphi$ to denote that $\varphi$ is true (e.g.,~satisfied) in world $\omega$.}
\end{definition}

We may refer to $P(\varphi)$ as degree of belief or probability of $\varphi$ interchangeably. Note that the possible worlds approach to probabilities is essentially equivalent to probabilities assigned directly to the formulae~\cite{bacchus1989representing}.

Now, the probability distribution on $\Omega$ can be induced from a weighted knowledge base, referred to as a \textit{belief base}:
\begin{equation}
    \B = \{(\phi_1, w_1 ), \ldots, ( \phi_n, w_n )\}
\end{equation}
\noindent where each formula $\phi_i \in \fml{L}$ is associated with a weight $w_i \in \R^+$.\footnote{We assume, without loss of generality, that all weights are non-negative because a formula with a negative weight $w$ can be replaced by its negation with weight $-w$.} 

Intuitively, the weights serve as meta-information and reflect the certainty about the truth of the corresponding formulae---the higher the weight, the more certain the formula is. In that sense, formulae with higher weights are prioritized for satisfaction, effectively capturing the certainty of the particular formulae. This mechanism is especially useful for handling inconsistency and non-monotonic reasoning patterns, thus capturing a broader spectrum of problems.\footnote{For example, the notion of inconsistency is relaxed as follows: Given two formulae $\phi$ and $\lnot \phi$ that contradict each other (e.g.,~$\{\phi, \neg \phi \}$ is inconsistent), if $P(\phi) = 0.9$, then from the axioms of probability we have that $P(\lnot \phi) = 0.1$. This then means that the worlds where $\lnot \phi$ is true are more unlikely than the worlds where $\phi$ is true, but not impossible.} Further, we will denote with $\B^{\downarrow w}$ the \textit{classical projection} of $\B$, that is, $\B^{\downarrow w} = \{ \phi_i \: | \: (\phi_i, w_i) \in \B \}$.

Given a belief base $\B$, one way to induce a probability distribution is the following:

\begin{equation}
    \forall \omega\in \Omega, \; P_{\B}(\omega) = \frac{1}{Z}\exp \left ( \: \overset{n}{ \underset{i=1}{{\mathlarger{\sum}}}} w_i \cdot \I(\omega, \phi_i) \right )
\end{equation}

\noindent where $\I(\omega, \phi) = 1$ if $\omega \models \phi$ and $0$ otherwise, and 
$Z = \underset {\omega \in \Omega} \sum \exp \left ( \: \overset{n}{ \underset{i=1}{\sum}} w_i \cdot \I(\omega, \phi_i) \right )$ 
is the normalization factor.

The induced probability distribution quantifies the likelihood that a given (possible) world is the actual world. Higher formula weights amplify the (log-) probability difference between a world that satisfies the formula and one that does not, other things being equal. Consequently, worlds that violate fewer formulas are deemed more probable. Note that a belief base $\B$ is essentially a \textit{log-linear model}~\cite{bishop2007discrete}, from which a \textit{joint probability distribution} of the set of variables of $\fml{L}$ is induced. Interestingly, log-linear models are special cases of Markov Logic Networks and can represent any positive distribution~\cite{richardson2006markov}. When taken from context, we will simply use $P$ to denote the distribution induced from $\B$.

Entailment in a belief base $\KB$ becomes \textit{graded}, that is, we now say that $\B$ entails a formula $\phi$ with degree of belief $P(\phi)$. However, when all weights are equal and tend to infinity, a belief base represents a uniform distribution over the worlds that satisfy it and, as such, entailment of a formula can be answered by computing the probability of the formula and checking whether it is $1$. In other words, entailment under belief bases collapses to classical entailment under knowledge bases. See \citet{richardson2006markov} for a proof.

Finally, the weighted formulae in a belief base $\B$ can be viewed as \textit{soft} constraints in the sense described in Section~\ref{mus:mcs}. In contrast, \textit{hard} constraints can be imposed as formulae with ``infinite'' weights.\footnote{In practice, infinite weights can be replaced with $\overset{n}{ \underset{i=1}{\sum}} w_i + 1$.}

\subsection{The Model Reconciliation Problem}
\label{sec:MRP}

The \textit{Model Reconciliation Problem} (MRP), as introduced by \citet{chakraborti2017plan}, highlights the critical need for aligning the planning models of a human user and an agent to facilitate effective collaboration and understanding. This alignment becomes especially pertinent in scenarios where the agent's plan deviates from human expectations, necessitating a mechanism to reconcile these differences through explanations. In this approach, the (planning) agent must have knowledge of the human's model in order to contemplate their goals and foresee how its plan will be perceived by them. When there exist differences between the models of the agent and the human such that the agent’s plan diverges from the human’s expectations, the agent provides a minimal set of model differences, namely a \textit{model reconciling explanation}, to the human.

It is important to highlight that, in order to effectively solve MRP, the following (implicit) assumptions typically hold: 

\squishenum
    \item[1.] The \textit{agent model represents the ground truth} or, in other words, the agent model is the ``correct'' encoding of the domain. This assumption is predicated on the notion that the explanation is generated from the agent's perspective, thereby rendering it reasonable to assume that the agent ``thinks'' that its model is accurate or correct.

    \item[2.] The \textit{agent has access to the human model}, which is an approximation of the actual human model. In the worst case, it can be empty; but, practically, it can be approximated based on past interactions \cite{sreedharan2018handling,Juba_learning21}.
    
    \item[3.] \textit{Both models are assumed to be deterministic}, and they thus are able to represent only deterministic domains. Note that while this is a restricting assumption, in Section~\ref{sec:plmrp} we present a framework that relaxes it.

\squishend

Now, building upon the MRP foundation, we introduced its logic-based variant (L-MRP), where the models of the agent and the human user are represented as logical knowledge bases \cite{vasLMRP}.\footnote{We have also introduced extensions that are complementary to the work presented in this paper \cite{vasileiouplease,vasileiou2024dialectical}.} As a \textit{model reconciling explanation} must take into account both the knowledge base $\KB_\alpha$ of the agent providing an explanation as well as the knowledge base $\KB_h$ of the human receiving the explanation, it is defined slightly differently compared to monolithic explanations defined by Definition \ref{def:expl}:

\begin{definition}
[Model Reconciling Explanation]
Given the knowledge bases of an agent $\KB_\alpha$ and a human user $\KB_h$ as well as an explanandum $\varphi$, such that $\KB_\alpha \models \varphi$ and $\KB_h \not \models \varphi$, $\fml{E} = \langle \epsilon^{+}, \epsilon^{-} \rangle$ is a model reconciling explanation if and only if $\epsilon^{+} \subseteq \KB_\alpha$, $\epsilon^{-} \subseteq \KB_h$, and $(\KB_h \cup \epsilon^{+})\setminus \epsilon^{-}\models \varphi$.
\label{lmrp}
\end{definition}

When $\KB_h$ is \textit{updated} with a model reconciling explanation $\fml{E} = \langle \epsilon^{+}, \epsilon^{-} \rangle$, new formulae $\epsilon^+$ from $\KB_\alpha$ are added to $\KB_h$ and formulae $\epsilon^-$ from $\KB_h$ are retracted to ensure consistency. Note that since a model reconciling explanation is from the perspective of the agent's knowledge base $\KBa$, we implicitly assume that if a formula in $\KB_h$ is inconsistent with $\KBa$, then that formula is ``false'' from the perspective of the agent.

\begin{example}
    Let $\KB_\alpha = \{a, \lnot a \vee b, \lnot a \vee c \}$ and $\KB_h = \{\lnot a, \lnot a \vee b \}$ be the knowledge bases of an agent and a human user, respectively, where $\KB_\alpha \models b$ and $\KB_h \not \models b$. A model reconciling explanation is then $\fml{E} = \langle \{a \}, \{ \lnot a \} \rangle$, where $(\KB_{h} \cup \{a\})\setminus \{\lnot a\} = \{a, \lnot a \vee b \} \models b$.
\end{example}


\section{Motivating Application: Office Robot Delivery}
\label{sec:example}

To demonstrate our ideas, let us revisit the service robot scenario described earlier. Consider a robot operating in an office building with the task of delivering a package from its current location (Room 1) to a destination (Room 2). The robot can navigate through two possible corridors: corridor A, which is shorter but often crowded, or corridor B, which is longer but typically less crowded. Figure~\ref{fig:example} illustrates a simplified abstraction of this problem in the form of a grid world.

The robot's action dynamics include two primary actions: \texttt{move} and \texttt{deliver}. The \texttt{move} action allows the robot to navigate between rooms through a specified corridor, subject to preconditions (the robot must be at the origin location) and probabilistic effects that depend on the corridor's crowding status, i.e.,~when attempting to move through a crowded corridor, the robot has a higher probability of failing to move and remaining in its current location. The \texttt{deliver} action can only be performed when the robot is at the destination room and results in the package being delivered. The complete problem specification and encoding is provided in the appendix.

We encode this probabilistic planning scenario as a belief base $\B$ containing weighted formulae (with subscripts denoting time steps) that represent the robot's beliefs about the environment states and actions. For instance, the initial states and the action dynamics for the \texttt{move} action (for $t=0$) are: \\

\textbf{Initial state beliefs:}
\begin{align*}
&(\texttt{robot}\textrm{-}\texttt{at}(room1)_0, \infty), \\
&(\neg \texttt{package}\textrm{-}\texttt{delivered}_0, \infty), \\
&(\texttt{crowded}(A), 3), \\
&(\texttt{crowded}(B), 0.5),
\end{align*}


\textbf{\texttt{Move} preconditions:}
\begin{align*}
&( \texttt{move}(room1,room2,A)_0 \rightarrow \texttt{robot}\textrm{-}\texttt{at}(room1)_0, \infty), \\
&( \texttt{move}(room1,room2,B)_0 \rightarrow \texttt{robot}\textrm{-}\texttt{at}(room1)_0, \infty), \\
\end{align*}

\textbf{\texttt{Move} effects:}
\begin{align*}
&(\texttt{move}(room1,room2,A)_0 \wedge \texttt{crowded}(A) \rightarrow \texttt{robot}\textrm{-}\texttt{at}(room2)_{1}, 0.5), \\
&(\texttt{move}(room1,room2,A)_0 \wedge \texttt{crowded}(A) \rightarrow \texttt{robot}\textrm{-}\texttt{at}(room1)_{1}, 3).
\end{align*}

\begin{figure}[t]
    \centering
    \includegraphics[width=0.7\linewidth]{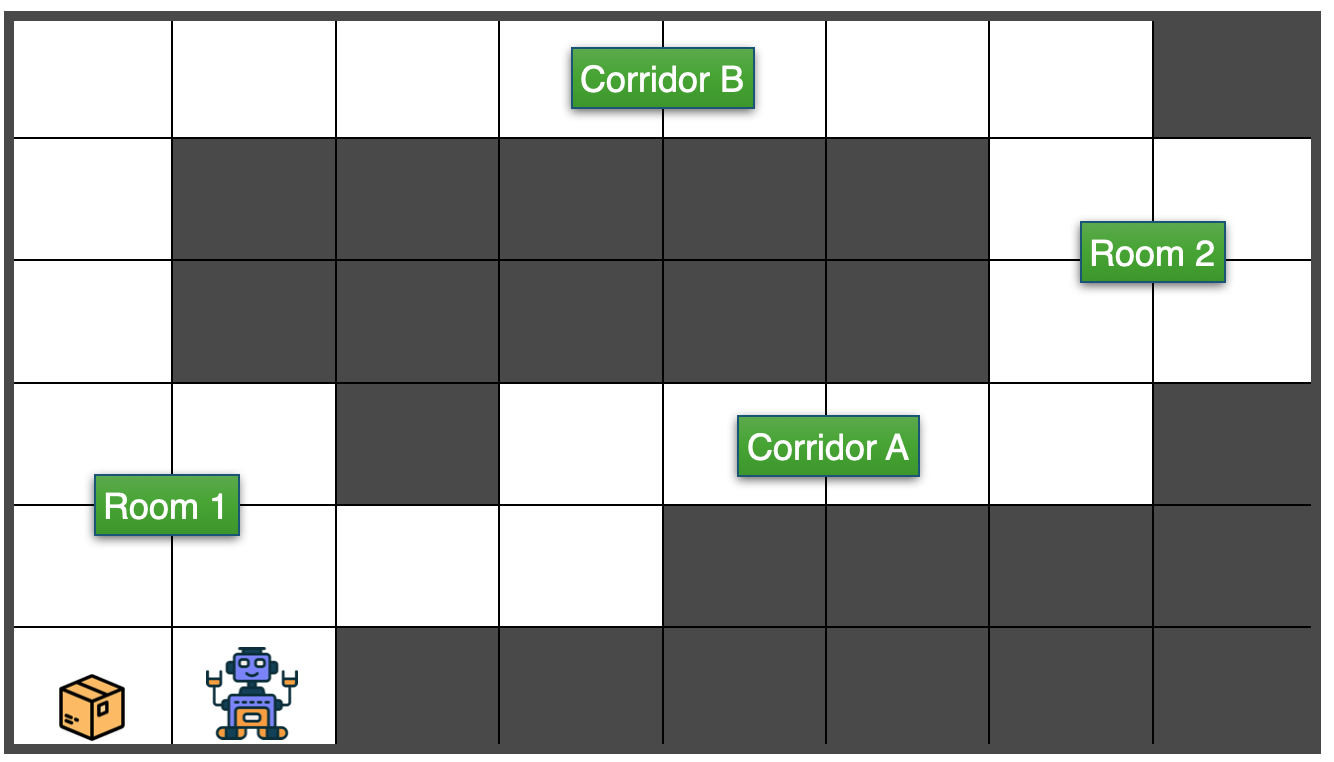}
    \caption{Grid world illustration of the package delivery problem. The robot has to deliver the package from Room 1 to Room 2 using either Corridor A or Corridor B.}
    \label{fig:example}
\end{figure}

In this belief base, the weights encode the robot's degrees of belief. For instance, the weight of 3 for $crowded(A)$ represents a strong belief that corridor A is crowded, while the weight of 0.5 for $crowded(B)$ indicates a weaker belief that corridor B is crowded. Similarly, the probabilistic action effects encode that when a corridor is crowded, it is less likely that the robot will move through it to the next room. Given its belief base $\B$, the robot computes that the optimal plan is $\pi = \langle \texttt{move}(room1, room2, B)_0, \texttt{deliver}_1 \rangle$, with probability $P(\pi) = 0.6$.

Now, suppose a human supervisor observes the robot's choice and asks: ``Why didn't you take corridor A instead?'' This query can be formalized as the explanandum $\varphi = \neg \texttt{move}(room1, room2, A)$, essentially encoding ``not moving from Room 1 to Room 2 through corridor A'', and allowing the robot to generate an explanation of why that is the case. In the following section, we will revisit this scenario and show how our methods can be used to generate probabilistic explanations that answer the query.

\section{A Framework for Probabilistic Explanation Generation}
\label{sec:pexpl-gen}

In this section, we approach the notion of \textit{probabilistic explanation} in the following setting, where $\B$ is a belief base and $P$ its induced probability distribution:

\begin{definition}[Probabilistic Explanation Generation Problem (PEGP)]
    Given a belief base $\B$ and an explanandum $\varphi$ in logic $\fml{L}$, the goal of PEGP is to identify an explanation (i.e.,~a set of formulae) $\pe \subseteq \B$ such that the probability of the explanandum given the explanation is higher than the probability of the explanandum, i.e.,~$P(\varphi \: | \: \pe) > P(\varphi)$.
\end{definition}

Under this definition, we account for the uncertainty inherent in knowledge bases, and seek to find explanations that not only identify contributing factors for an explanandum but also quantify the uncertainty in these factors. Intuitively, a solution $\pe$ to PEGP seeks to increase the degree of belief in the explanandum $\varphi$. If $P(\varphi \: | \: \pe) > P(\varphi)$, this then represents the case where $\pe$ increases the degree of belief in $\varphi$ and the greater the value of $P(\varphi \: | \: \pe)$ the greater the degree of belief in $\varphi$.

In what follows, we present a framework designed to extend the classical concepts of monolithic explanation, as defined by Definition~\ref{def:expl}, and model reconciling explanation, as defined by Definition~\ref{lmrp}, into our probabilistic setting.

\subsection{Probabilistic Monolithic Explanations}
\label{sec:pexpl}

Building on the classical notion of monolithic explanation presented in Definition~\ref{def:expl}, a \textit{probabilistic monolithic explanation} for an explanandum $\varphi$ from belief base $\B$ is defined as follows:

\begin{definition}[Probabilistic Monolithic Explanation]
\label{def:pexpl}
Let $\B$ be a belief base, $\B^{\downarrow w}$ its classical projection, and $\varphi$ an explanandum. We say that $\pe \subseteq \B^{\downarrow w}$ is a probabilistic monolithic explanation for $\varphi$ from $\B$ if and only if $P(\varphi \: | \: \pe) > P(\varphi)$.
\end{definition}

\begin{example}
\label{pex:1}
    Consider the belief base $\B = \{ (a, 1), (\lnot a \vee b, 2) \}$ and the explanandum $b$. The probability of the explanandum is $P(b) = 0.73$. Then, $\pe_1 = \{ a \}$ and $\pe_2 = \{ \lnot a \vee b \}$ are two probabilistic monolithic explanations for $b$ from $\B$, that is, $P(b \: | \: \pe_1) = 0.88 > P(b)$ and $P(b \: | \: \pe_2) = 0.78 > P(b)$.
\end{example}

It is important to note that Definition~\ref{def:pexpl} can be extended to the case where the formulae $\pe$ do not necessarily come from $\B$, but rather from the language $\fml{L}$. However, we restrict our attention only to formulae from $\B$ in order to be compatible with the classical notion of monolithic explanations (see Definition~\ref{def:expl}) and the algorithms that we will present in Section~\ref{sec:comp}. For brevity, and until the end of this section, we will refer to probabilistic monolithic explanations as monolithic explanations.

Looking at Example~\ref{pex:1}, we can see that monolithic explanations will typically vary in their capacity to increase the degree of belief in the explanandum. In other words, each monolithic explanation provides us with some \textit{explanatory gain} for the explanandum. Following \citet{good1960weight,good1968corroboration}, explanatory gain is defined as follows:\footnote{ \citet{good1960weight} originally introduced this measure to quantify the \textit{(weak) explanatory power} of a hypothesis with respect to evidence, essentially evaluating how effectively the hypothesis explains the evidence.}


\begin{definition}[Explanatory Gain of Monolithic Explanations]
\label{def:expl-gain}
Let $\pe$ be a monolithic explanation for explanandum $\varphi$ from belief base $\B$. The explanatory gain of $\pe$ for $\varphi$ is defined as $G(\pe, \varphi) = \log \left (\frac{P(\varphi \: | \: \pe)}{P(\varphi)} \right )$.\footnote{We use $\log$ with base $2$ in our calculations.}$^,$\footnote{Note that $G(\pe, \varphi)$ is always positive due to the requirement of monolithic explanations that $P(\varphi \: | \: \pe) > P(\varphi)$ (Definition~\ref{def:pexpl}).}
\end{definition}

\noindent In essence, the explanatory gain can be thought of as a measure that quantifies how well the monolithic explanation $\pe$ \textit{explains} the explanandum $\varphi$ or, equivalently, the degree to which $\pe$ entails $\varphi$. The greater the value of $G(\pe, \varphi)$, the more substantial the explanatory gain and, hence, the more effective $\pe$ is at explaining $\varphi$. 

It is essential to recognize that this measure, while initially introduced to assess the weak explanatory power of hypotheses in light of evidence \cite{good1960weight}, it is used here to evaluate monolithic explanations. By quantifying the extent to which a monolithic explanation explains an explanandum, we can systematically identify the most informative monolithic explanations within a probabilistic framework.

\begin{example}
\label{pex:2}
Continuing from Example~\ref{pex:1}, consider the monolithic explanations $\pe_1 = \{ a \}$, $\pe_2 = \{ \lnot a \vee b \}$, and $\pe_3 = \{a, \lnot a \vee b \}$ for explanandum $b$. The explanatory gains of $\pe_1$, $\pe_2$, and $\pe_3$ for $b$ are $G(\pe_1, b) = \log \left (\frac{P(b \: | \: \pe_1)}{P(b)}\right ) = \log \left (\frac{0.88}{0.73}\right ) = 0.27$, $G(\pe_2, b) = \log \left (\frac{P(b \: | \: \pe_2)}{P(b)}\right ) = \log \left (\frac{0.78}{0.73}\right ) = 0.11$, and $G(\pe_3, b) = \log \left (\frac{P(b \: | \: \pe_3)}{P(b)}\right ) = \log \left (\frac{1}{0.73}\right ) = 0.45$, respectively.
\end{example}

A natural course of action when seeking monolithic explanations for an explanandum is to seek the one with the highest explanatory gain. While tempting, it is important to emphasize that when a monolithic explanation entails the explanandum, then the explanatory gain takes on its greatest value. For example,

\begin{example}
    Consider the three monolithic explanations $\pe_1$, $\pe_2$, and $\pe_3$ from Example~\ref{pex:2}. Notice that $\pe_3 =  \{a, \lnot a \vee b\}$ entails $b$ (i.e.,~$\pe_3 \models b$) and that its explanatory gain is higher than that of $\pe_1$ and $\pe_2$. As $\pe_1$, $\pe_2$, and $\pe_3$ are the only three possible explanations for $b$, $G(\pe_3, b)$ is indeed the maximum achievable explanatory gain for $b$.
\end{example}

We formalize this in the following proposition:
\begin{proposition}
\label{prop:explgain}
    Given a monolithic explanation $\pe$ for an explanandum $\varphi$ from belief base $\B$, if $\pe \models \varphi$, then $G(\pe, \varphi)$ achieves its maximal value for $\varphi$, specifically $G(\pe, \varphi) = - \log P(\varphi)$.
\end{proposition}

\begin{proof}
    If $\pe \models \varphi$, then for all possible worlds $\omega$ in which $\omega \models \pe$, it holds that $\omega \models \varphi$. That is, the worlds $\omega$ in which $\pe$ is true are subsumed by the worlds in which $\varphi$ is true, which implies that also $\omega \models \varphi \land \pe$. Consequently, $P(\varphi \: | \: \pe) = \frac{\underset{\omega \models \varphi \land \pe}{{\mathlarger{\sum}}}P(\omega)}{\underset{\omega \models \pe}{{\mathlarger{\sum}}}P(\omega)} = \frac{\underset{\omega \models \pe}{{\mathlarger{\sum}}}P(\omega)}{\underset{\omega \models \pe}{{\mathlarger{\sum}}}P(\omega)} = 1$. Therefore, when $\pe \models \varphi$, the explanatory gain of $\pe$ for $\varphi$ is $G(\pe, \varphi) = \log \left (\frac{P(\varphi \: | \: \pe)}{P(\varphi)}\right ) = \log \left (\frac{1}{P(\varphi)}\right ) = - \log P(\varphi)$.
\end{proof}

The following corollary follows naturally from Proposition~\ref{prop:explgain}:

\begin{corollary}
    \label{cor:explgain1}
    Let $\tilde{E}(\varphi)$ denote the set of all monolithic explanations for explanandum $\varphi$ from belief base $\B$. For any two monolithic explanations $\pe_1, \pe_2 \in \tilde{E}(\varphi)$, if $\pe_1 \models \varphi$ and $\pe_2 \models \varphi$ (resp. $\pe_2 \not \models \varphi$), then $G(\pe_1, \varphi) = G(\pe_2, \varphi)$ (resp. $G(\pe_1, \varphi) > G(\pe_2, \varphi)$).
\end{corollary}

What Proposition~\ref{prop:explgain} and Corollary~\ref{cor:explgain1} essentially underscore is that the exclusive focus on explanatory gain as an evaluation metric for a monolithic explanation neglects the inherent likelihood of the explanation itself. That is, the explanatory gain of a monolithic explanation for an explanandum evaluates how effectively the explanation explains the explanandum, \textit{assuming that the explanation itself is true}. Nonetheless, this premise often lacks practical relevance because, in probabilistic contexts, each monolithic explanation is associated with a probability reflecting its likelihood for being true. Therefore, a good measure for evaluating monolithic explanations should incorporate the explanation's inherent plausibility.

Addressing this gap, \citet{good1968corroboration} introduced the concept of \textit{(strong) explanatory power} that integrates the monolithic explanation's explanatory gain with its probability, offering a more balanced metric for evaluating monolithic explanations.\footnote{Good's measure of (strong) explanatory power is defined as ${\mathlarger{\log}}\left(\frac{P(\varphi \: | \: h) \cdot P(h)^\gamma}{P(\varphi)}\right )$, where $h$ is a hypothesis and $0 < \gamma < 1$ a constant \cite{good1968corroboration}.} Building on Good's measure of explanatory power, we adapt it to our setting and define it as follows:\footnote{For a detailed defense of Good's measure as a quantitative criterion for explanatory power, alongside a discussion of relevant properties and a comprehensive comparison with other measures, we refer the reader to the work by \citet{glass2023good}.}

\begin{definition}[Explanatory Power of Monolithic Explanations]
\label{def:expl-pow}
Let $\pe$ be a monolithic explanation for explanandum $\varphi$ from belief base $\B$. The explanatory power of $\pe$ for $\varphi$ is defined as $\EP(\pe, \varphi) = G(\pe, \varphi) + \gamma \cdot P(\pe)$, where $\gamma \in [0,1]$ is a constant.
\end{definition}

This definition effectively combines the measure of how much a monolithic explanation explains the explanandum (explanatory gain) with the likelihood of the explanation itself, mediated by a parameter $\gamma$. The constant $\gamma$ serves as a tuning parameter, enabling the adjustment of the relative importance of the monolithic explanation's probability in the overall assessment of explanatory power. This flexibility is important for tailoring the evaluation process to specific contexts or preferences, where the balance between the informativeness of a monolithic explanation and its plausibility may vary.

\begin{example}
\label{pex:3}
    Consider the belief base $\B = \{ (a, 1.5), (b, 3), (\lnot a \vee c, 1), (\lnot b \vee c, 1) \}$ and the explanandum $c$ with initial probability $P(c) = 0.84$. Notice that $\pe_1 = \{ a, \lnot a \vee c \}$ and $\pe_2 = \{ b, \lnot b \vee c \}$ are two monolithic explanations for $c$ from $\B$, each of which entail $c$ (i.e.,~$\pe_1 \models c$ and $\pe_2 \models c$), with probabilities $P(\pe_1) = 0.68 $ and $P(\pe_2) = 0.80$, respectively. This means that their explanatory gain for $c$ is equal (Corollary~\ref{cor:explgain1}), that is, $G(\pe_1, c) = G(\pe_2, c) = 0.25$. Now, assuming $\gamma = 0.5$, the explanatory power of $\pe_1$ and $\pe_2$ respectively is $\EP(\pe_1, c) = 0.25 + 0.5 \cdot 0.68 = 0.59$ and $\EP(\pe_2, c) = 0.25 + 0.5 \cdot 0.80  = 0.65$.
\end{example}

With the introduction of explanatory power as an evaluative measure of (probabilistic) monolithic explanations, we can now define a (probabilistic) \textit{preference relation} among monolithic explanations, which allows for a systematic approach to determining the most effective monolithic explanation for a given explanandum:

\begin{definition}[Preference Relation for Monolithic Explanation]
\label{def:prel}
Let $\pe_1$ and $\pe_2$ be two monolithic explanations for explanandum $\varphi$ from belief base $\B$. $\pe_1$ is preferred over $\pe_2$, denoted as $\pe_1 \succeq \pe_2$, if and only if $\EP(\pe_1, \varphi) \geq \EP(\pe_2, \varphi)$.
\end{definition}

This definition enables a quantitatively grounded approach to preference among monolithic explanations, where the preference is directly tied to the explanatory power of each explanation. It facilitates a structured way to navigate the space of potential monolithic explanations, prioritizing those that not only explain the explanandum more effectively, but also align more closely with the existing knowledge represented by the belief base $\B$.

\begin{example}
Continuing from Example~\ref{pex:3}, the two monolithic explanations for $c$ from $\B$ are $\pe_1$ and $\pe_2$ and have explanatory power $\EP(\pe_1, c) = 0.59$ and $\EP(\pe_2, c) = 0.65$. Thus, $\pe_2$ is preferred over $\pe_1$ (i.e.,~$\pe_2 \succeq \pe_1$).
\end{example}

Finally, given the set of all monolithic explanations for an explanandum, we say that a monolithic explanation is \textit{most preferred} if and only if it is (probabilistically) preferred over every other possible monolithic explanation for that explanandum. Formally,

\begin{definition}[Most-Preferred Monolithic Explanation]
\label{def:mpref}
Let $\tilde{E}(\varphi)$ denote the set of all monolithic explanations for explanandum $\varphi$ from belief base $\B$. A monolithic explanation $\pe^* \in \tilde{E}(\varphi)$ is the most-preferred monolithic explanation if and only if $\pe^* \succeq \pe$ for all $\pe \in \tilde{E}(\varphi)$.
\end{definition}

\subsubsection{Motivating Application: Illustrative Example}

Recall from our motivating scenario in Section~\ref{sec:example} that we are interested in explaining the explanandum $\varphi = \neg \texttt{move}(room1, room2, A)$, which is akin to asking ``why didn't you take corridor A?'' Using our framework, we identify two potential explanations for this explanandum:

\begin{enumerate}
    \item  $\pe_1=\{\texttt{crowded}(A)\}$ or, in natural language, ``Corridor A is very likely crowded.''
    \begin{itemize}
        \item $P(\varphi|\pe_1) = 0.62$
        \item Explanatory gain: $G(\pe_1, \varphi) = 0.1$
        \item Explanatory power: $EP(\pe_1, \varphi) = 0.5$
    \end{itemize}


   \item $\pe_2 = \{\texttt{crowded}(A), \texttt{move}(room1, room2, A)_0 \wedge \texttt{crowded}(A) \rightarrow \texttt{robot}\textrm{-}\texttt{at}(room1)_1, \texttt{move}(room1, room2, A)_0 \wedge \texttt{crowded}(A) \rightarrow \texttt{robot}\textrm{-}\texttt{at}(room2)_1\}$ or, in natural language, ``Corridor A is very likely crowded, and when attempting to move through a crowded corridor, my movement outcomes are uncertain; I may either succeed or fail to move through to the next room.''
    \begin{itemize}
        \item $P(\varphi|\pe_2) = 1$
        \item Explanatory gain: $G(\pe_2, \varphi) = 0.79$
        \item Explanatory power: $EP(\pe_2, \varphi) = 1.04$
    \end{itemize}
\end{enumerate}

Comparing the two explanations, we see that explanation $\pe_2$ has the highest explanatory power (1.04), and is thus the most-preferred explanation according to Definition~\ref{def:mpref}. Such an explanation can help the human supervisor to understand both the robot's belief about the environment (crowded corridor) and how this belief influences action outcomes (movement outcomes under crowded conditions), thereby providing a more complete picture of the robot's decision-making process.

\xhdr{Comparison with MPE:}It is important to highlight the difference with the traditional \textit{Most Probable Explanation} (MPE) method. MPE typically consists of finding the world with the highest probability given some evidence \cite{shterionov2015most}. However, a world does not show the chain of inferences of a given explanandum, and importantly, it is not minimal by definition since it usually includes a (possibly large) number of probabilistic facts whose truth value is irrelevant for the explanandum. For example, applying standard MPE to our running example we get the following outcome:
\begin{align*}
\omega_{\text{MPE}} = \{ &  \texttt{crowded}(A), \lnot \texttt{crowded}(B), \lnot \texttt{deliver}_0, \texttt{deliver}_1, \lnot \texttt{move}(room1,room2,A)_0, \\
& \lnot \texttt{move}(room1,room2,A)_1, \texttt{move}(room1,room2,B)_0, \lnot \texttt{move}(room1,room2,B)_1, \\
& \lnot \texttt{move}(room2,room1,A)_0, \lnot \texttt{move}(room2,room1,A)_1, \lnot \texttt{move}(room2,room1,B)_0, \\
& \lnot \texttt{move}(room2,room1,B)_1,  \lnot \texttt{package-delivered}_0, \lnot \texttt{package-delivered}_1, \\
& \texttt{package-delivered}_2, \texttt{robot-at}(room1)_0, \lnot \texttt{robot-at}(room1)_1, \\
&\lnot \texttt{robot-at}(room1)_2, \lnot \texttt{robot-at}(room2)_0,
\texttt{robot-at}(room2)_1, \\ 
& \texttt{robot-at}(room2)_2 \}
\end{align*}
While $\omega_{\text{MPE}}$ is indeed the most probable world ($P(\omega_{\text{MPE}}) = 0.32$) in which the query ($\neg \texttt{move}(room1, room2, A)$) is true, it does not pinpoint to the (minimal) combination of formulae that justify the query.


\subsection{Probabilistic Model Reconciling Explanations}
\label{sec:plmrp}

Recall from Section~\ref{sec:MRP} that, in the model reconciliation problem (MRP), the models of the agent and the human user diverge with respect to an explanandum, insofar as the explanandum is explicable in the agent's model but inexplicable in the human's model. The goal is then to find a model reconciling explanation (i.e.,~a set of model differences) such that the explanandum becomes explicable in the human's model. Three important assumptions underlying MRP typically hold: (1)~the agent model is the ground truth; (2)~the agent has access to the human model; and (3)~both models are deterministic. 

As we described in Section~\ref{sec:MRP}, assumption (1) is reasonable since explanations are generated from the agent's perspective. In other words, the agent ``thinks'' that its model is correct. For assumption (2), the agent does not have access to the human's actual model, but an approximation of it. In the worst case, it can be empty; but, practically, it can be approximated based on past interactions \cite{sreedharan2018handling,Juba_learning21}. For assumption (3), we will relax the assumption that the human model is deterministic in our work, but we will still assume that the agent model is deterministic.

The motivation for moving away from deterministic human models becomes stronger when we consider two key points. First, since the agent is using an approximated human model, deterministic approximations are more likely to be inaccurate compared to probabilistic ones. Consequently, deterministic models may generate explanations that are incorrect or not meaningful for the user, thereby reducing the effectiveness of MRP. Secondly, it is likely that humans hold beliefs with varying degrees of certainty, highlighting a shortfall of deterministic models in capturing this range of uncertainties. These factors together underscore the necessity for models that incorporate probabilistic aspects, thus potentially enabling a more accurate and user-relevant application of MRP.

To that end, we will now expand the scope of MRP to cases in which the agent is uncertain about the human model. Particularly, we build on our previous work \cite{vasLMRP}, wherein both the agent and human models are represented as (logical) knowledge bases (see Definition~\ref{lmrp}), and extend it to the case where the human knowledge base is probabilistic (i.e.,~a belief base). In other words, we are now interested in \textit{probabilistic model reconciling explanations}.

First, we show through the following example how the concepts surrounding probabilistic monolithic explanations introduced in the previous section are applicable to the case of an agent knowledge base $\KBa$ and a human belief base $\Bh$.

\begin{example}
    \label{pex:lmrp1}
Let $\KBa = \{a, \lnot a \vee b, c \}$ and $\Bh = \{(c, 2), (\lnot c \vee \lnot a, 2)\}$ be the knowledge bases of an agent and the belief base of a human, respectively. Additionally, let $b$ be the explanandum, where $\KBa \models b$ and $P_h(b) = 0.5$. The goal in this example would then be to find which formulae from $\KBa$ increase the probability of the explanandum for $\Bh$, that is, to find a probabilistic monolithic explanation $\pe$ for $b$ from $\KBa$ for $\Bh$ such that $P_h(b \: | \: \pe) > P_h (b)$ (Definition~\ref{def:pexpl}). 

Given $\KBa$, there are three possible monolithic explanations: $\pe_1 = \{a \}$, $\pe_2 = \{ \lnot a \vee b\}$, and $\e_3 = \{a, \lnot a \vee b \}$. Evaluating them with respect to the probability distribution induced by $\Bh$, we get $P_h(b \: | \: \pe_1) = 0.5$, $P_h(b \: | \: \pe_2) = 0.55$, and $P_h(b \: | \: \pe_3) = 1$. Notice now that only $\pe_2$ and $\pe_3$ qualify as monolithic explanations since $P_h(b \: | \: \pe_2) > P_h(b)$ and $P_h(b \: | \: \pe_3) > P_h(b)$, whilst $\pe_1$ does not qualify as a monolithic explanation as $P_h(b \: | \: \pe_1) = P_h(b) = 0.5$.

Given $\pe_2$ and $\pe_3$ as the two possible monolithic explanations, we can now evaluate their effectiveness in terms of explanatory gain (Definition~\ref{def:expl-gain}) and explanatory power (Definition~\ref{def:expl-pow}). In terms of explanatory gain, we get $G(\pe_2, b) = 0.14$ and $G(\pe_3, b) = 1$. In terms of explanatory power (for $\gamma = 0.5$), we get $\EP(\pe_2, b) = 0.59$ and $\EP(\pe_3, b) = 1.04$. Finally, following the definition of most-preferred monolithic explanation (Definition~\ref{def:mpref}), we get that $\pe_3$ is the most-preferred monolithic explanation for $b$ from $\KBa$ for $\Bh$.
\end{example}

On the one hand, example~\ref{pex:lmrp1} shows that the definitions introduced in Section~\ref{sec:pexpl} can be directly applied to the case of an agent knowledge base $\KBa$ and a human belief base $\Bh$. On the other hand, there is something important to highlight here. Despite $\pe_3$ being the most-preferred monolithic explanation (i.e.,~it has the highest explanatory power), notice that its probability $P_h(\pe_3) = 0.09$ is rather low, which means that its negation $\lnot \pe_3$ has a much higher probability with $P_h(\lnot \pe_3) = 0.91$. Logically, this is explained by the fact that $\pe_3$ is inconsistent with the formulae in $\Bh^{\downarrow w}$. Therefore, the probabilistic monolithic explanation $\pe_3$ may not achieve the intended ``reconciliation'' between the agent and the human.

Recall that a model reconciling explanation (see Definition~\ref{lmrp}) is of the form  $\fml{E} = \langle \e^+, \e^- \rangle$, where $\e^-$ is specifically intended to resolve the inconsistency between the agent and the human with respect to the explanandum. Intuitively, the provision of $\e^-$ can be thought of as the agent's suggestion of what is ``false'' in the human knowledge base, at least compared to the agent knowledge base. In the case of a human belief base $\Bh$, we can account for $\e^-$ by finding a set of formulae from $\Bh$ such that $P_h (\e^+ \: | \: \lnot \e^-) > P_h (\e^+)$. For example,

\begin{example}
\label{pex:lmrp2}
Let $\KBa = \{a, \lnot a \vee b, c \}$ and $\Bh = \{(c, 2), (\lnot c \vee \lnot a, 2)\}$ from Example~\ref{pex:lmrp1}. From the perspective of $\KBa$, explanation $\pe = \{a, \lnot a \vee c \}$ can be seen as the formulae that should be true (e.g.,~added) in $\Bh$ (i.e.,~$\pe^+ = \pe$). However, notice that $\pe^+$ is inconsistent with $\Bh^{\downarrow w} = \{c, \lnot c \vee \lnot a \}$. Thus, from the perspective of $\KBa$, some formulae from $\Bh^{\downarrow w}$ are false (e.g.,~they should be retracted). One can see that $\pe^- = \lnot c \vee \lnot a$ is the only formula that should be false as it is the only one that is inconsistent with $\KBa$. Indeed, if $\pe^-$ is assumed to be false, then the probability of $\pe^+$ increases, i.e.,~$P_h (\pe^+ \: | \: \lnot \pe^-) = 0.5 > P_h (\pe^+) = 0.09$. Therefore, $\pe^+$ and $\pe^-$ can be seen as a model reconciling explanation for $b$ from $\KBa$ for $\Bh$.
\end{example}

Before formally defining what constitutes a probabilistic model reconciling explanation, we state the following assumptions underlying our framework:

\squishlist
    \item \textbf{Shared Domain Language}: The agent and the human user share the same (propositional) language $\fml{L}$, that is, they share the same set of atomic variables $\fml{V}$ from which formulae specific to a domain can be constructed.

    \item  \textbf{Agent Knowledge Base}: The agent model is represented by the (deterministic) knowledge base $\KBa$, encoding the ground truth of the domain. 

    \item \textbf{Human Belief Base}: The human model is represented by the belief base $\Bh$ (and its associated probability distribution $P_h$), reflecting the agent's uncertainty, for example, its degrees of belief about the human model. The agent has access to $\Bh$ a-priori.\footnote{We leave the question of acquiring (or learning) the human belief base open for future work.} 
\squishend

We define a \textit{probabilistic model reconciling explanation} as follows:
 
\begin{definition}[Probabilistic Model Reconciling Explanation]
\label{def:plmrp}
Given the knowledge base $\KBa$ of an agent, the belief base $\Bh$ of a human user, and an explanandum $\varphi$ such that $\KBa \models \varphi$ and $P_h (\varphi) < 1$, $\pE = \langle \pe^{+}, \pe^{-} \rangle$ is a probabilistic model reconciling explanation if and only if $\pe^+ \subseteq \KBa$ and $\pe^- \subseteq \Bh^{\downarrow w}$, and $P_h(\varphi \: | \: \pe^+) > P_h (\varphi)$ and $P_h(\pe^+ \: | \: \lnot \pe^-) > P_h (\pe^+)$.
\end{definition}

A probabilistic model reconciling explanation $\pE = \langle \pe^{+}, \pe^{-} \rangle$ for $\varphi$ from $\KBa$ for $\Bh$ is a tuple that increases the degree of belief in $\varphi$ with $\pe^+$, as well as increasing the degree of belief in $\pe^+$ with $\pe^-$ if $\pe^+$ is inconsistent with $\Bh^{\downarrow w}$. For brevity, until the end of this section, we will refer to probabilistic model reconciling explanations $\pE$ as model reconciling explanations. In this context, the notion of explanatory gain takes the following form:

\begin{definition}[Explanatory Gain for Model Reconciling Explanations]
\label{def:expl-gain-lmrp}
Let $\pE = \langle \pe^+, \pe^- \rangle$ be a model reconciling explanation for explanandum $\varphi$ from $\KBa$ for $\Bh$. The explanatory gain of $\pE$ for $\varphi$ is defined as $G (\pE, \varphi) = \log \left (\frac{P(\varphi \: | \: \pe^+)}{P(\varphi)} \right) + \log \left ( \frac{P(\pe^+ \: | \: \lnot \pe^-)}{P(\pe^+)}\right )$.
\end{definition}

In essence, the explanatory gain of $\pE = \langle \pe^+, \pe^- \rangle$ for $\varphi$ evaluates to what extent $\pe^+$ increases the probability of $\varphi$, as well as the extent to which $\pe^-$ increases the probability of $\pe^+$, \textit{assuming that $\pe^-$ is false}.

\begin{example}
\label{pex:lmrp3}
Let $\pE = \langle \{a, \lnot a \vee b \}, \{ \lnot c \vee \lnot a\} \rangle$ be the model reconciling explanation for $b$ from $\KBa$ for $\Bh$ in Example~\ref{pex:lmrp2}. The explanatory gain of $\pE$ for $b$ is $G (\pE, b) = \log \left (\frac{1}{0.5} \right ) + \log \left ( \frac{0.5}{0.09} \right ) = 1 + 2.47 = 3.47$.
\end{example}

Similarly, the notion of explanatory power is defined in the following way:

\begin{definition}[Explanatory Power for Model Reconciling Explanations]
\label{def:expl-pow-lmrp}
Let $\pE = \langle \pe^+, \pe^- \rangle$ be a model reconciling explanation for explanandum $\varphi$ from $\KBa$ for $\Bh$. The explanatory power of $\pE$ for $\varphi$ is defined as $\EP (\pE, \varphi) = G_h (\pE, \varphi) + \gamma \cdot \left (P_h(\pe^+) + P_h(\pe^-) \right)$, where $\gamma \in [0,1]$ is a constant.
\end{definition}

This definition of explanatory power of $\pE = \langle \pe^{+}, \pe^{-} \rangle$ for $\varphi$ assesses, in addition to the explanatory gain of $\pE$, the likelihoods of $\pe^+$ and $\pe^-$, with $\gamma$ parameterizing their relative importance in the overall assessment. 

\begin{example}
\label{pex:lmrp4}
Continuing from Example~\ref{pex:lmrp3}, the explanatory power of $\pE = \langle \{a, \lnot a \vee b \}, \{ \lnot c \vee \lnot a \} \rangle$ for $b$ (for $\gamma = 0.5$) is $\EP (\pE, b) = 3.47 + 0.5 \cdot (0.09 + 0.90) = 3.96$ 
\end{example}

Finally, a preference relation and a most-preferred model reconciling explanation can be defined in the same manner as in Definition~\ref{def:prel} and Definition~\ref{def:mpref}, respectively.

\begin{definition}[Preference Relation for Model Reconciling Explanation]
    Let $\pE_1$ and $\pE_2$ be two model reconciling explanations for explanandum $\varphi$ from knowledge base $\KBa$ for beleif base $\Bh$. $\pE_1$ is preferred over $\pE_2$, denoted $\pE_1 \succeq \pE_2$, if and only if $\EP (\pE_1) \geq \EP (\pE_2)$.
\end{definition}

\begin{definition}[Most-Preferred Model Reconciling Explanation]
Let $\tilde{E}(\varphi)$ denote the set of all model reconciling explanations for explanandum $\varphi$ from knowledge base $\KBa$ for belief base $\Bh$. A model reconciling explanation $\pE^* \in \tilde{E}(\varphi)$ is the most-preferred model reconciling explanation for $\varphi$ if and only if $\pE^* \succeq \pE$ for all $\pE^ \in \tilde{E}(\varphi)$.
\end{definition}

\subsubsection{Motivating Application: Illustrative Example}

Continuing with our motivating example, we will now examine a case where the human's belief base differs from the robot's in a way that leads to divergent expectations about the robot's decision. Particularly, consider a deterministic version of the robot's belief base $\KBa$ as described in Section~\ref{sec:example}, and let us define the human's belief base $\Bh$ with the following key differences:

\begin{align*}
&(\lnot \texttt{crowded}(A), 1) \\
&(\texttt{crowded}(B), 3)
\end{align*}

The human user believes that corridor A is generally unlikely to be crowded and corridor B is very likely crowded---the opposite of the robot's beliefs. Additionally, the human believes that the robot can successfully move through corridors A and B, whether crowded or not. Given their belief base $\Bh$, the expected optimal plan from the human's perspective would be $\pi_h = \langle \texttt{move}(room1, room2, A)_0, \texttt{deliver}_1 \rangle$. When the human observes the robot taking corridor B instead, they would naturally question this decision.

The robot, aware of this discrepancy, can then generate a probabilistic model reconciling explanation $\pE = \langle \pe^+, \pe^- \rangle$ to address $\varphi$:

\begin{enumerate}
     \item[] $\pe^+ = \{\texttt{crowded}(A), \texttt{move}(room1, room2, A)_0 \wedge \texttt{crowded}(A) \rightarrow \texttt{robot}\textrm{-}\texttt{at}(room1)_1\}$
    \item[] $\pe^- = \{\lnot \texttt{crowded}(A) \}$

\begin{itemize}
    \item $P_h(\varphi | \pe^+) = 0.7$ (increased from $P_h(\varphi) = 0.4$)
    \item $P_h(\pe^+ | \neg \pe^-) = 1$ (increased from $P_h(\pe^+) = 0.15$)
    \item Explanatory gain: $G(\pE, \varphi) = 3.04$
    \item Explanatory power: $EP(\pE, \varphi) = 3.22$
\end{itemize}
\end{enumerate}

The explanation essentially communicates: ``I didn't take corridor A because it is very likely crowded, contrary to your belief that it's not crowded. Additionally, when a corridor is crowded, my movement through it will likely fail, causing me to remain in my original location.'' What we see here is that if we focus on both environmental factors (crowded corridor) and action dynamics (movement success), we get an explanation that provides a complete account of the robot's decision-making process, thereby helping the human gain a better understanding of the robot.

\section{Computing Explanations}
\label{sec:comp}

We now describe algorithms for computing explanations. We first review two algorithms proposed in our previous work \cite{vas21} for computing classical (deterministic) monolithic explanations (Definition~\ref{def:expl}) and model reconciling explanations (Definition~\ref{lmrp}), and then show how to extend them to the probabilistic case.

\subsection{Classical Explanations}
\label{sec:comp:clas}

We previously introduced an approach for computing minimum size monolithic explanations for an explanandum $\varphi$ from a knowledge base $\KB$, where $\KB \models \varphi$ \cite{vas21}. The principal idea of this approach is to reduce the problem of computing a monolithic explanation of minimum size to the one of computing a \emph{smallest minimal unsatisfiable set} (SMUS) over an inconsistent knowledge base. 

In particular, notice that, by definition, we have that $\KB \models \varphi$ if and only if $\KB \cup \{\lnot \varphi \}$ is inconsistent. Moreover, in Proposition~\ref{prop:support-mus}, we have already stated the relation between a monolithic explanation and a \emph{minimal unsatisfiable set} (MUS). This suggests that, in order to extract a monolithic explanation, we just need to run an MUS solver over the knowledge base $\KB^s \cup \{ \lnot \varphi^h \}$, where $\KB^s$ and $\varphi^h$ denote that $\KB$ and $\varphi$ are treated as soft and hard constraints, respectively, and then remove $\lnot \varphi$ from the returned MUS.\footnote{Recall that soft constraints may be removed by the MUS solver, while hard constraints will not be removed.} The hitting set duality relating MUSes and \emph{minimal correction sets} (MCSes) (see Lemma~\ref{lem:smus}) is a key aspect for the computation of an SMUS. 

\begin{algorithm*}[t!]
\DontPrintSemicolon
	\KwIn{Knowledge base $\KB$ and explanandum $\varphi$}
	\KwResult{A minimum size monolithic explanation $\epsilon$ for $\varphi$ from $\KB$}

        $\fml{H}\gets\emptyset$\; \label{alg1:initH}
	\While{true}{		
		$seed\gets \texttt{minHS}(\fml{H})$\label{alg1:hit} \tcp*{compute a minimal hitting set}
		$\epsilon \gets \{c_i \;|\;  i \in seed \}$\;\label{alg1:iseed}
		\uIf{\emph{\textbf{not}} \emph{\texttt{SAT}}$(\epsilon \cup \{ \lnot \varphi \})$}{ \label{alg1:stcheck}
			\Return $\epsilon$ \tcp*{minimum size monolithic explanation}
		}
		\Else{            
			$\C \gets \texttt{getMCS}(seed, \KB^s \cup \{ \lnot \varphi^h \})$ \tcp*{compute a minimal correction set} \label{alg1:mcs}
		} 
		$\fml{H}\gets \fml{H}\cup \{ \C \}$\;
	}
	\caption{\texttt{monolithic-explanation}$(\KB, \varphi)$}
	\label{alg:expl}
\end{algorithm*}

\begin{table}[t!]
	\caption{Example of Algorithm \ref{alg:expl} for computing a monolithic explanation of minimum size.}
	\label{table:alg:expl}
 \resizebox{1\linewidth}{!} {
	\begin{tabular}{ | l l l | }
		\hline
		 & $\KB = \{  \; \stackrel{C_1}{a \lor b} , \stackrel{C_2}{\lnot b \lor c} ,  \stackrel{C_3}{\lnot c}, \stackrel{C_4}{\lnot b \lor d} \}$ & \\
		 & $\KB \models a$ & \\[1mm]
		\hline
		 & & \\
		1. & $\fml{H} \gets \emptyset$ &  \\
		2. & $seed \gets \emptyset$ & \# $minHS(\fml{H})$ \\
		3. & $ \emptyset \not \models a$ & \# $SAT(\epsilon \cup \{\lnot  a \})$ \\
		4. & $\C \gets \{C_1\}$ & \# MCS computed on $KB^s \cup  \{\lnot a^h\}$ starting with the seed $seed$ \\
		5. & $\fml{H} \gets \{\{C_1\}\}$ & \\
		6. & $seed \gets \{C_1\}$ & \# $minHS(\fml{H})$ \\
		7. & $\{ a \vee b \} \not \models a$ & \# $SAT(\e \cup \{ \lnot a \})$ \\
		8. & $\fml{C} \gets \{C_2\}$ & \# MCS computed on $KB^s \cup \{\lnot a^h \}$ starting with the seed $seed$ \\
		9. & $\fml{H} \gets \{\{C_1\}, \{C_2\}\}$ & \\
		10. & $seed \gets \{C_1, C_2\}$ & \# $minHS(\fml{H})$ \\
		11. & $ \{a \vee b, \lnot b \vee c \} \not \models a$ & \# $SAT(\e \cup \{ \lnot a \})$ \\
            12. & $\C \gets \{C_3\}$ & \# MCS computed on $KB^s \cup  \{\lnot a^h\}$ starting with the seed $seed$ \\
            13. & $\fml{H} \gets \{\{C_1\},\{C_2\},\{C_3\}\}$ & \\
            14. & $seed \gets \{C_1, C_2, C_3\}$ & \# $minHS(\fml{H})$ \\
		15. & $\{ a \vee b, \lnot b \vee c, \lnot c \} \models a$ & \# $\lnot SAT(\e \cup \{ \lnot a \})$ \\
		16. & $Return$ $ \{a \vee b, \lnot b \vee c, \lnot c \}$ & \# minimum size monolithic explanation for $a$ from $\KB$ \\
		\hline
	\end{tabular}
	}
\end{table}


Algorithm \ref{alg:expl} describes the main steps of our approach. $\fml{H}$ is a collection of sets, where each set corresponds to an MCS on $\KB$. At the beginning, it is initialized with the empty set (line~\ref{alg1:initH}). Each MCS in $\fml{H}$ is represented as the set of the indexes of the formulae in it. $\fml{H}$ stores the MCSes computed so far.
At each step, a minimal hitting set on $\fml{H}$ is computed (line \ref{alg1:hit}). In line \ref{alg1:iseed}, the formulae induced by the computed minimal hitting set is stored in $\epsilon$. Then, $\epsilon \cup \{ \lnot \varphi \}$ is evaluated for satisfiability (line \ref{alg1:stcheck}). If $\epsilon \cup \{ \lnot \varphi \}$ is inconsistent, then $\epsilon$ is a monolithic explanation of minimum size and the algorithm returns $\epsilon$. If instead $\epsilon \cup \{ \lnot \varphi\}$ is consistent, then it means that $\epsilon \not \models \varphi$ and the algorithm continues in line \ref{alg1:mcs}. The computation of an MCS of this kind can be performed via standard MCS procedures~\cite{marques-silva-ijcai13}, using the set of formulae indexed by the \emph{seed} as the starting formula to extend. Since $\varphi$ is set to hard (line \ref{alg1:mcs}), the returned MCS $\C$ is guaranteed to be contained in $\KB$. Due to the hitting set duality relation, we will also have $\epsilon \subseteq \KB$. Finally, notice that the procedure $getMCS$ always reports a new MCS because, by construction, we have $seed \subseteq \KB \setminus \C$. In fact, the $seed$ contains at least one formula for each previously computed MCS and, thus, $seed \cap \C = \emptyset$ (i.e.,~at least one formula for each previously computed MCS is not in $\C$).

Algorithm~\ref{alg:expl} is complete in the sense that eventually a monolithic explanation $\e \subseteq \KB$ of minimum size such that $\e \models \varphi$ will be returned. This can be easily verified by observing
that every time $\e \cup \{\lnot \varphi\}$ is satisfiable, a new MCS is computed. Eventually, all the MCSes will be computed and, from Propositions~\ref{prop:duality} and~\ref{prop:support-mus}, it follows that a minimal hitting set on the collection of all MCSes corresponds to the smallest MUS, and as such, to a monolithic explanation of minimum size. 

Note that deciding whether there exists a monolithic explanation of size less or equal to $k$ is $\Sigma_{2}^{p}$-complete and extracting a smallest monolithic explanation is in $FP^{\Sigma_{2}^{p}}$. This follows directly from the complexity of deciding and computing an SMUS on which Algorithm \ref{alg:expl} is based on \cite{ignatiev-plm15}.



\begin{algorithm}[t!]
\DontPrintSemicolon
 \setcounter{AlgoLine}{0}   
	\KwIn{Knowledge bases $\KBa$ and $\KB_h$ and explanandum $\varphi$ }
	\KwResult{A model reconciling explanation $\fml{E} = \langle \e^+, \e^- \rangle$ for $\varphi$ from $\KBa$ for $\KB_h$}
	$\R \gets\emptyset$\;\label{alg:lmrp:rec}
	$\KBa^h \gets \KBa \cap \KB_h $\label{alg:lmrp:kbh} \;
	$\KBa^s \gets \KBa \setminus \KBa^h$\label{alg:lmrp:kbs} \;
	\If{\emph{\textbf{not}} \emph{\texttt{SAT}}$(\KB_h \cup \KBa)$\label{alg:lmrp:sat-check-1}}{
		$E^- \gets  \texttt{getMCS}((\KB_h \setminus \KBa)^{s} \cup \KBa^{h})$ \tcp*{restore consistency on $\KB_h$} \label{alg:lmrp:preprocessing-1}
		$\KB_{h} \gets \KB_{h} \setminus E^-$ \label{alg:lmrp:preprocessing-2}\;
	}
	\While{true}{\label{alg:lmrp:loop}
		$seed\gets \texttt{minHS}(\R)$\;\label{alg:lmrp:mhs}
		$\e^+ \gets \{c_i \; | \; i \in seed \}$\tcp*{explanation $\e^+$ induced by the seed}
	    \uIf{\emph{\textbf{not}} \emph{\texttt{SAT}}$(\KB_h \cup \e^+ \cup \{\lnot \varphi\})$}{\label{alg:lmrp:sat-check}
            $\e^- \gets \emptyset$ \;
	       \If{\emph{\textbf{not}} \emph{\texttt{SAT}}$(\KB_h \cup \e^+ \cup E^-)$}{\label{alg:lmrp:sat-check-e_minus}
                $\e^- \gets \texttt{getMCS}((\KB_h \cup \e^+)^h \cup (E^-)^s)$ \label{alg:lmrp:e_minus} \; 
           }	      
			\Return $\langle \e^+, \e^- \rangle$ \; \label{alg:lmrp:return}
		}
		\Else{
			$\C \gets \texttt{getMCS}(seed, \KBa^{h} \cup \{ \lnot \varphi^h \} \cup \KBa^s )$\;\label{alg:lmrp:mcs}
			$\R \gets \R \cup \{ \C \}$\;
		}
	}
	\caption{\texttt{model-reconciling-explanation}$(\KBa, \KBa^{h}, \varphi)$}
	\label{alg:lmrp}
\end{algorithm}

\begin{table*}[t!]
	\caption{Example of Algorithm \ref{alg:lmrp} for computing a model reconciling explanation.}
	\label{table:alg:lmrp}

 \resizebox{1\linewidth}{!} {
 \small
	\begin{tabular}{ | l l l | }
		\hline
		 & $\KBa =  \; \{ \stackrel{C_1}{(a \lor b)} , \stackrel{C_2}{(\lnot b \lor c)}, \stackrel{C_3}{\lnot c} , \stackrel{C_4}{(\lnot b \lor d)} , \stackrel{C_5}{\lnot d} \}$ & \multirow{2}{*}{\Bigg\} We have that $\KBa \models a$ and $\KB_h \not \models a$} \\
		 & $\KB_h = \; \{ \stackrel{D_1}{b},  \stackrel{D_2}{\lnot c} \}$ & \\[1mm]
		\hline
		1. & $\fml{R} \gets \emptyset$ &  \\
		2. & $\KBa^h \gets \KBa \cap \KB_h = \{C_3\}$ & \\
		3. & $\KBa^s \gets \KBa \setminus (\KBa \cap \KB_h) = \{C_1, C_2, C_4, C_5\}$ & \\
            4. & $E^- \gets \{D_1 \}$ & \# MCS computed on $(\KB_h \setminus \KBa)^s \cup \KBa^h$ \\
            5. & $\KB_h  \gets \{D_1, D_2\} \setminus \{ D_1 \} = \{D_2 \} $ & \\ 
            6. & $seed \gets \emptyset$ & \# $minHS(\fml{R})$ \\
		7. & $\{ \lnot c \} \not \models a$ & \# $SAT(\KB_h \cup \e^+ \cup \{ \lnot a \})$ \\
		8. & $\fml{C} \gets \{C_1\}$ & \# MCS computed on $\KBa^{h} \cup \KBa^s \cup \{ \lnot a^h \}$ \\
		9. & $\fml{R} \gets \{\{C_1\}\}$ & \\
		10. & $seed \gets \{C_1\}$ & \# $minHS(\fml{R})$ \\
		11. & $ \{ \lnot c, a \vee b\} \not \models a$ & \# $SAT(\KB_h \cup \e^+ \cup \{ \lnot a \})$ \\
		12. & $\fml{C} \gets \{C_2, C_4\}$ & \# MCS computed on $\KBa^{h} \cup \KBa^s \cup \{ \lnot a^h \}$\\
		13. & $\fml{R} \gets \{\{C_1\}, \{C_2, C_4\}\}$ & \\
		14. & $seed \gets \{C_1, C_4\}$ & \# $minHS(\fml{R})$ \\
		15. & $ \{ \lnot c, a \vee b, \lnot b \vee d\} \not \models a$ & \# $SAT(\KB_h \cup \e^+ \cup \{ \lnot a \})$ \\
		16. & $\fml{C} \gets \{C_2, C_5\}$ & \# MCS computed on $\KBa^{h} \cup \KBa^s \cup \{ \lnot a^h \}$ \\
            17. & $\fml{R} \gets \{\{C_1\}, \{C_2, C_4\}, \{C_2, C_5 \}\}$ & \\
            18. & $seed \gets \{C_1, C_2\}$ & \# $minHS(\fml{R})$ \\
            19. & $ \{ \lnot c, a \vee b, \lnot b \vee c\} \models a$ & \# $\lnot SAT(\KB_h \cup \e^+ \cup \{ \lnot a \})$ \\
            20. & $\e^- \gets \{D_1\}$ & \# MCS computed on $(\KB_h \cup \e^+)^h \cup (E^-)^s$ \\
            
		21. & $Return$ $\langle \{C_1, C_2\}, \{ D_1\} \rangle$ & \# model reconciling explanation for $a$ from $\KBa$ for $\KB_h$ \\
		\hline
	\end{tabular}
	}
\end{table*}

\subsubsection{Model Reconciling Explanations}
\label{sec:comp:clas-lmrp}


We have also previously showed how Algorithm~\ref{alg:expl} can be further extended for computing model reconciling explanations for an explanandum $\varphi$ from an agent knowledge base $\KBa$ for a human knowledge base $\KB_h$ \cite{vas21}. However, we only considered the specific task of finding a model reconciling explanation $\e \subseteq \KBa \cup \KB_h$ such that $\KB_h \cup \e \models \varphi$ and $\e \setminus \KB_h$ is of minimum size. Notice that, in general, $\KB_h \cup \e$ might be inconsistent. However, in our approach, we discard this possibility by preprocessing $\KB_h$. In particular, we create a new $\KB_h' \subseteq \KB_h$ by removing a minimal set of formulae in $\KB_h$ that makes $\KB_h \cup \KBa$ inconsistent. The new $\KB_h'$ is such that $\KB_h' \cup \e$ is always consistent.
 
We now modify this approach for computing model reconciling explanations $\fml{E} = \langle \e^+, \e^- \rangle$, where  $\e^+ \subseteq \KBa$ and $\e^- \subseteq \KB_h$, such that $(\KB_h \cup \e^+) \setminus \e^- \models \varphi$. Particularly, in addition to $\e^+$, our approach now computes $\e^-$ as well.

Algorithm \ref{alg:lmrp} describes the pseudocode of our approach. At the beginning of the algorithm, we initialize $\R$ to the null set (line~\ref{alg:lmrp:rec}). $\R$ is used to store the MCSes, which acts as a mediator between $\KBa$ and $\KB_h$. Lines~\ref{alg:lmrp:kbh}-\ref{alg:lmrp:kbs} are used to specify which clauses of $\KBa$ will be treated as hard and soft constraints, respectively. We then check if $\KB_h \cup \KBa$ is inconsistent (line~\ref{alg:lmrp:sat-check-1}). This is important in order to avoid the possibility of finding subsets $\e^+$ that explain why $\KB_h \cup \KBa$ is inconsistent instead of the target explanandum. In case $\KB_h \cup \KBa$ is inconsistent, we preprocess $\KB_h$ by removing from $\KB_h \setminus \KBa$ a minimal set of formulae causing the conflict (i.e.,~an MCS) (lines~\ref{alg:lmrp:preprocessing-1}-\ref{alg:lmrp:preprocessing-2}), where $E^-$ stores the set of potential formulae $\e^-$ to retract. The reconciliation procedure
starts in line \ref{alg:lmrp:loop}. The algorithm proceeds iteratively by computing a minimal hitting set on $\R$ and then testing for satisfiability the formulae $\e^+$ (lines~\ref{alg:lmrp:mhs}-\ref{alg:lmrp:sat-check}). 
The test checks whether adding $\e^+$ to $\KB_h$ is sufficient for entailing $\varphi$. If $\KB_h \cup \e^+ \cup \{\lnot \varphi\}$ is unsatisfiable, then $\KB_h \cup \e^+ \models \varphi$. In that case, the algorithm then checks whether $\KB_h \cup \e^+ \cup E^-$ is inconsistent, and if it is, it computes an MCS $e^-$ on $(\KB_h \cup \e^+)^h \cup (E^-)^s$ (lines \ref{alg:lmrp:sat-check-e_minus}-\ref{alg:lmrp:e_minus}). The model reconciling explanation $\langle \e^+, \e^- \rangle$ is then returned in line \ref{alg:lmrp:return}. Otherwise, the algorithm continues in line \ref{alg:lmrp:mcs}, where a new MCS is computed and added to $\R$.\footnote{Note that the algorithm is complete as it is based on Algorithm~\ref{alg:expl}, which is complete.}

Table~\ref{table:alg:lmrp} shows an example trace of Algorithm~\ref{alg:lmrp}.

\subsection{Probabilistic Explanations}
\label{sec:comp:prob}

We now show how the algorithms described in the previous section can be used for computing probabilistic monolithic explanations (Definition~\ref{def:pexpl}) and probabilistic model reconciling explanations (Definition~\ref{def:plmrp}).

\subsubsection{Monolithic Explanations}
\label{sec:comp:prob-expl}

Consider an explanandum $\varphi$ and a belief base $\B$. First, notice that if we assume that the classical projection of $\B$ entails the explanandum $\varphi$, that is $\B^{\downarrow w} \models \varphi$, then Algorithm~\ref{alg:expl} can directly be applied on $\B^{\downarrow w}$ and $\varphi$.\footnote{Recall that the classical projection of belief base $\B$ is the unweighted version of the set of formulae from $\B$.} In that case, Algorithm \ref{alg:expl} guarantees to find a monolithic explanation with maximum explanatory gain, since we know from Proposition~\ref{prop:explgain} that explanatory gain achieves its maximum value for $\varphi$ when the monolithic explanation entails $\varphi$. Nevertheless, this does not guarantee that the monolithic explanation will be the most-preferred one, that is, the one with the highest explanatory power (Definition~\ref{def:mpref}).

Obviously, a straightforward way of computing a most-preferred monolithic explanation is to use Algorithm~\ref{alg:expl} to enumerate all possible monolithic explanations for $\varphi$, and return the one that has the highest probability, which corresponds to the one with the highest explanatory power. But enumerating through all possible monolithic explanations and computing their probabilities can be computationally prohibitive, as even extracting a smallest monolithic explanation is in $FP^{\Sigma^p_2}$ \cite{ignatiev-plm15} and computing the probability of a formula is $\#P$-complete~\cite{roth1996hardness,chavira2008probabilistic}. We can, however, account for this high computational complexity by seeking for a monolithic 
explanation that is guaranteed to have a probability above a certain threshold.

First, the following lemma notes that for all possible monolithic explanations $\pe$ for explanandum $\varphi$, the following upper and lower probability bounds hold:

\begin{lemma}
\label{lem:k-bound}
    Let $\tilde{E}(\varphi)$ be the set of all monolithic explanations for explanandum $\varphi$ from belief base $\B$, where $\pe \models \varphi$ for all $\pe \in \tilde{E}(\varphi)$, and let $\omega_1$ be the most-probable world in which $\varphi$ is true. Then, for any $\pe \in \tilde{E}(\varphi)$, it holds that $P(\omega_1)\leq P(\pe) \leq P(\varphi)$.
\end{lemma}

\begin{proof}
    For the upper probability bound, since we assume that for all $\pe \in \tilde{E}(\varphi)$, $\pe \models \varphi$, then it must hold that for all $\pe \in \tilde{E}(\varphi)$, the worlds where $\pe$ is true are subsumed by the worlds where $\varphi$ is true (entailment property). This implies that for any $\pe \in \tilde{E}(\varphi)$, $P(\pe) \leq P(\varphi)$.
    
    For the lower bound, since $\omega_1$ is the most-probable world of $\varphi$, that is, the world where the highest number of formulae from $\B$ are satisfied, then all monolithic explanations for $\varphi$ must be true in $\omega_1$ (i.e.,~$\omega_1 \models \pe$). As such, for any $\pe \in \tilde{E}(\varphi)$, $P(\pe) \geq P(\omega_1)$.
\end{proof}

However, some monolithic explanations may have a higher lower probability bound. Formally, we call such explanations \textit{$k$-bounded monolithic explanations}:


\begin{definition}[$k$-Bounded Monolithic Explanation]
\label{def:k-bound}
    Let $\tilde{E}(\varphi)$ be the set of all monolithic explanations for explanandum $\varphi$ from belief base $\B$. Let $\Omega(\varphi) = \{\omega_1, \ldots, \omega_n \}$ be the set of possible worlds in which $\varphi$ is true, where $P(\omega_1) \geq P(\omega_2) \geq \ldots \geq P(\omega_n)$. Also let $I_k = \overset{k}{ \underset{i=1}{{\mathlarger{\bigcap}}}} \{ \phi \: | \: \phi \in \B^{\downarrow w}, \omega_i \models \phi \}$ be the intersection of formulae that are true in worlds $\omega_1$ to $\omega_k$.
    We say that $\pe \in \tilde{E}(\varphi)$ is a $k$-bounded monolithic explanation for $\varphi$ from $\B$, with lower bound $P(\pe) \geq \overset{k}{ \underset{i=1}{{\mathlarger{\sum}}}} P(\omega_i)$, if and only if $\pe \subseteq I_k$.
\end{definition}

\begin{example}
\label{cex:1}
Consider the belief base $\B = \{(a, 1), (\lnot a \vee b, 3), (c, 2), (\lnot c \vee b, 1)\}$ and explanandum $b$. The two monolithic explanations for $b$ from $\B$ that entail $b$ are $\pe_1 = \{a, \lnot a \vee b \}$ and $\pe_2 = \{c, \lnot c \vee b \}$, where $P(\pe_1) = 0.64$ and $P(\pe_2) = 0.77$. Notice that there are four possible worlds in which $b$ is true: $\omega_1= \{a = T, b=T, c=T \},$ $\omega_2 = \{a = F, b=T, c=T \}$, $\omega_3 = \{a = T, b=T, c=F \}$, and $\omega_4 = \{a = F, b=T, c=F \}$, where $P(\omega_1) = 0.57$, $P(\omega_2) = 0.20$, $P(\omega_3) = 0.07$, and $P(\omega_4) = 0.02$. The maximum number of intersections that entail $b$ is $k=2$ (i.e.,~$I_2 = \{\lnot a \vee b, c, \lnot c \vee b \}$). Indeed, $\pe_2 \subseteq I_2$ and $P(\pe_2) = 0.77 = P(\omega_1) + P(\omega_2)$. Finally, notice how $\pe_2$ is also the most-preferred monolithic explanation for $b$ from $\B$; for $\gamma = 0.5$, $\EP(\pe_2, b) = 0.57 > \EP(\pe_1, b) = 0.50$.
\end{example}

\begin{proposition}
\label{prop:bexpl}
Let $\B$ be a belief base and $\varphi$ an explanandum. A $1$-bounded monolithic explanation $\pe$ for $\varphi$ from $\B$ always exists.
\end{proposition}

\begin{proof}
The proof follows directly from Lemma~2.
\end{proof}

Interestingly, there also exists a maximal $k$-bounded monolithic explanation.

\begin{corollary}
\label{cor-k-bound}
    If $I_k \models \varphi$ and $I_{k+1} \not \models \varphi$, then $\exists \pe \subseteq I_k$ with maximal lower bound $P(\pe) \geq P(I_k)$
\end{corollary}
\begin{proof}
    First, notice that if $I_k \models \varphi$ and $I_{k+1} \not \models \varphi$, then $I_{k+j} \not \models \varphi$ for all $j=1,\ldots, n-k$. As such, $k$ is the maximum number of intersections (from $\omega_1$ to $\omega_k$) such that $I_k \models \varphi$. Thus, since $\pe \models \varphi$ for all $\pe \in \tilde{E}(\varphi)$, it must be the case that there exists at least one $\pe$ such that $\pe \subseteq I_k$, from which we know that $P(I_k) \leq P(\pe)$. Moreover, as $I_k$ is the set of formulae that are true in worlds $\omega_1$ to $\omega_k$, its probability must be at least equal to the sum of the probabilities of these worlds (i.e.,~$P(I_k) \geq \overset{k}{ \underset{i=1}{{\mathlarger{\sum}}}} P(\omega_i)$). Therefore, $P(\pe) \geq P(I_k) \geq \overset{k}{ \underset{i=1}{{\mathlarger{\sum}}}} P(\omega_i)$, meaning that the probability of $\pe$ has a maximal lower bound by the top $k$ most-probable worlds of~$\varphi$.
\end{proof}


\noindent The utility of a $k$-bounded monolithic explanation in computing probabilistic monolithic explanations can be described as follows. If we take the top $k$ most-probable worlds in which the explanandum $\varphi$ is true, then we can prune the search space of possible monolithic explanations by taking the intersection of those worlds and checking if it entails $\varphi$---if it does, then we know that at least one monolithic explanation must be true in that world with probability at least equal to the sum of the probabilities of these top $k$ worlds. Building on this, we now present an algorithm for computing $k$-bounded monolithic explanations for $\varphi$ from $\B$, where we use Algorithm~\ref{alg:expl} as our core monolithic explanation generation engine.

\begin{algorithm*}[t!]
\DontPrintSemicolon
 \setcounter{AlgoLine}{0}   

	\KwIn{Belief base $\B$, explanandum $\varphi$, and user-defined parameter $\hat{k}$}
	\KwResult{A $k$-bounded monolithic explanation $\pe$ for $\varphi$ from $\B$ for some $k \leq \hat{k}$}
 
        $k \gets \hat{k}$ \;
         $\Omega_\varphi  \gets \texttt{getTopKWorlds}(\B \cup \{(\varphi, \infty)\}, k)$ \tcp*{find candidate set of formulae } \label{alg:pexpl:wsat}
         \While{true}{ \label{alg:pexpl:main}         
         \tcp{get intersecting formulae from top $k$ worlds of $\varphi$}
            $I_k \gets \texttt{getIntersections}(\B^{\downarrow w}, \Omega_\varphi, k)$ \; \label{alg:pexpl:Ik}
            \uIf{\emph{\textbf{not}} \emph{\texttt{SAT}}$(I_k \cup \{ \lnot \varphi \})$}{ \label{alg:pexpl:sat}
             $\pe \gets$ \texttt{monolithic-explanation}$(I_k, \varphi)$ \; \label{alg:pexpl:alg1}
                \Return $\pe$ \; \label{alg:pexpl:ret}
            
            }
            \Else{
                $k \gets k-1$ \label{alg:pexpl:k}
            
            }

         }
	\caption{\texttt{probabilistic-monolithic-explanation}$(\B, \varphi, \hat{k})$}
	\label{alg:pexpl}
\end{algorithm*}

Algorithm~\ref{alg:pexpl} describes the main steps of our approach. The important factor is the user-defined parameter $\hat{k}$, which dictates the number of worlds of $\varphi$ to be considered. It is an integer with range $1 \leq \hat{k} \leq |\Omega(\varphi)|$, where $\Omega(\varphi)$ is the set of all possible worlds of $\varphi$. Intuitively, the larger the $\hat{k}$, the more exhaustive the search will be as more worlds will be considered. The algorithm starts in line~1 with $k$ taking the user-defined value $\hat{k}$, and then proceeds to line~\ref{alg:pexpl:wsat}, where it uses a weighted MaxSAT solver to find the top $k$ most-probable worlds of $\varphi$. Note that $(\varphi, \infty)$ denotes that $\varphi$ is added to the solver as a hard constraint. The main loop of the algorithm starts in line~\ref{alg:pexpl:main}. In line~\ref{alg:pexpl:Ik}, \texttt{getIntersections} extracts the set of intersecting formulae $I_k$ from $\B^{\downarrow w}$ that are true in worlds $\omega_1$ to $\omega_k$. If $I_k \models \varphi$, then we know that a monolithic explanation is in $I_k$ and the algorithm proceeds to use Algorithm~\ref{alg:expl} with $I_k$ and $\varphi$ as inputs to compute and return a monolithic explanation (lines~\ref{alg:pexpl:sat}-\ref{alg:pexpl:ret}). Otherwise, the algorithm discounts $k$ by $1$ and repeats the process until a suitable $k$ is found.

Algorithm~\ref{alg:pexpl} is complete in the sense that, eventually, a monolithic explanation will be returned.

\begin{theorem}
    Algorithm~\ref{alg:pexpl} is guaranteed to terminate with a solution.
\end{theorem}

\begin{proof}
 The proof rests on the fact that, in the worst case, the parameter $k$ will reach a value of $1$. This will then correspond to the most-probable world of $\varphi$, which entails all possible monolithic explanations for $\varphi$. From Lemma~2, we know that the most-probable world of $\varphi$ entails all possible monolithic explanations for $\varphi$, that is, for any $\tilde{E}(\varphi)$, $\omega_1 \models \pe$, and $\pe \subseteq I_1$. Therefore, as Algorithm \ref{alg:pexpl} uses $I_1$ as an input to Algorithm~\ref{alg:expl}, which is guaranteed to return a solution, the algorithm is also guaranteed to terminate with a solution.
\end{proof}

\begin{theorem}
    Algorithm~\ref{alg:pexpl} is guaranteed to return a maximal $k$-bounded monolithic explanation if the user-defined parameter $\hat{k}$ is initialized to $|\Omega(\varphi)|$.
\end{theorem}

\begin{proof}
    First, note that if the user-defined parameter is initialized to $\hat{k} = |\Omega(\varphi)|$, then Algorithm~\ref{alg:pexpl} will perform an exhaustive and iterative search, starting from $k=|\Omega(\varphi)|$, to find $I_k$, such that $I_k \models \varphi$, and use it in Algorithm~\ref{alg:expl}. Now, as the algorithm discounts $k$ by 1 at each new iteration, eventually it will be the case that $I_k \models \varphi$ and $I_{k+1} \not \models \varphi$. From Corollary~\ref{cor-k-bound}, we then know that $\exists \pe \subseteq I_k$ such that $P(\pe) \geq P(I_k) \geq \overset{k}{ \underset{i=1}{{\mathlarger{\sum}}}} P(\omega_i)$, which means that $\pe$ corresponds to a $k$-bounded monolithic explanation. Therefore, the algorithm is guaranteed to return a maximal $k$-bounded monolithic explanation for $\varphi$.
\end{proof}

\subsubsection{Model Reconciling Explanations}
\label{sec:comp:prob-lmrp}

We now move on to the case of computing probabilistic model reconciling explanations $\pE = \langle \pe^+, \pe^- \rangle$ for an explanandum $\varphi$ from an agent knowledge base $\KBa$ for a human belief base $\Bh$. Similarly to what we described in Section~\ref{sec:comp:prob-expl}, Algorithm~\ref{alg:lmrp} can directly be used on $\KBa$ and $\Bh^{\downarrow w}$ for computing model reconciling explanations. Additionally, the concept of a $k$-bounded explanation (Definition~\ref{def:k-bound}) can also be used to guarantee a lower bound on the probability of $\pe^+$.

Algorithm~\ref{alg:plmrp} shows the pseudocode of our approach. The initial computational steps are similar to those in Algorithm~\ref{alg:pexpl}, with the exception that $\KBa$ is now also considered in the computation of the most-probable worlds of the $\varphi$. Specifically, in line~\ref{alg:plmrp:Ba}, $\KBa$ is converted into a belief base $\Ba$ where each formula is given a weight that is larger than the sum of weights of $\Bh$. This is to enforce these formulae to be true in the worlds of the explanandum $\varphi$. Then, $\Ba$ is used in conjunction with $\Bh$ to compute the top $k$ most-probable worlds of $\varphi$ (line \ref{alg:plmrp:wsat}). The algorithm proceeds in line \ref{alg:plmrp:Ik} to extract formulae from $\KBa$ that are true in the first $k$ intersections of the worlds of $\varphi$. If they entail $\varphi$, the algorithm then proceeds to compute a model reconciling explanation by invoking Algorithm~\ref{alg:lmrp} (lines \ref{alg:plmrp:SAT}-\ref{alg:plmrp:lmrp}). Otherwise, the algorithm continues by discounting $k$ by $1$ and repeats the process.

Note that Algorithm~\ref{alg:plmrp} is complete and correct as it is based on Algorithms~\ref{alg:lmrp}~and~\ref{alg:pexpl}, which are complete and correct.


\begin{algorithm}[t!]
\DontPrintSemicolon
 \setcounter{AlgoLine}{0}   
	\KwIn{Knowledge base $\KBa$, belief base $\Bh$, explanandum $\varphi$, and user-defined parameter $\hat{k}$ }
	\KwResult{A probabilistic model reconciling explanation $\pE = \langle \pe^+, \pe^- \rangle$ for $\varphi$ from $\KBa$ for $\Bh$} 

         $k \gets \hat{k}$ \;
         $\KBa^h \gets \KBa \cap \Bh^{\downarrow w}$ \label{alg:plmrp:kbs} \; 
         $W \gets \overset{n}{ \underset{i=1}{\sum}} \{ w_i \: | \: (\phi_i, w_i) \in \Bh \}$ \;
         $\Ba \gets \{ (\phi, W) \: | \: \phi \in \KBa \setminus \KBa^h \}$ \label{alg:plmrp:Ba}\;
        $\Omega_\varphi  \gets \texttt{getTopKWorlds}(\Bh \cup \Ba \cup \{(\varphi, \infty)\}, k)$\; \label{alg:plmrp:wsat}

         \While{$true$}{ \label{alg:plmrp:main}
            $I_k \gets \texttt{getIntersections}(\KBa, \Omega_\varphi, k)$ \; \label{alg:plmrp:Ik}
            \uIf{\emph{\textbf{not}} \emph{\texttt{SAT}}$((I_k\cup \KBa^h \cup \{\lnot \varphi \} ))$}{ \label{alg:plmrp:SAT}
                $\langle \pe^+, \pe^- \rangle \gets \texttt{model-reconciling-explanation}(I_k\cup \KBa^h, \Bh^{\downarrow w}, \varphi)$ \label{alg:plmrp:lmrp} \;
                \Return $\langle \pe^+, \pe^- \rangle$
            }
           \Else{
                $k \gets k -1$ \;
            
           }
           }
	

	
	\caption{\texttt{probabilistic-model-reconciling-explanation}$(\KBa, \Bh, \varphi, \hat{k})$}
	\label{alg:plmrp}
\end{algorithm}

\section{Experimental Evaluations}
\label{sec:exper}

This section presents a comprehensive evaluation of the proposed algorithms, assessing their effectiveness and efficiency across a range of scenarios. 


\subsection{Experimental Setup}
Experiments were conducted on a system equipped with an M1 Max processor and 32GB of memory. The algorithms were implemented in Python, utilizing the PySAT toolkit~\cite{imms-sat18} for SAT solving, MCS/MUS finding, weighted MaxSAT, and minimal hitting set computations. The time limit for all experiments was set to $500s$.\footnote{Code repository: \url{https://github.com/YODA-Lab/Probabilistic-Monolithic-Model-Reconciling-Explanations}.}

For our benchmarks, we selected a diverse set of problem instances:
\squishlist
    \item \textbf{Classical Planning Problems}: We encoded classical planning problems from the International Planning Competition (IPC) in the style of \citet{kautz1996encoding}, and used them as knowledge bases. The explanandum for each problem was the plan optimality query, which we constructed as described by \citet{vasLMRP}.

    \item \textbf{Agent Scheduling Problems}: We encoded logic-based agent scheduling problems based on the description by \citet{vasileioua2023lasp}, and used them as the knowledge bases. The explanandum for each problem was a set of unsatisfied agent constraints. 

    \item \textbf{Random CNF Problems}: We generated random CNF formulae as knowledge bases using CNFgen \cite{lauria2017cnfgen}. The explanandum for each problem was a conjunction of backbone literals,\footnote{The backbone literals of a propositional knowledge base are the set of literals entailed by the knowledge base.} which we computed using the minibones algorithm proposed by \citet{janota2015algorithms}. 
\squishend

\noindent  Note that we created associated belief bases for each problem by simply adding a random weight to each formula in the knowledge base.

\subsection{Results and Discussion}

We now describe and discuss our experimental results, first for monolithic explanations and then for model reconciling explanations.

\subsubsection{Monolithic Explanations}

We evaluated Algorithm~\ref{alg:pexpl}, referred to as {\sc alg3}, on computing probabilistic monolithic explanations. Since the core monolithic explanation generation engine of {\sc alg3} is powered by Algorithm~\ref{alg:expl}, referred to as {\sc alg1}, we also evaluate its performance on the same instances. These experiments aim to answer the following questions:
\begin{itemize}
    \item[] \textbf{Q1}: What is the performance of {\sc alg3} on computing monolithic  explanations across different problem instances?
    \item[] \textbf{Q2}: Does the efficacy of {\sc alg3} change under different values of the user-defined parameter $\hat{k}$?
\end{itemize}

\begin{table}[!t]
\small
    \centering
    \begin{tabular}{|c||c|c|c||c|c|c||c|c|c|}
    \hline
    \multirow{2}{*}{Parameter $\hat{k}$} & \multicolumn{3}{c||}{Planning} & \multicolumn{3}{c||}{Scheduling} & \multicolumn{3}{c|}{Random CNF} \\
    \cline{2-10}
     & S & T/O & Runtime & S & T/O & Runtime & S & T/O & Runtime \\
    \hline \hline
    1 & 28 & 9 & 82.0s & 30 & 5 & 80.0s& 25 & 5 & 12.4s \\
    \hline
    50 & 32 & 5 & 79.0s& 30 & 5 & 53.8s& 25 & 5 &8.6s \\
    \hline
    100 & 31 & 6 & 49.6s& 30 & 5 & 44.7s & 25 & 5 &5.8s\\
    \hline
    150 & 31& 6 & 45.5s& 30 & 5 & 38.0s & 25 & 5 &3.4s\\
    \hline
    200 & 31 & 6 & 45.2s& 30 & 5 & 37.2s & 25 & 5 &1.6s\\
    \hline
    \end{tabular}
    \caption{Number of Instances Solved (S) vs. Timed Out (T/O) by {\sc alg1} ($\hat{k}=1$) and {\sc alg3} ($\hat{k}=50$, $\hat{k}=100$, $\hat{k}=150$, $\hat{k}=200$). Note that Runtime denotes the average runtime over all solved instances.}
    \label{tab:pexpl}
\end{table}

\begin{figure}[!t]
    \centering
    \begin{subfigure}[b]{0.48\textwidth}
        \centering
        \includegraphics[width=\textwidth]{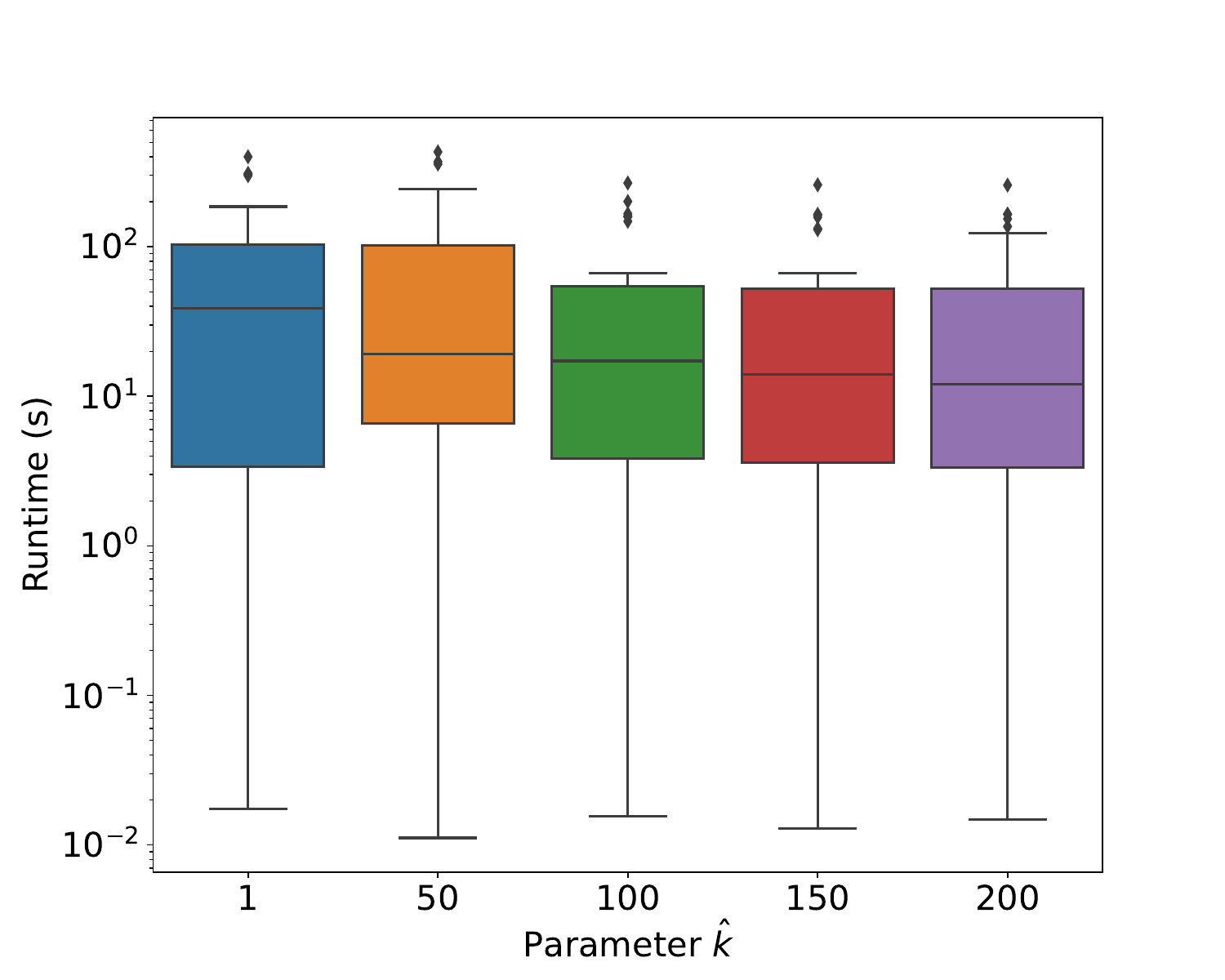}
        \caption{Planning Instances.}
    \label{fig:Tvsk:plan}
    \end{subfigure}
    \hfill 
    \begin{subfigure}[b]{0.48\textwidth}
        \centering
        \includegraphics[width=\textwidth]{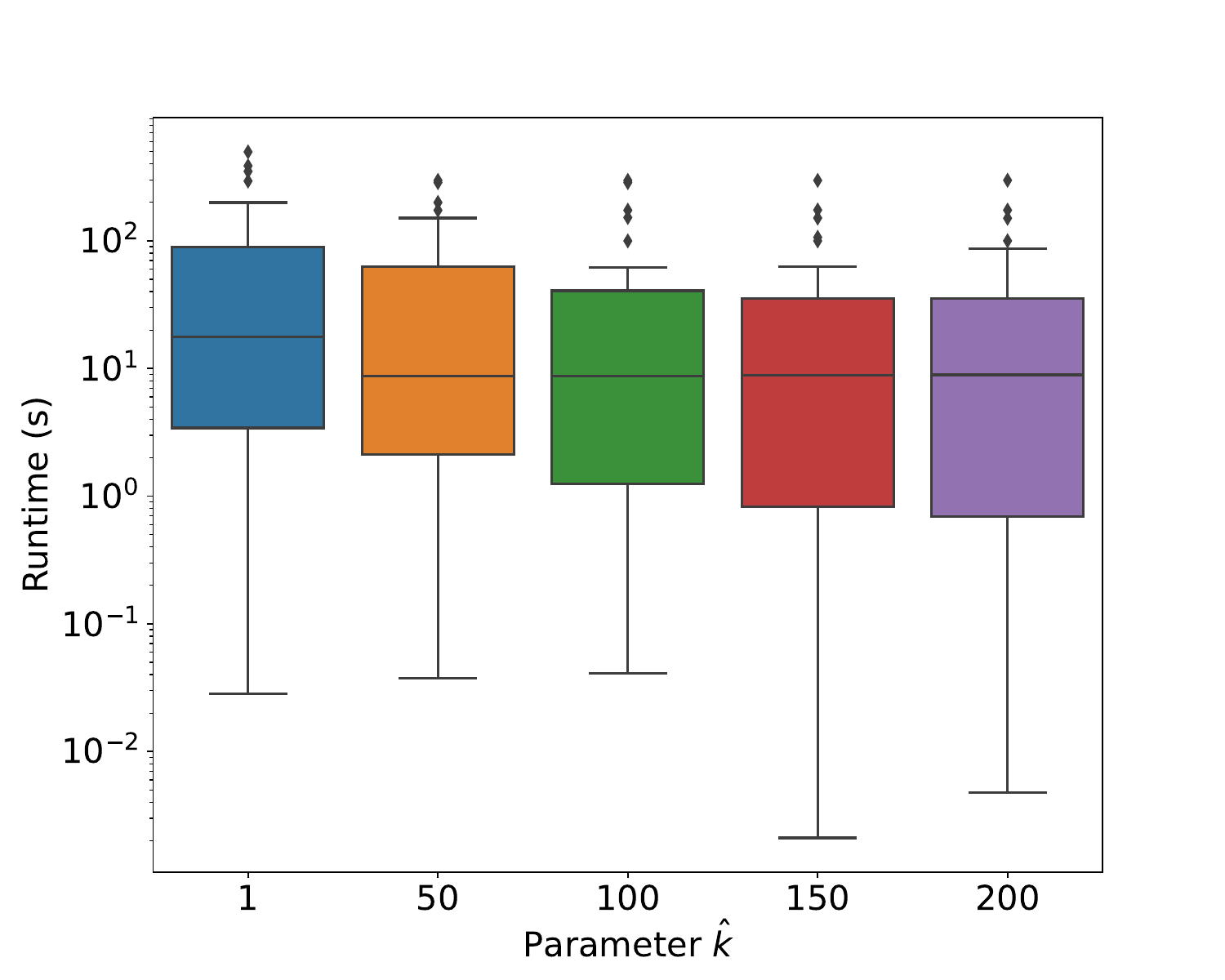}
        \caption{Scheduling Instances.}
    \label{fig:Tvsk:sch}
    \end{subfigure}
    
    
    \begin{subfigure}[b]{0.5\textwidth} 
        \centering
        \includegraphics[width=\textwidth]{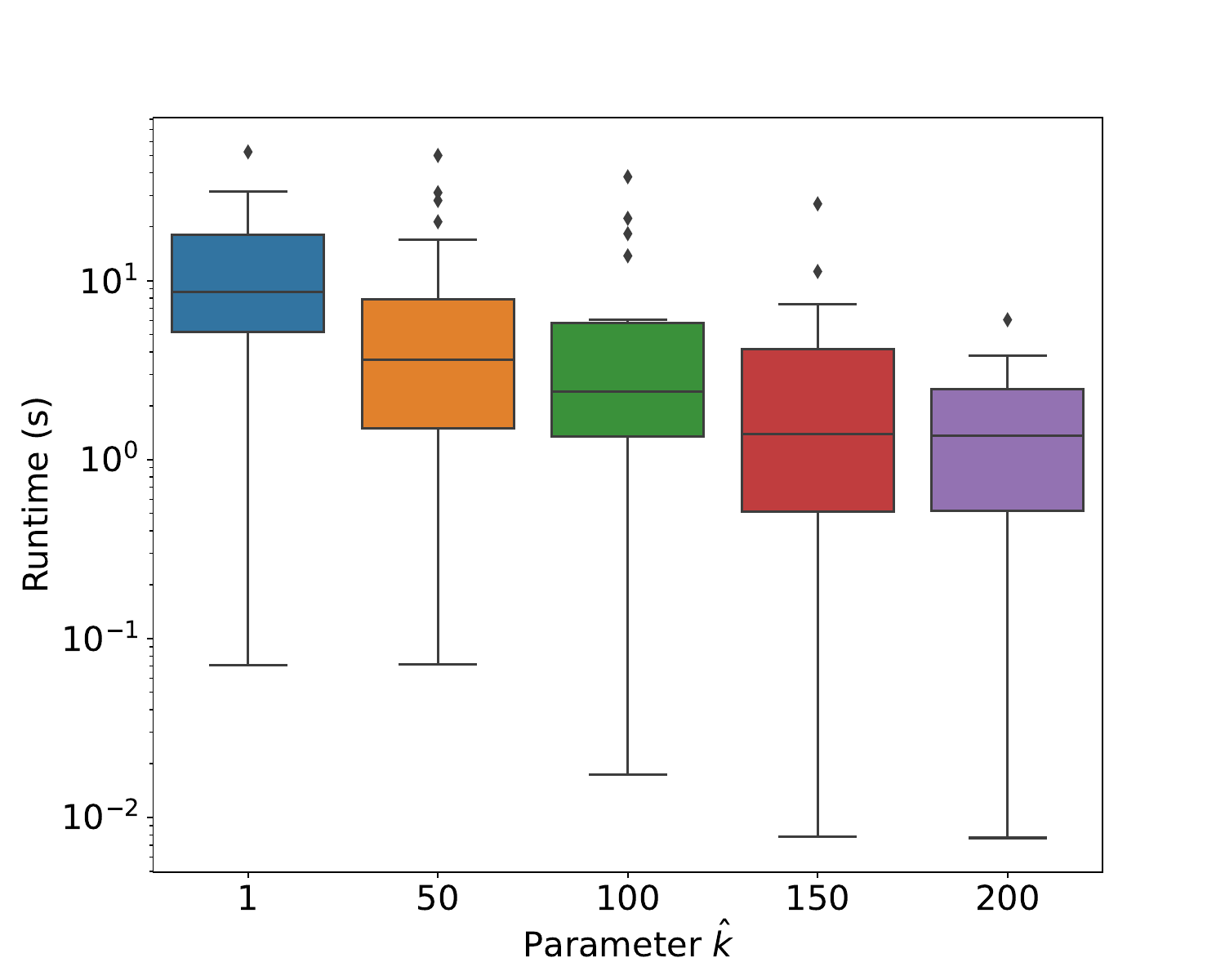}
        \caption{Random Instances.}
    \label{fig:Tvsk:rand}
    \end{subfigure}
    
    \caption{Runtime distributions of {\sc alg1} ($\hat{k}=1$) and {\sc alg3} ($\hat{k}=50$, $\hat{k}=100$, $\hat{k}=150$, $\hat{k}=200$) across all planning, scheduling, and random CNF instances.}
    \label{fig:Tvsk}
\end{figure}

Table~\ref{tab:pexpl} tabulates the instances solved (i.e.,~found a monolithic explanation within the time limit) and not solved (i.e.,~timed out) by {\sc alg1} ($\hat{k}=1$) and {\sc alg3} at $\hat{k}=\{5,100,150,200\}$.\footnote{{\sc alg3} at $\hat{k}=1$ corresponds to {\sc alg1} because each encoded knowledge base is consistent and entails the explanandum. As such, all formulae in the knowledge base are true in the most-probable world of the explanandum (i.e.,~$\hat{k}=1$), which means that {\sc alg3} reduces to {\sc alg1}.} We observe that the algorithm managed to solve most instances across different values of $\hat{k}$. Figure~\ref{fig:Tvsk} shows the runtime distributions of {\sc alg3} across all values of $\hat{k}$ for computing a monolithic explanation. Interestingly, we observe that the runtimes decrease as $\hat{k}$ increases. This can be explained by the fact that for larger values of $\hat{k}$, {\sc alg3} considers the intersections of more worlds where the explanandum is true, which means that the number of formulae that are true in these intersections decreases. As such, the overall search space of monolithic explanations decreases as well, thus resulting in a reduced runtime needed for {\sc alg1} to compute a monolithic explanation. This can also be observed more granularly in Figure~\ref{fig:time_common_pexpl}, where we can see the runtime distributions of {\sc alg1} ($\hat{k}=1$) and {\sc alg3} at $\hat{k}=200$ for each instance of the planning, scheduling, and random CNF problems. Again, the runtime of {\sc alg3} at $\hat{k}=200$ is smaller than that of {\sc alg1}. Moreover, and as expected, in Figure~\ref{fig:TvsKB}, we can observe a positive correlation between runtime and the size of the encoded knowledge bases---as the size of the knowledge base increases, the runtimes increase as well. This is due to the fact that there is an increasing number of variables and formulae that must be considered, thus increasing the computational effort needed by the WMaxSAT, MCS, and hitting set solvers.

All of these observations indicate the feasibility and practical efficacy of {\sc alg3} across all benchmarks. In particular, from these experiments, we may conclude that the performance of {\sc alg3} increases as the user-defined parameter $\hat{k}$ increases. To reiterate, this is mainly because the overall search space of monolithic explanations that needs to be considered by {\sc alg1} (the main monolithic explanation generation engine) decreases. Finally, it is important to note that the performance of these algorithms lies in the effectiveness of the underlying WMaxSAT, MCS, and hitting set solvers. In other words, this also implies that any advancement in those solvers will automatically reflect in performance gains in our algorithms. Thus, future work can look at efficient and optimized solvers and examine whether there is any variability in performance.

\begin{figure}[!t]
    \centering
    \begin{subfigure}[b]{0.49\textwidth}
        \centering
        \includegraphics[width=\textwidth]{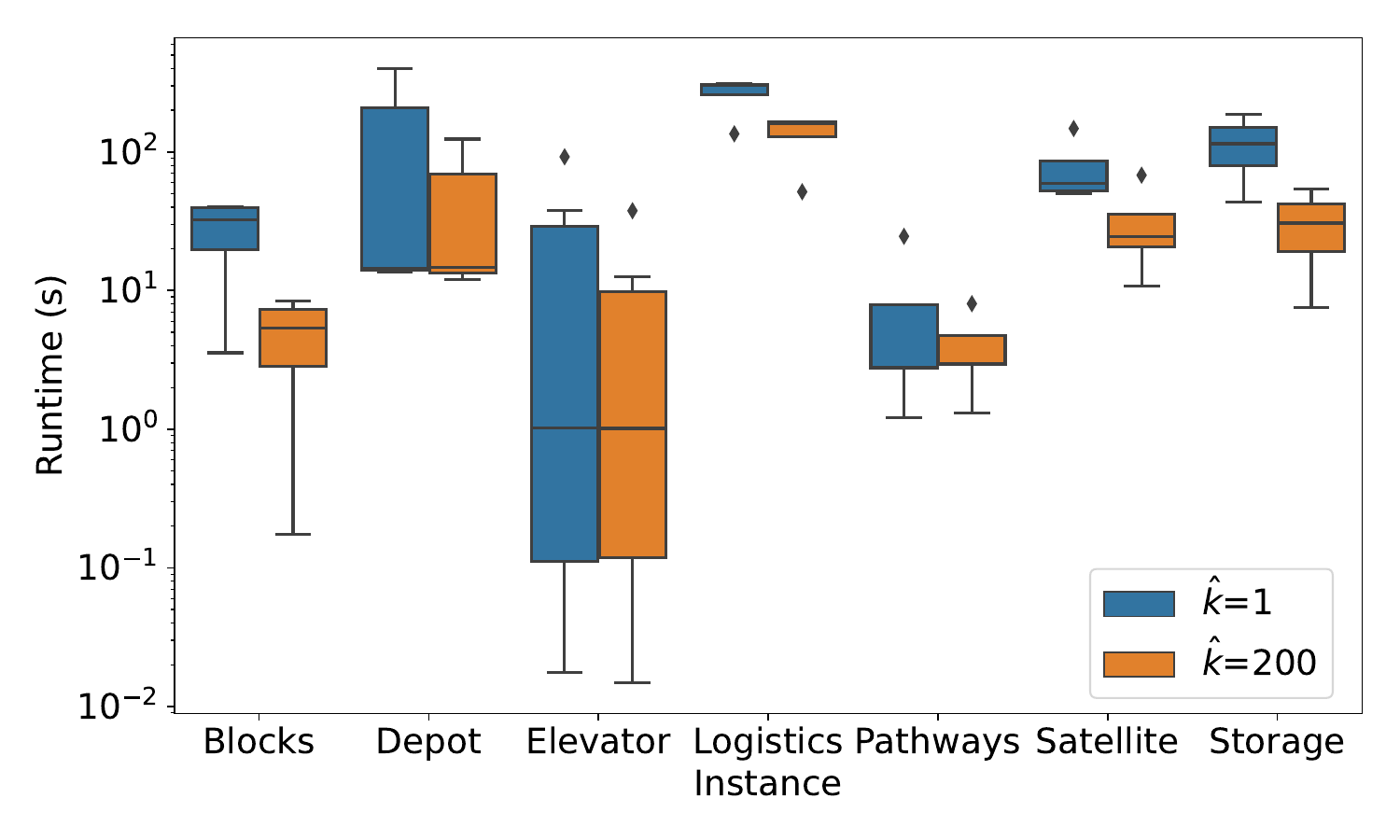}        \caption{Planning Instances.}
    \label{fig:Tvsk:plan}
    \end{subfigure}
    \hfill 
    \begin{subfigure}[b]{0.49\textwidth}
        \centering
        \includegraphics[width=\textwidth]{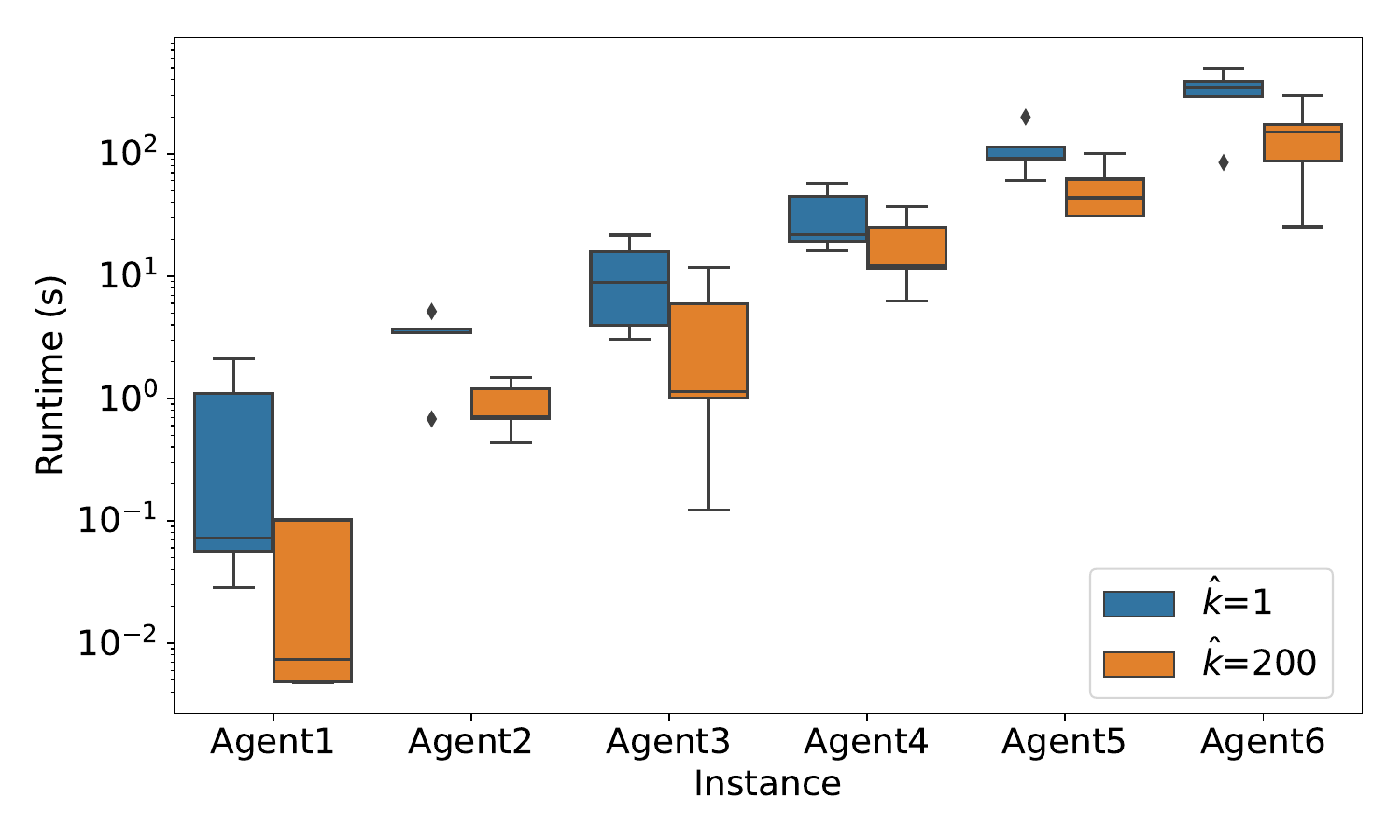}        \caption{Scheduling Instances.}
    \label{fig:Tvsk:sch}
    \end{subfigure}
    
    \vspace{0.2cm}
    
    \begin{subfigure}[b]{0.5\textwidth} 
        \centering
        \includegraphics[width=\textwidth]{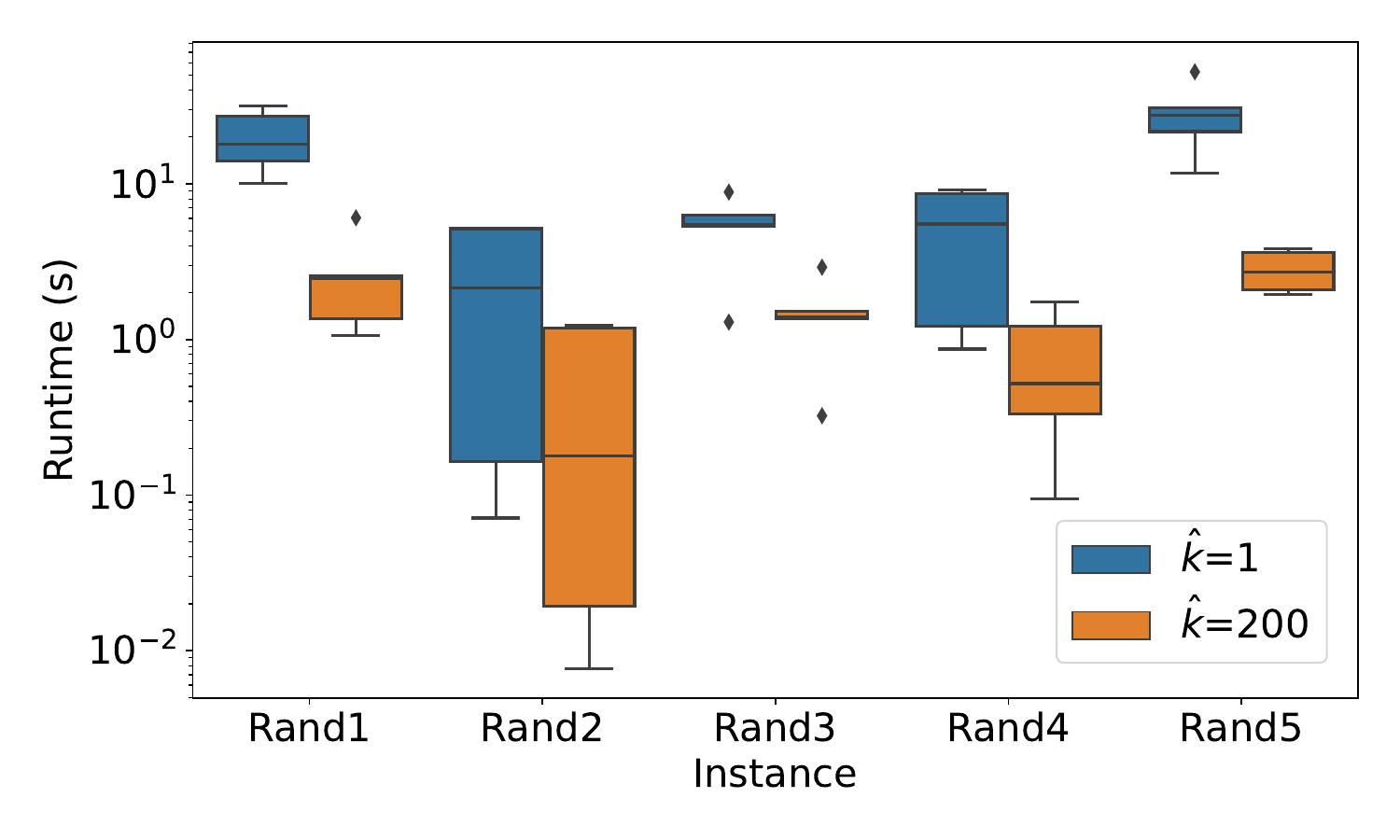}
        \caption{Random Instances.}
    \label{fig:Tvsk:rand}
    \end{subfigure}
    
    \caption{Runtime distributions of {\sc alg1} ($\hat{k}=1$) and {\sc alg3} ($\hat{k}=200$) across commonly solved planning, scheduling, and random CNF instances.}
    \label{fig:time_common_pexpl}
\end{figure}

\begin{figure}[!t]
    \centering
    \begin{subfigure}[b]{0.49\textwidth}
        \centering
        \includegraphics[width=\textwidth]{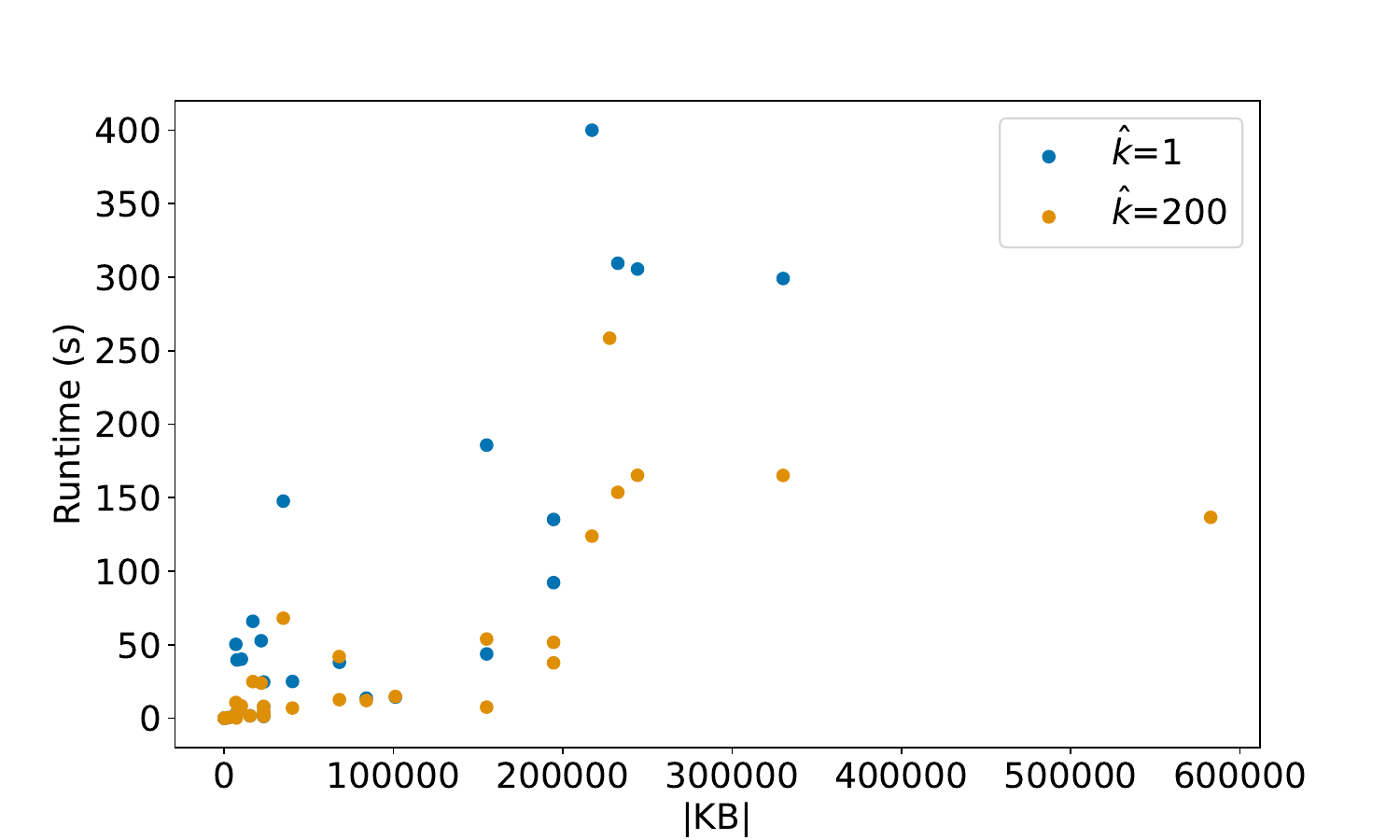}
        \caption{Planning Instances.}
        \label{fig:TvsKB:plan}
    \end{subfigure}
    \hfill 
    \begin{subfigure}[b]{0.49\textwidth}
        \centering
        \includegraphics[width=\textwidth]{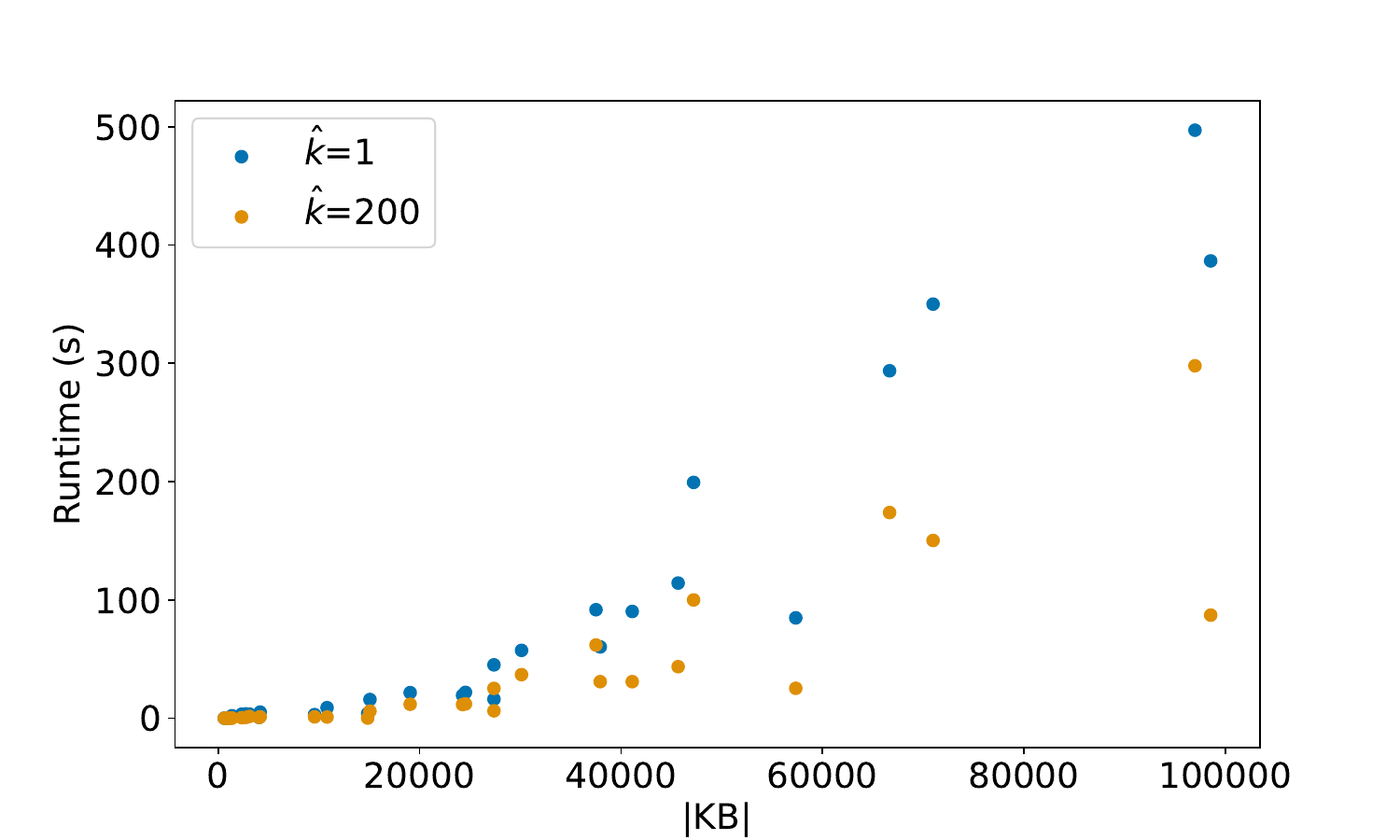}
        \caption{Scheduling Instances.}
               \label{fig:TvsKB:sch}
    \end{subfigure}
    
    
    \begin{subfigure}[b]{0.5\textwidth} 
        \centering
        \includegraphics[width=\textwidth]{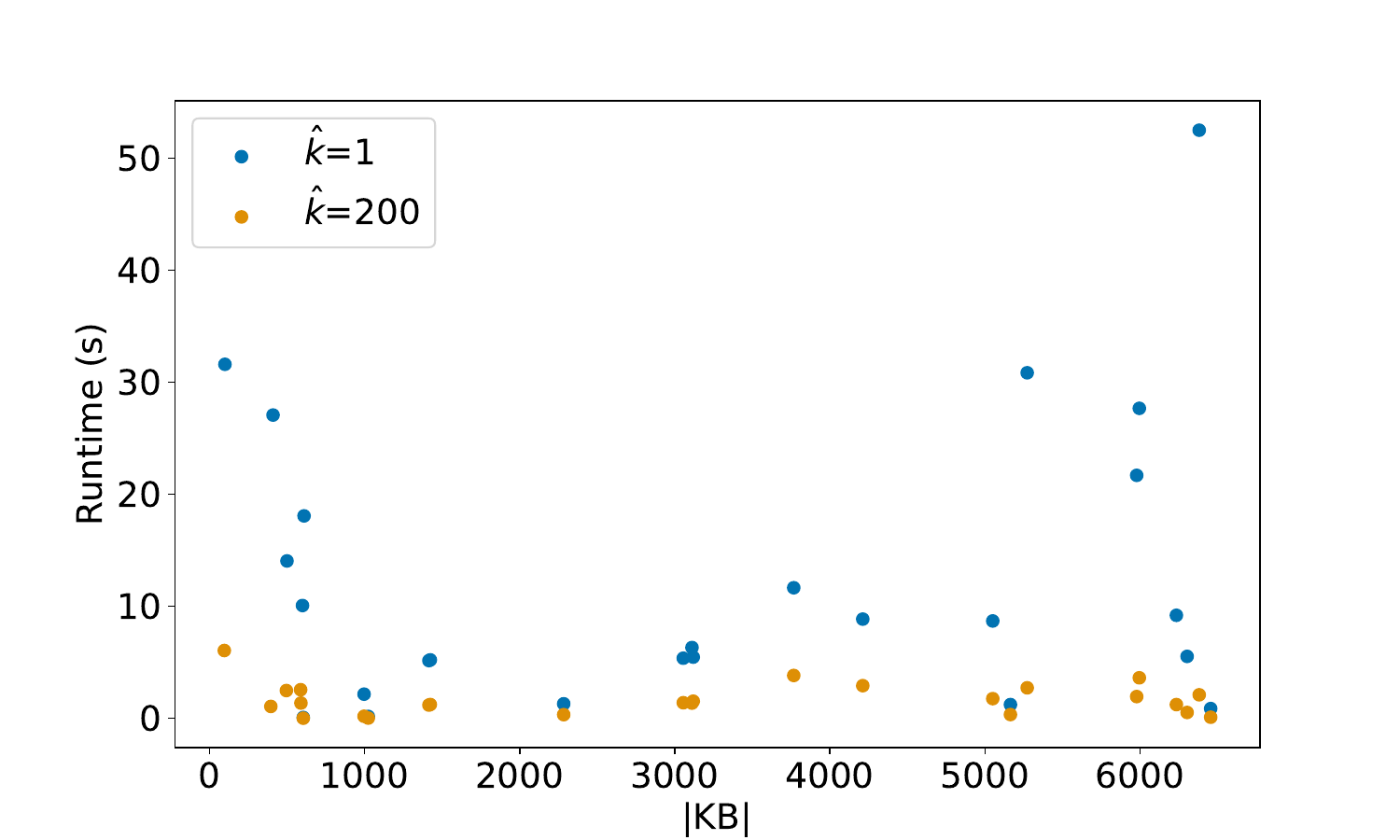}
        \caption{Random Instances.}
        \label{fig:TvsKB:rand}
    \end{subfigure}
    
    \caption{Average runtime of {\sc alg1} ($\hat{k}=1)$ and {\sc alg3} ($\hat{k}=200$) to compute an explanation across different knowledge base sizes for the planning, scheduling, and random CNF instances.}
    \label{fig:TvsKB}
\end{figure}

\subsubsection{Model Reconciling Explanations}

We now examine the effectiveness of Algorithm~\ref{alg:lmrp}, referred to as {\sc alg2}, and Algorithm~\ref{alg:plmrp}, referred to as {\sc alg4}, on computing model reconciling explanations. We chose the value of $\hat{k}=200$ for {\sc alg4} as it was the better performing parameter for {\sc alg3} in our previous experiments. More specifically now, we are interested in scenarios with varying degrees of knowledge asymmetry between the agent and human models. To simulate such scenarios, we used the actual encoded knowledge bases as the model of the agent ($\KBa$), and tweaked that model and assigned it to be the model of the human ($\KB_h$ or $\Bh$). We considered the following ways to tweak the human model, resulting in the following five scenarios:

\squishlist
\item \textbf{Scenario~1:}~We randomly removed 10\% of the formulae and removed 20\% of literals from 10\% of the total formulae in the human's model.
\item \textbf{Scenario~2:}~We randomly removed 20\% of the formulae and removed 20\% of literals from 20\% of the total formulae in the human's model.
\item \textbf{Scenario~3:}~We randomly removed 30\% of the formulae and removed 20\% of literals from 30\% of the total formulae in the human's model.
\item \textbf{Scenario~4:}~We randomly removed 40\% of the formulae and removed 20\% of literals from 40\% of the total formulae in the human's model.
\item \textbf{Scenario~5:}~We randomly removed 50\% of the formulae and removed 20\% of literals from 50\% of the total formulae in the human's model.
\squishend

In general, these experiments aim to answer the following two questions:
\begin{itemize}
    \item[] \textbf{Q1}: What is the performance of the algorithms on computing model reconciling explanations across different problem instances?
    \item[] \textbf{Q2}: What is the performance of the algorithms in scenarios with varying degrees of knowledge asymmetry between the agent and the human model?
\end{itemize}

\begin{table*}[t!]
    \centering
     \resizebox{1\textwidth}{!}{
    \begin{tabular}{|c||c|c|c|c|c|c||c|c|c|c|c|c||c|c|c|c|c|c|}
    \hline
    \multirow{2}{*}{Sce-} & \multicolumn{6}{c||}{Planning} & \multicolumn{6}{c||}{Scheduling} & \multicolumn{6}{c|}{Random CNF} \\
    \cline{2-19}
    \multirow{2}{*}{nario} & \multicolumn{3}{c|}{{\sc alg2}} & \multicolumn{3}{c||}{{\sc alg4}} & \multicolumn{3}{c|}{{\sc alg2}} & \multicolumn{3}{c||}{{\sc alg4}} & \multicolumn{3}{c|}{{\sc alg2}} & \multicolumn{3}{c|}{{\sc alg4}} \\
    \cline{2-19}
     & S & T/O & Runtime & S & T/O & Runtime & S & T/O & Runtime & S & T/O & Runtime & S & T/O & Runtime & S & T/O & Runtime  \\
    \hline \hline
    1 & 25 & 12 &67.0s & 28 & 9 & 59.7s & 33 & 2 &  51.4s & 33 & 2 & 33.1s& 27 & 5 & 30.4s&  21 & 11 & 12.3s \\
    \hline
    2 & 25 & 13 & 69.2s & 27 & 10 & 71.8s & 31 & 4 & 40.9s & 31 & 4 & 28.0s& 26 & 6 & 18.7s&  20 & 12 & 0.5s  \\
    \hline
    3 &  24& 14 & 67.9s& 26 & 12 & 68.8s & 32 & 3 & 60.3s & 32 & 3 & 37.1s & 29 & 3 & 20.4s & 21 & 11 & 2.5s  \\
    \hline
    4 &  25& 13 & 82.6s& 27 & 11 & 84.0s & 30 & 4 & 35.7s& 30 & 4 & 22.9s & 23 & 9 & 5.2s& 20 & 12 & 0.5s \\
    \hline
    5 & 22 & 15 & 84.3s& 24 & 13 & 89.9s & 30 & 4 & 34.4s & 30 & 4 & 21.5s & 24 & 8 &3.8s & 20 & 11 & 0.5s \\
    \hline
    \end{tabular}
    }
    \caption{Instances Solved (S) vs. Timed Out (T/O) for the Planning, Scheduling, and Random CNF Benchmarks for {\sc alg2} and {\sc alg4} at $\hat{k}=200$. Note that Runtime denotes the average runtime over all solved instances.}
    \label{tab:solved_plmrp}
\end{table*}

\begin{figure}[!t]
    \centering
    \begin{subfigure}[b]{0.48\textwidth}
        \centering
        \includegraphics[width=\textwidth]{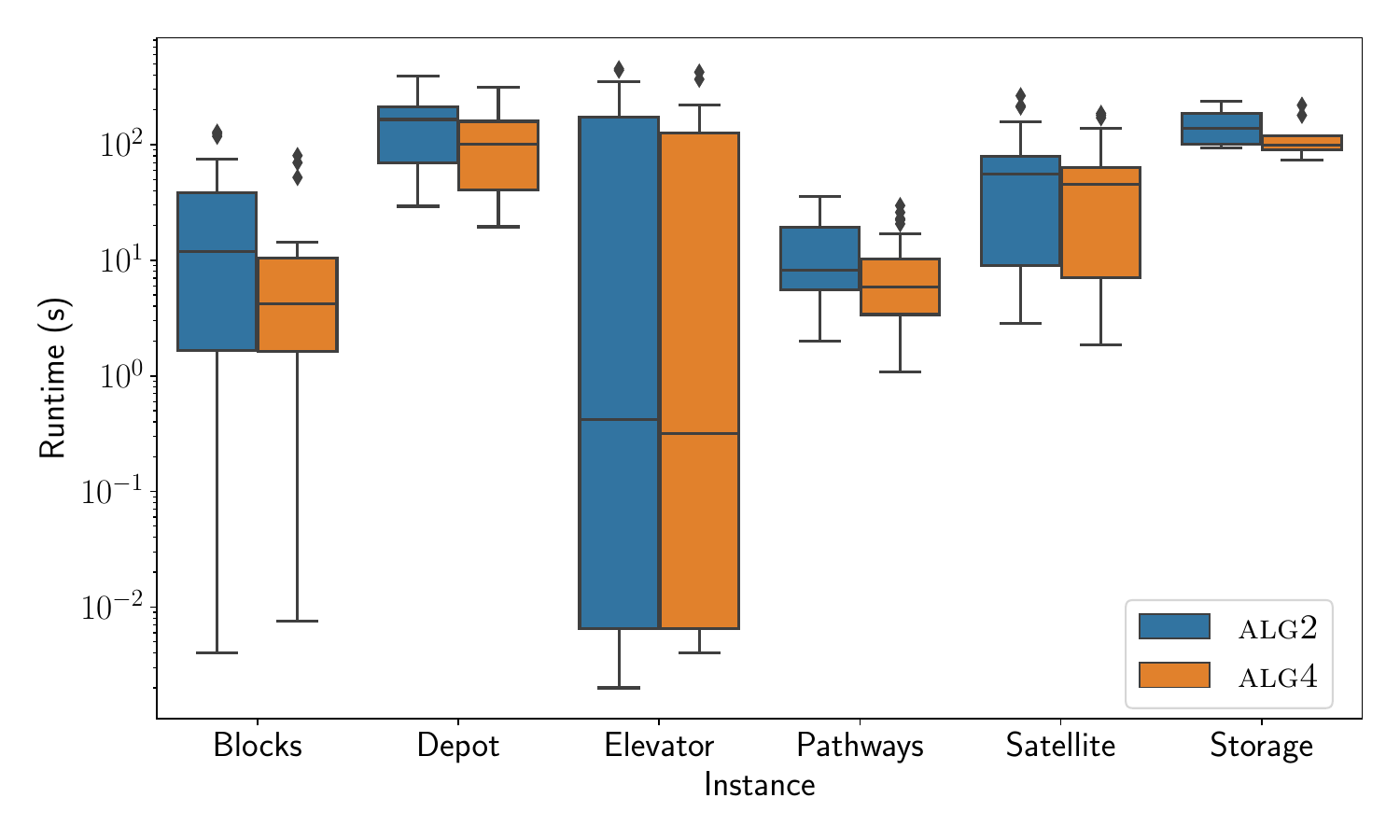}        \caption{Planning Instances.}
    \label{fig:Tvsk:plan}
    \end{subfigure}
    \hfill 
    \begin{subfigure}[b]{0.48\textwidth}
        \centering
        \includegraphics[width=\textwidth]{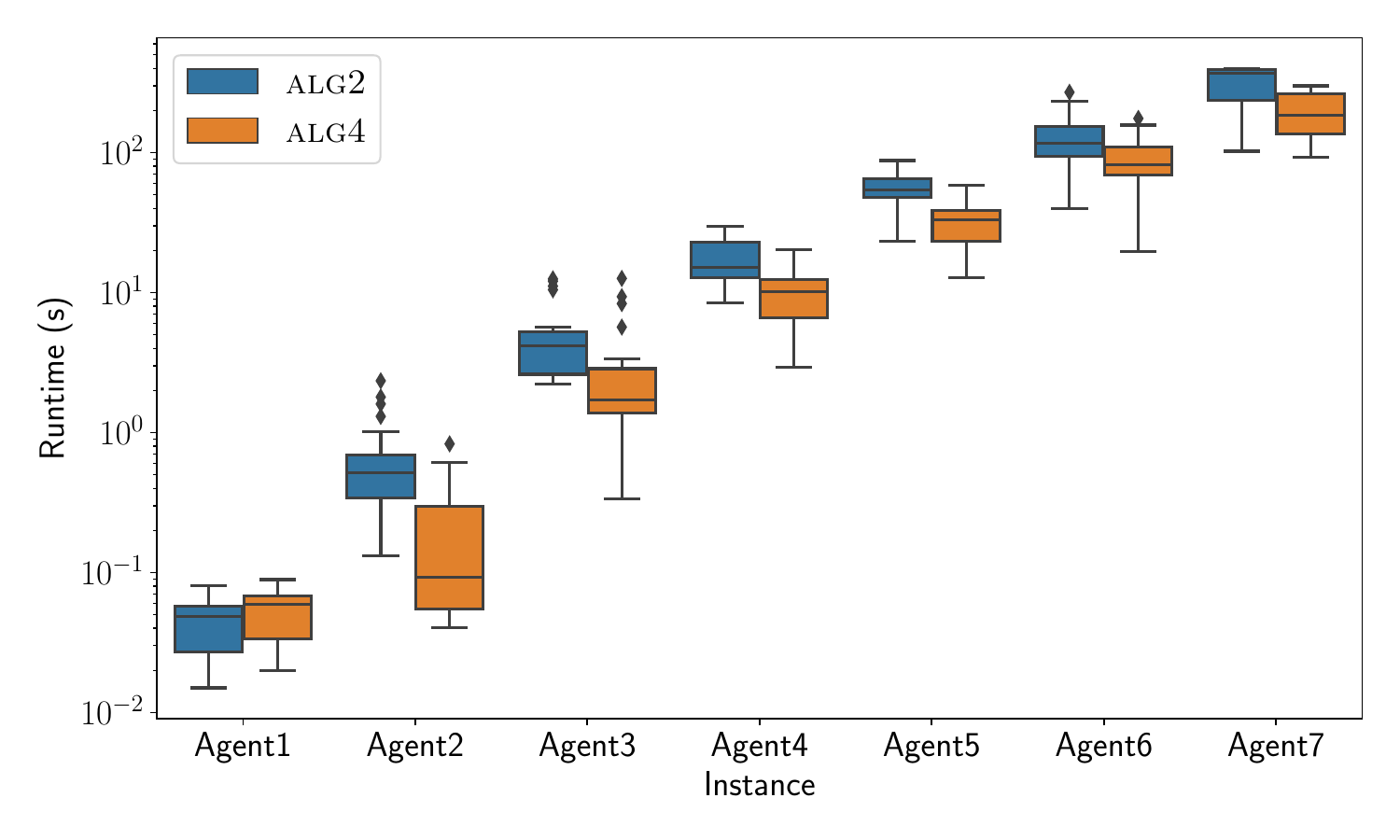}        \caption{Scheduling Instances.}
    \label{fig:Tvsk:sch}
    \end{subfigure}
    
    \vspace{0.2cm}
    
    \begin{subfigure}[b]{0.5\textwidth} 
        \centering
        \includegraphics[width=\textwidth]{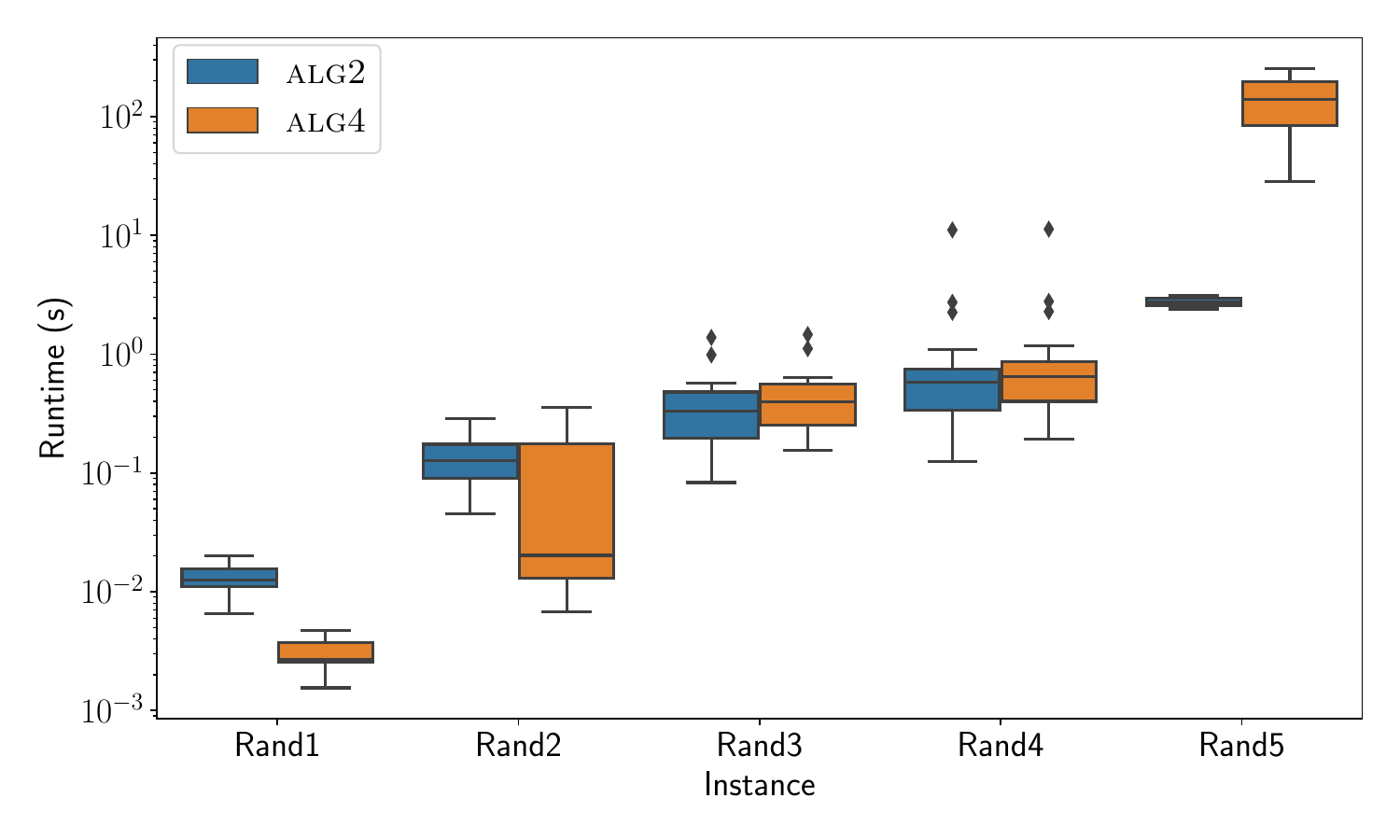}
        \caption{Random Instances.}
        \label{fig:grid2}
    \end{subfigure}
    \caption{Runtime distributions of {\sc alg2} and {\sc alg4} at $\hat{k}=200$ to compute an explanation across all commonly solved instances.}
    \label{fig:common:all}
\end{figure}

\begin{figure}[!t]
    \centering
    \begin{subfigure}[b]{0.49\textwidth}
        \centering
        \includegraphics[width=\textwidth]{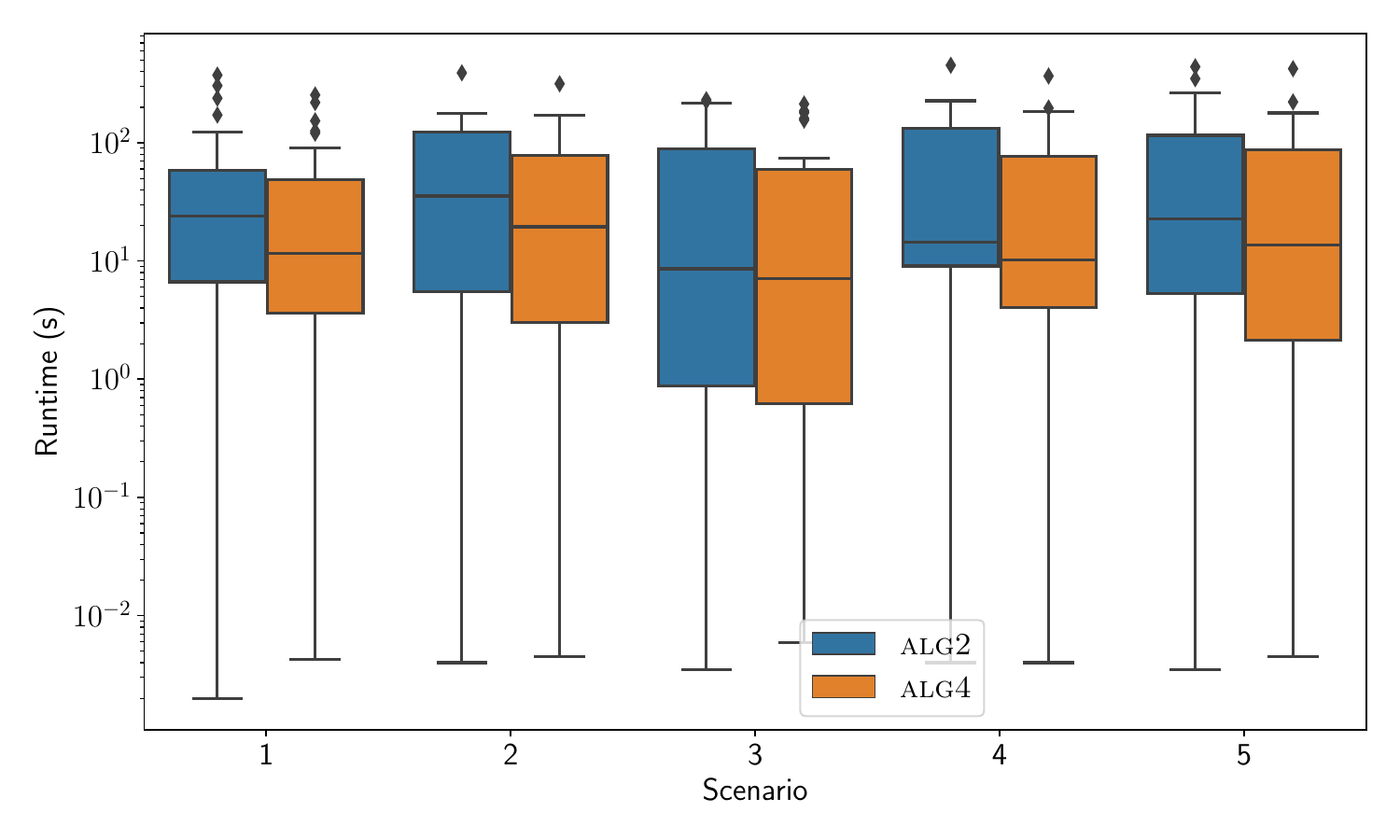}        \caption{Planning Instances.}
    \label{fig:Tvsk:plan}
    \end{subfigure}
    \hfill 
    \begin{subfigure}[b]{0.49\textwidth}
        \centering
        \includegraphics[width=\textwidth]{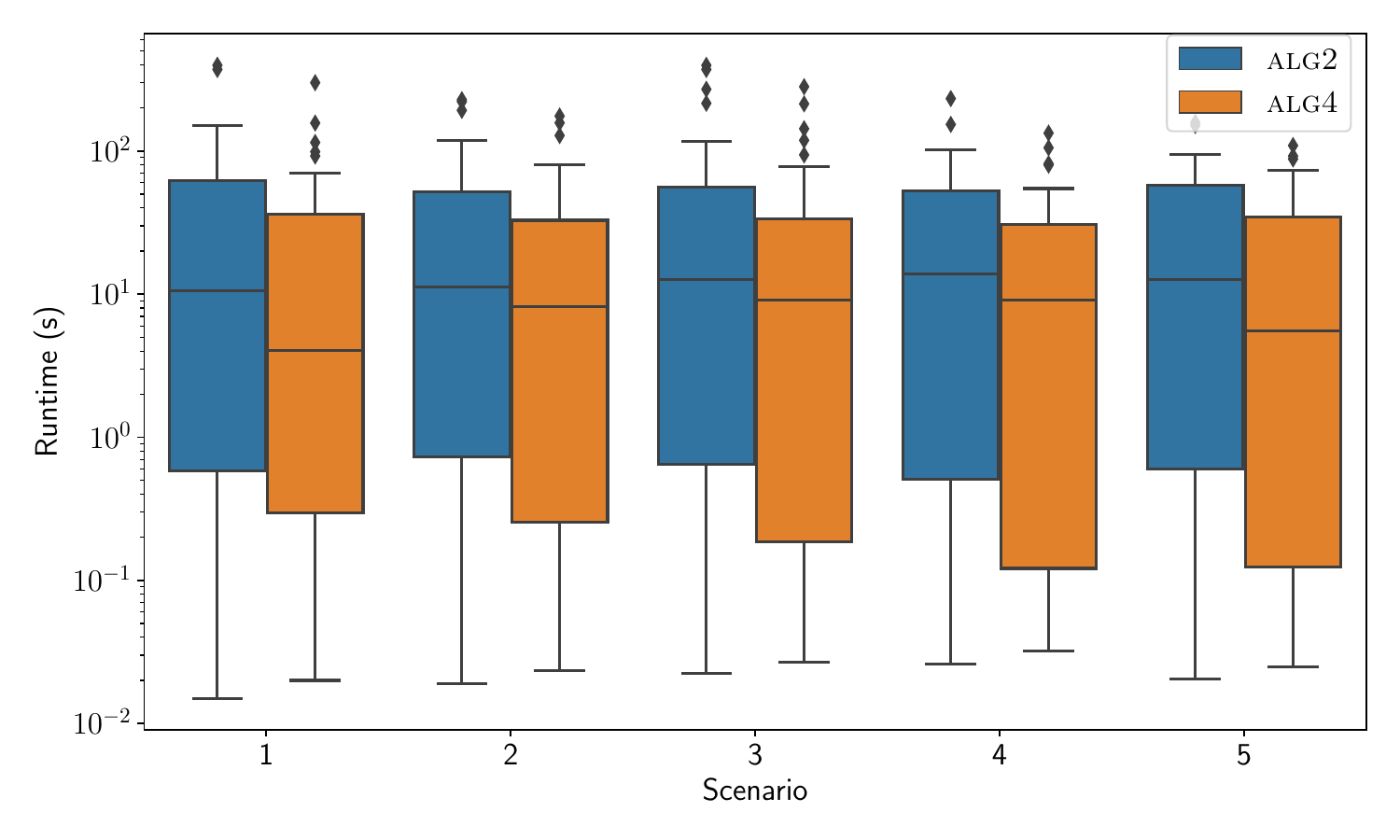}        \caption{Scheduling Instances.}
    \label{fig:Tvsk:sch}
    \end{subfigure}
    
    \vspace{0.2cm}
    \begin{subfigure}[b]{0.51\textwidth} 
        \centering
        \includegraphics[width=\textwidth]{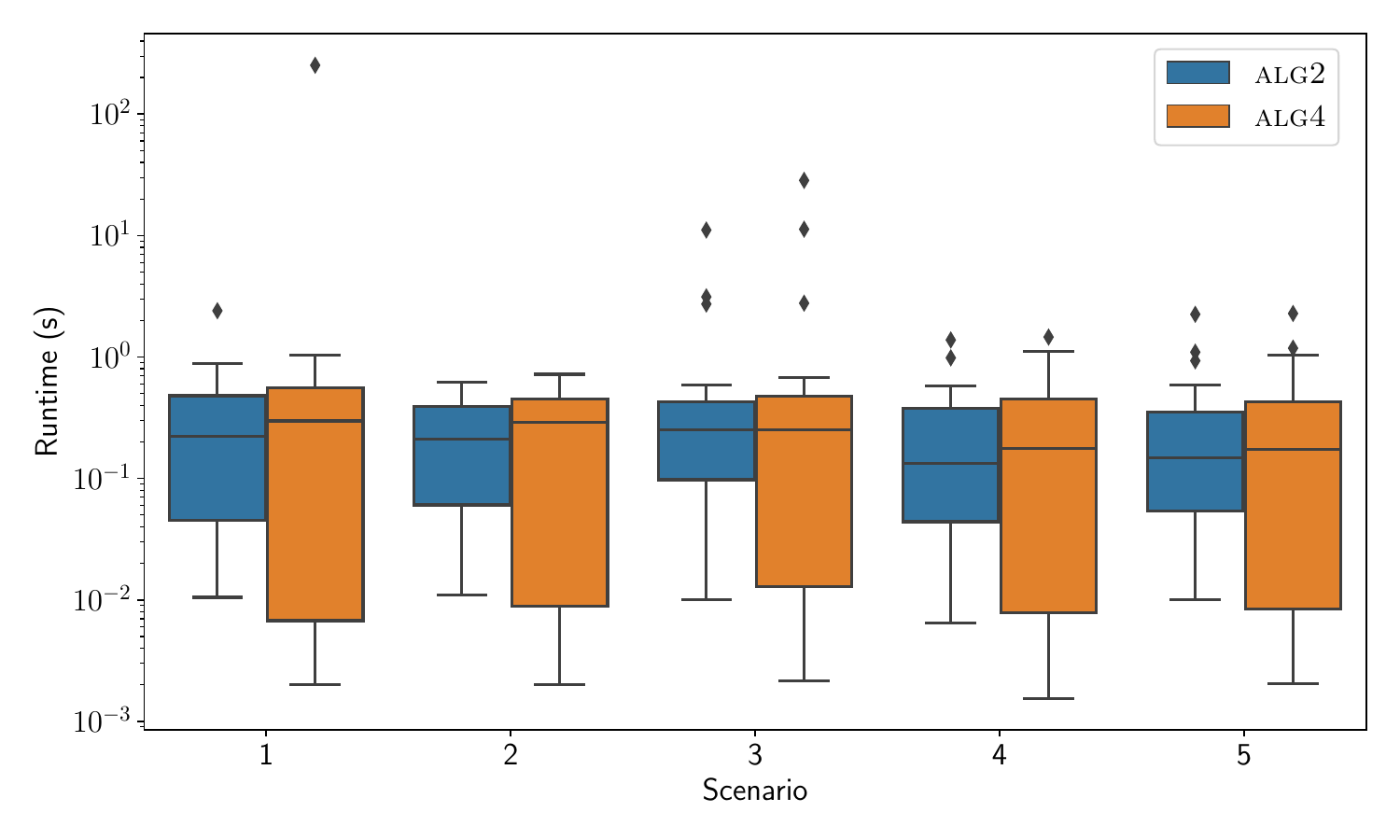}
        \caption{Random Instances.}
        \label{fig:grid2}
    \end{subfigure}
    \caption{Runtime distributions of {\sc alg2} and {\sc alg4} at $\hat{k}=200$ to compute an explanation across commonly solved instances in each of the five scenarios.}
    \label{fig:common:scenario}
\end{figure}

\begin{figure}[!t]
    \centering
    \begin{subfigure}[b]{0.45\textwidth}
        \includegraphics[width=\textwidth]{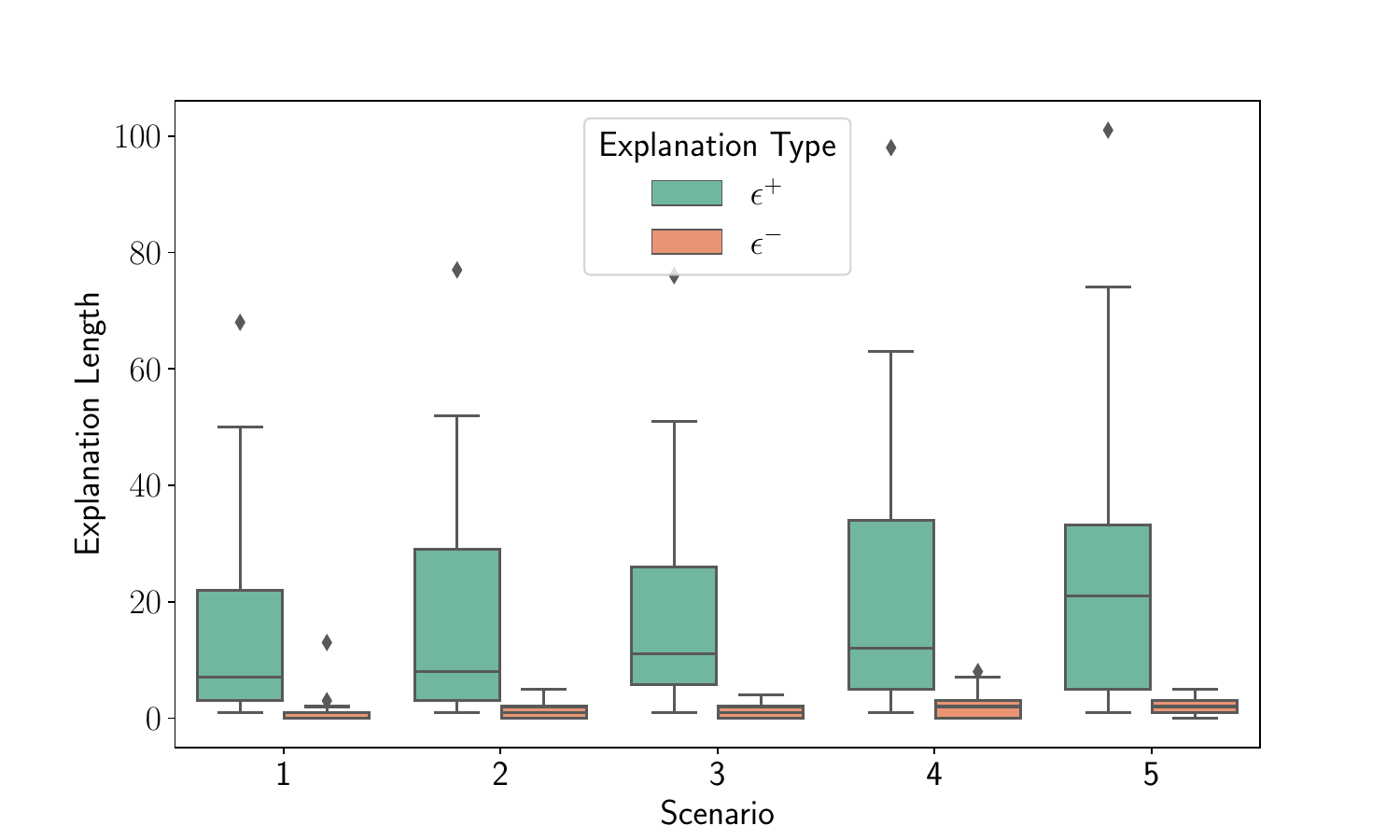}
        \caption{{\sc alg2} on Planning Instances.}
        \label{fig:grid1}
    \end{subfigure}
    \quad 
    \begin{subfigure}[b]{0.45\textwidth}
        \includegraphics[width=\textwidth]{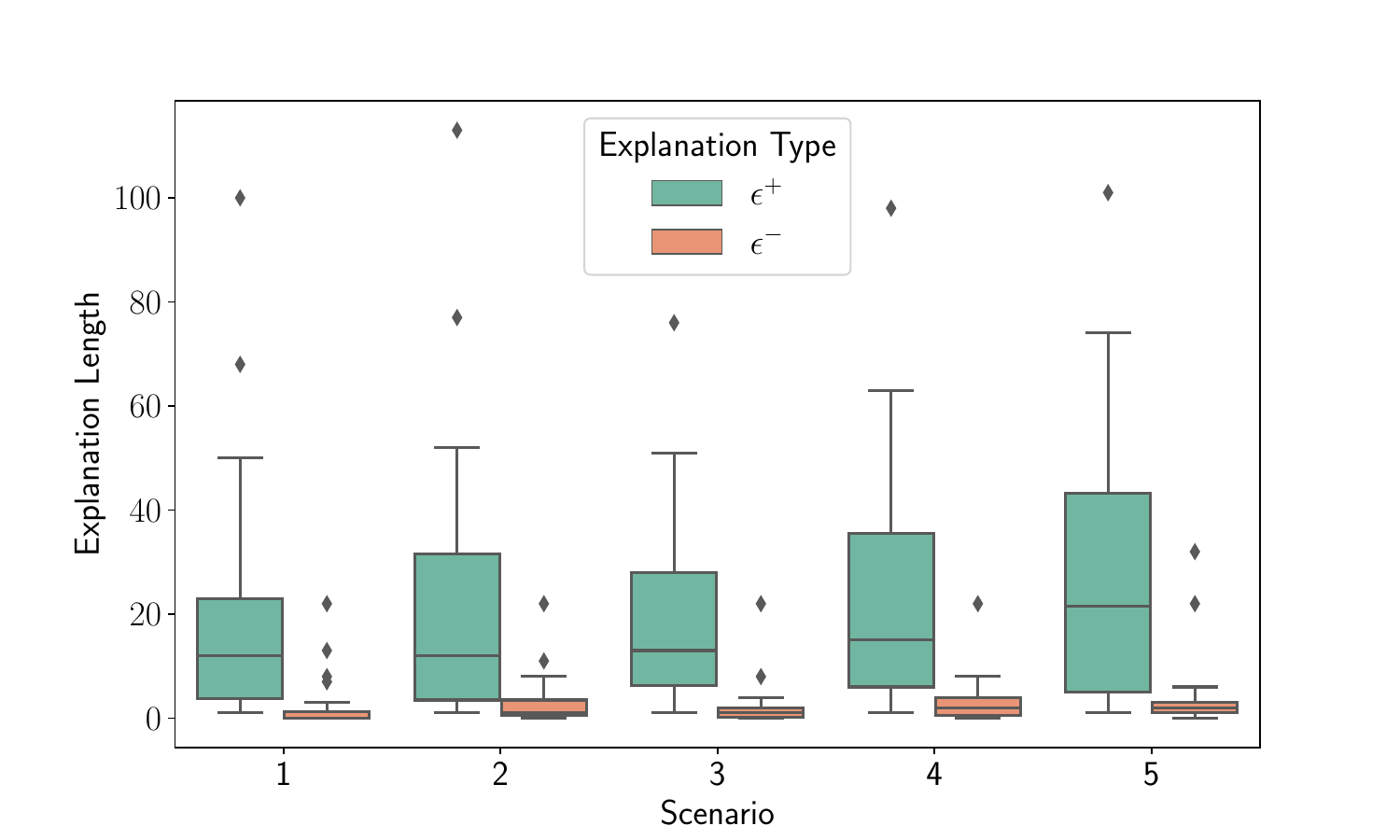}
        \caption{{\sc alg4} on Planning Instances.}        \label{fig:grid2}
    \end{subfigure}
    \\
    \begin{subfigure}[b]{0.45\textwidth}
        \includegraphics[width=\textwidth]{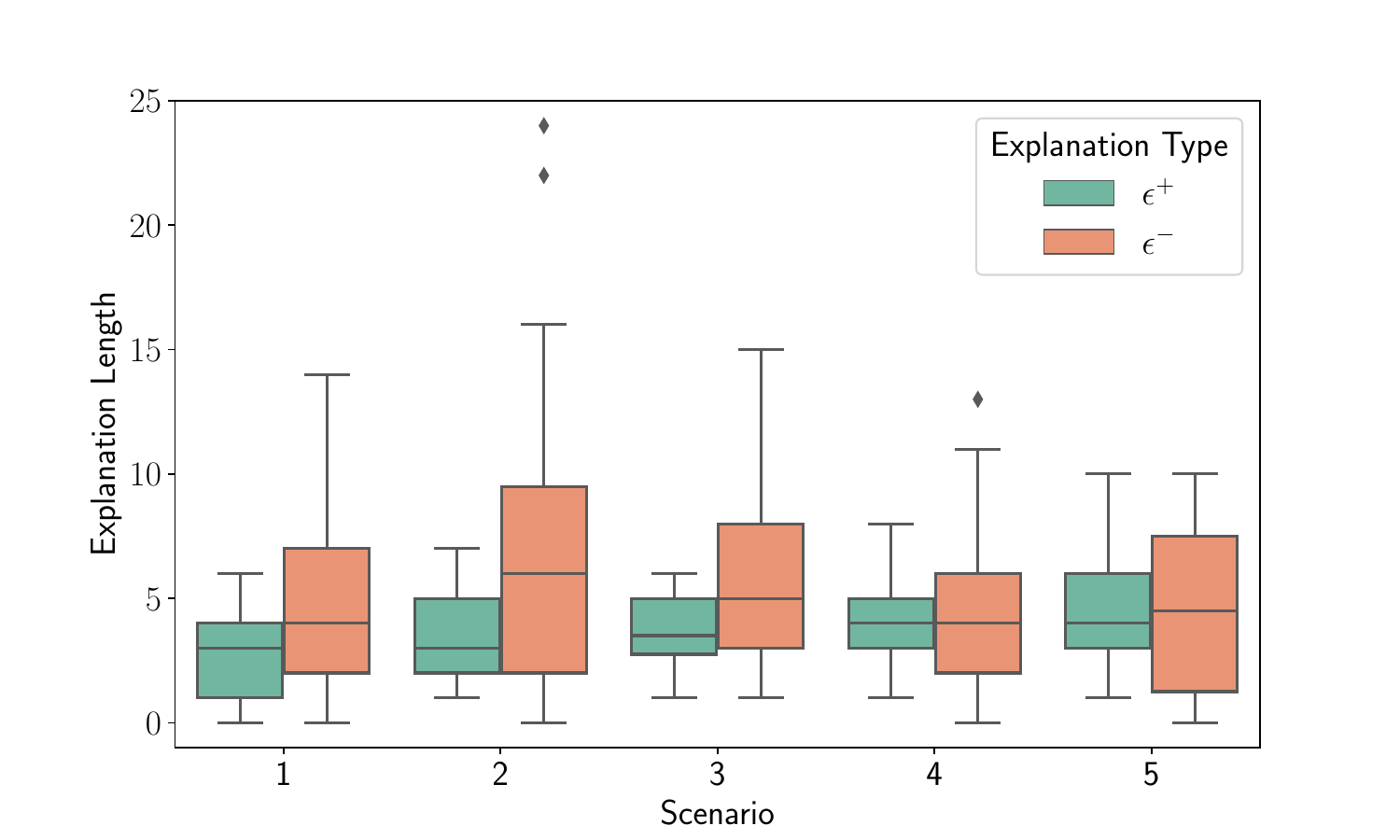}
        \caption{{\sc alg2} on Scheduling Instances.}        \label{fig:grid3}
    \end{subfigure}
    \quad
    \begin{subfigure}[b]{0.45\textwidth}
        \includegraphics[width=\textwidth]{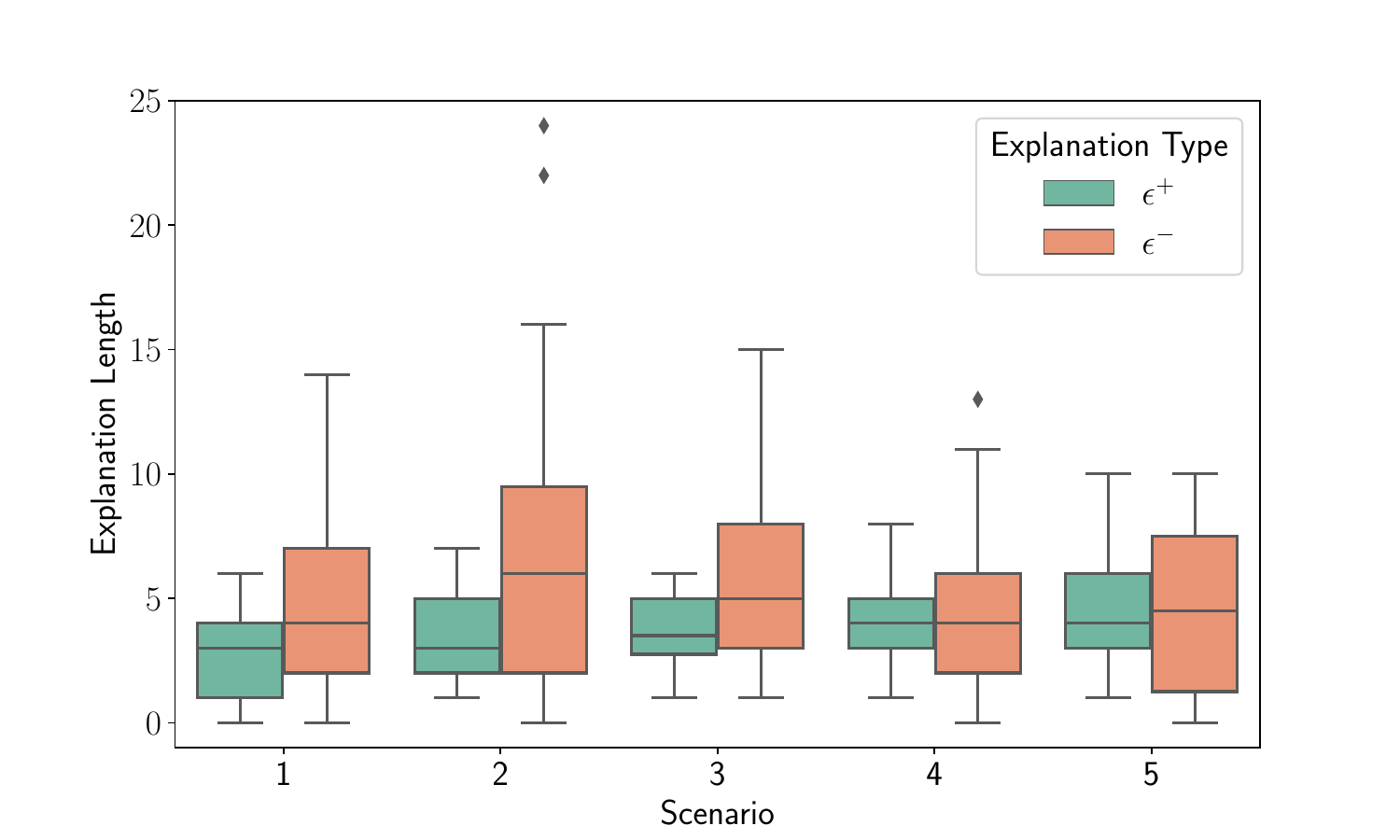}
        \caption{{\sc alg4} on Scheduling Instances.}        \label{fig:grid4}
    \end{subfigure}
    \\
    \begin{subfigure}[b]{0.45\textwidth}
        \includegraphics[width=\textwidth]{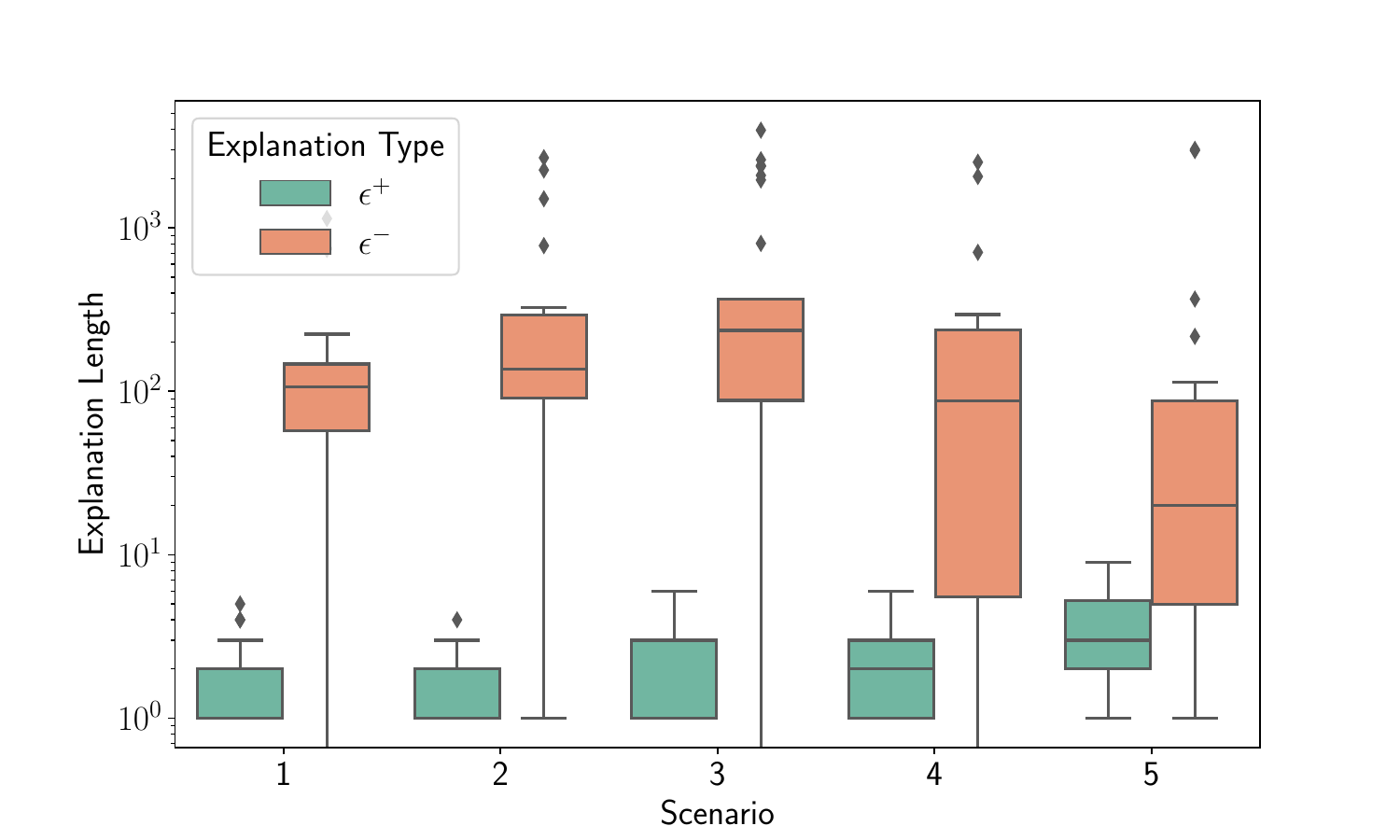}
        \caption{{\sc alg2} on Random CNF Instances.} 
        \label{fig:grid3}        
    \end{subfigure}
    \quad 
    \begin{subfigure}[b]{0.45\textwidth}
        \includegraphics[width=\textwidth]{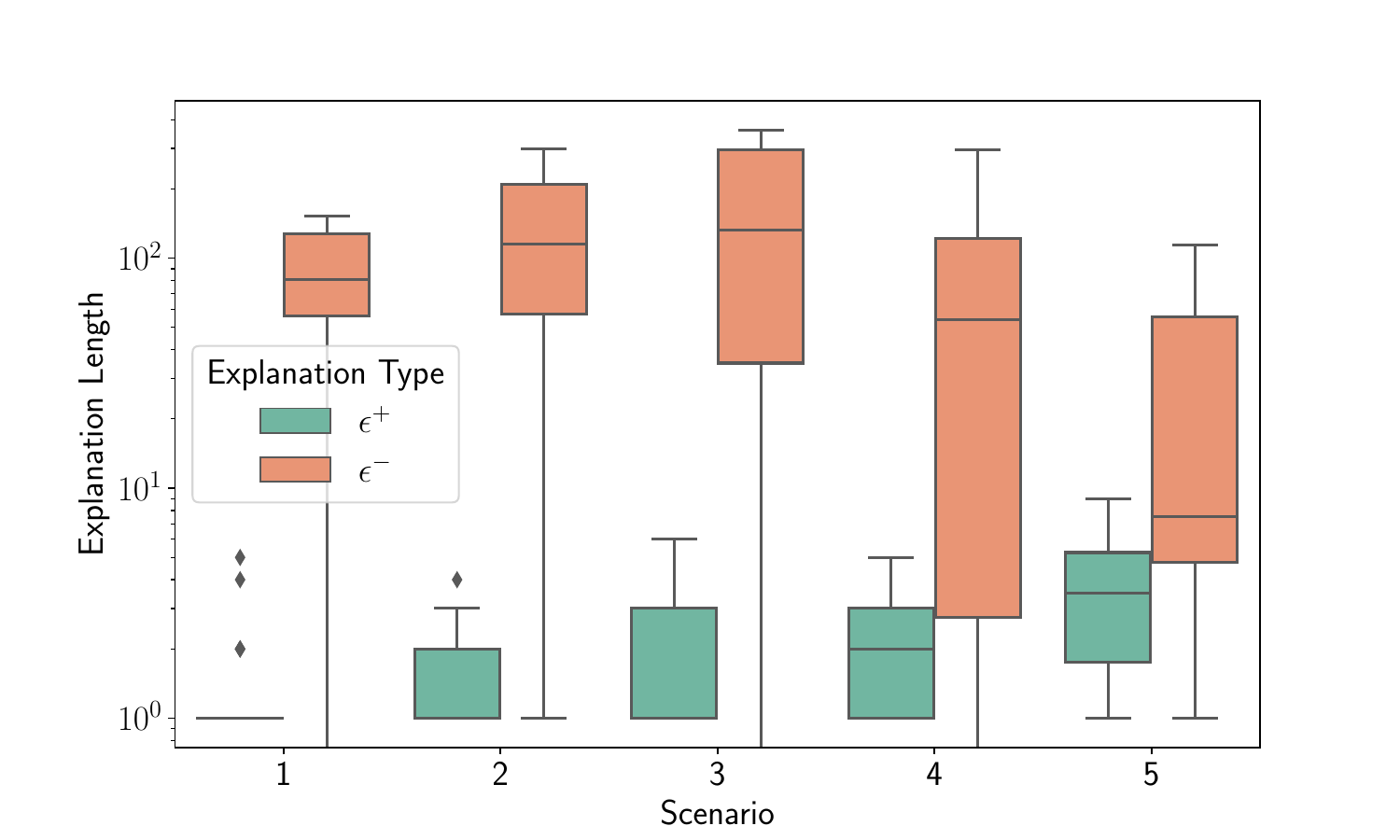}
        \caption{{\sc alg4} on Random CNF Instances.} 
        \label{fig:grid2}
    \end{subfigure}
    \caption{Distributions of the lengths of explanations $\e^+$ and $\e^-$ computed by {\sc alg2} and {\sc alg4} at $k=200$ across all planning, scheduling, and random CNF instances.}
    \label{fig:expl-lmrp}
\end{figure}

Table~\ref{tab:solved_plmrp} tabulates the instances solved and timed out by {\sc alg2} and {\sc alg4} at $\hat{k}=200$ across the five scenarios, where we observe the following trends. For the planning instances, the runtime of both algorithms increases as the difference between the models of the agent and human increases (Scenarios 1 to 5), since both algorithms search over the explanation search space, which increases as the number of differences between the two models increases. As in the previous experiments, {\sc alg4} at $\hat{k}=200$ yields faster runtimes than {\sc alg2}. For the scheduling instances, we observe that the runtimes increase from Scenario 1 to 3, but decrease from Scenarios 4 to 5. Upon closer inspection, this is mainly because the instances solved in these scenarios were easier (i.e.,~smaller knowledge base sizes) than those solved in the other three scenarios, thus resulting in smaller average runtimes. A similar trend is observed for the random CNF instances. However, in the random CNF instances, {\sc alg2} managed to solve more instances than {\sc alg4}. After examining them more closely, we found that the main bottleneck of {\sc alg4} in those instances was computing the most-probable worlds of the explanandum (i.e.,~the WMaxSAT solver). Even for smaller values of $\hat{k}$, the solver failed to compute all the worlds under the specified time limit---the increase in search space (e.g.,~because of considering $\Ba$ and $\Bh$) increased the complexity of these instances. We expect that an optimized and more dedicated solver may be able to overcome this limitation. The runtime distributions for {\sc alg2} and {\sc alg4} across all commonly solved instances and across commonly solved instances in each scenario can be seen in Figures~\ref{fig:common:all} and~\ref{fig:common:scenario}, respectively. For these instances, we observe, like in the previous experiments, that {\sc alg4} has faster runtimes than {\sc alg2}.

Moreover, in Figure~\ref{fig:expl-lmrp} we see the distributions of the model reconciling explanation lengths computed by both algorithms. As expected, the general trend is that the size of the explanation $\e^+$ (i.e.,~formulae from $\KBa$ for $\KB_h$ (or $\B_h$)) increases with each scenario, as the difference between the agent and human models increase. The same trend can be seen for $\e^-$---each scenario from 1 to 5 has an increasing amount of inconsistencies between the two models. Interestingly, $\e^-$ was largest in the random CNF instances. This indicates that the inconsistencies between the human model and the corresponding $\e^+$ were high. That can also be used to explain why {\sc alg4} failed to solve a subset of random CNF instances---highly inconsistent knowledge bases are considered as the most difficult instances for MaxSAT solvers.

In conclusion, the comparative analysis of {\sc alg2} and {\sc alg4} at $\hat{k}=200$ across varied problem instances shows some trends in performance and computational complexity. The observed increase in runtime with the increase of differences between agent and human models underscores the direct relationship between model disparity and the explanation search space size. Notably, {\sc alg4} consistently outperforms {\sc alg2} in terms of runtime across most scenarios, except in certain random CNF instances where the computation of most-probable worlds becomes a bottleneck due to the limitations of the WMaxSAT solver. This highlights a potential area for further optimization and development of more efficient solvers. Furthermore, the analysis of model reconciling explanation lengths reveals an expected increase in inconsistency measures as the model differences widen, particularly highlighted in random CNF instances.

\section{Related Work}
\label{sec:related}

We now provide a discussion of related work from the planning and knowledge representation and reasoning (KR) literature. We focus on these two areas as our approach is motivated by the model reconciliation problem introduced by the planning community and it bears some similarity to other logic-based approaches in KR.

\subsection{Related Planning Work}

We have briefly discussed the body of work within \emph{explainable AI planning} (XAIP) on the topic of model reconciliation in Section~\ref{sec:intro} and situated our work within that literature.\footnote{For a more complete discussion on XAIP related work, please see Section~9 of \citet{vasLMRP}.} For model reconciling explanations, a similar line of research that considers uncertainty about the human user's model is that of \citet{sreedharan2018handling}. In their work, they focus on scenarios where the human's model is located within a space of possible human models that the agent has, and their method operates in that space to find conformant explanations, i.e.,~explanations applicable to a set of possible models. In contrast, in our approach we represent an uncertain human model as a probability distribution that the agent has, and we have explicitly defined the notion of probabilistic explanation as well as metrics to measure its quality.

For monolithic explanations, most XAIP works have focused on contrastive explanations in deterministic settings. These explanations take the form of ``Why not A (instead of B)?'', where A is an alternative (or foil) suggested by the human to a decision B proposed by the agent. Contrastive explanations have found applications in linear temporal logic systems~\cite{kasenberg2020generating}, general epistemic accounts~\cite{belle2023counterfactual}, multi-agent optimization problems~\cite{zehtabi2024contrastive}, and in oversubscription planning~\cite{eifler2020new}. There have also been used to provide a taxonomy of user questions that often arise during interactive plan exploration~\cite{krarup2021contrastive}, as well as towards creating user interfaces for decision support systems~\cite{karthik2021radar,kumar2022vizxp}.

\subsection{Related KR Work}

In the monolithic case, our definition of a probabilistic explanation (Definition~\ref{def:pexpl}) may appear similar to what was proposed by \citet{gardenfors1988knowledge}.~However, an important distinction is that \citeauthor{gardenfors1988knowledge} is dealing with epistemic states that do not contain the explanandum, while we are dealing with belief bases that do contain the explanandum. We also define a different notion of explanatory power as well as present algorithms for computing explanations. \citet{urszulaprob} have also considered the problem of defining what constitutes an explanation in probabilistic systems, however they solely focus on epistemic states defined over causal structures.

The notion of (monolithic) explanation has also been explored by the \emph{probabilistic logic programming} (PLP) community~\cite{de2008probabilistic,fierens}, a formalism that extends logic programming languages with probabilities. In PLP, explanations have been associated with possible worlds.\footnote{Recall that a possible world is a truth value assignment to the atoms in the language.} The most prominent task there is that of the most probable explanation (MPE), which consists of finding the world with the highest probability given some evidence \cite{shterionov2015most}. However, a world does not show the chain of inferences of a given explanandum and it is not minimal by definition, since it usually includes a (possibly large) number of probabilistic facts whose truth value is irrelevant for the explanandum. An alternative approach is using the proof of an explanandum as an explanation \cite{kimmig2011implementation}, where a proof is a (minimal) partial world $\omega'$ such that for all worlds $\omega \supseteq \omega'$, the explanandum is true in $\omega$. In this case, one can easily ensure minimality, but even if the partial world contains no irrelevant facts, it is still not easy to determine the chain of inferences behind a given explanandum. Finally, \citet{renkens2014explanation} have leveraged explanations in PLP as approximation techniques for more efficiently computing weighted model counting problems.

It is important to mention that logic-based methods have also been employed in machine learning problems \cite{marques2022delivering}. For example, \citet{shrotri2022constraint} proposed CLIME, a constraint-driven explanation framework for black-box ML models that enables users to specify Boolean constraints to guide the perturbation phase when generating explanations. Unlike their approach that focuses on constraining the input space during perturbation, our approach integrates uncertainty directly into the explanation generation process through probabilistic logic. Both approaches use concepts from formal methods, but our work addresses the specific case of explanations in probabilistic scenarios and model reconciliation, providing a complementary perspective to CLIME's constraint-driven framework. A similar line of research is the work by \citet{izza2023computing}, where they introduced an approach to computing probabilistic explanations across various classifier types (such as decision trees and naive Bayes classifiers). Their framework addresses the challenge of generating minimal sets of features that guarantee a prediction with a certain probability threshold. While \citeauthor{izza2023computing} focus on efficiently computing explanations for specific classifier architectures, our framework provides a more general theoretical foundation for explanation generation in uncertain environments and extends to the model reconciliation problem.

In the model reconciliation case, we have extended our previous work on the logic-based model reconciliation problem \cite{vasLMRP} to handle scenarios where the human model is uncertain. To the best of our knowledge, the application of probabilistic explanations in the context of model reconciliation that we consider in this work is novel.

Finally, the algorithms presented in this paper are an extension of our previous work \cite{vas21}. Specifically, Algorithms~\ref{alg:expl} and~\ref{alg:lmrp} are inspired by a procedure for computing a \emph{smallest minimal unsatisfiable set} (SMUS) of an inconsistent formula, originally presented by \citet{ignatiev-plm15}. The method is also related to other similar approaches for enumerating MUSes and \emph{minimal correction sets} (MCSes). Moreover, our approach is similar in spirit to the HS-tree presented by \citet{rei87}. Although the original purpose was to enumerate diagnoses, Reiter's procedure can be easily adapted to enumerate MUSes (called conflicts in that paper) as already noted by \citet{previti-m13}. However, the computation of an SMUS might require more substantial modifications.  Procedures like the one presented by Reiter, which target MCSes (diagnoses) instead of MUSes (conflicts), can be seen as the dual version of our algorithm. In particular, the algorithm MaxHS~\cite{daviesB11} applies the same idea of iteratively computing and testing a minimal hitting set for the computation of a MaxSAT solution (the complement of the smallest MCSes). 

There are other approaches that exploit the duality between MUSes and MCSes, but instead of iteratively checking if the current hitting set is an MUS, they first compute the set of all MCSes~\cite{kas-jar08}. This has the potential advantage that once all the MCSes are known, every minimal hitting set on the collection of all MCSes is guaranteed to be an MUS (Proposition \ref{prop:duality}). However, as the number of MCSes is, in the worst case, exponential in the size of the formula, this approach might fail even before reporting the first MUS. This is particularly unnecessary when the target is to return a single support, like the one presented in this paper.

\section{Discussion}
\label{sec:discussion}

In this paper, we presented a logic-based framework for generating probabilistic explanations within uncertain knowledge bases (e.g.,~belief bases). We distinguished between two notions of explanation---\textit{monolithic}, where explanations are defined with respect to an agent knowledge base, and \textit{model reconciling}, where explanations are defined with respect to an agent knowledge base as well as a human user knowledge base. Particularly, we introduced formal definitions for probabilistic monolithic and probabilistic model reconciling explanations, along with metrics to evaluate their quality (explanatory gain and explanatory power). 

Our framework operates under several assumptions that we need to address. First, a fundamental assumption is that the problem domain can be encoded in a logical language. This assumption holds for many structured decision-making scenarios but may not apply to domains where knowledge is primarily subsymbolic or where the underlying decision processes are difficult to formalize in a logical language. The general effectiveness of our approach thus depends critically on how well the domain can be represented within a logical framework. Further, we used propositional (probabilistic) logic as our language of choice due to its simplicity, capacity to encode a plethora of domains, as well as computationally efficiency (e.g.,~through SAT solvers). While our methods and algorithms are presented within propositional logic, their underlying principles are broadly applicable to any constraint system where the satisfiability of subsets can be decided. This opens avenues for extending our work to other logical systems, such as Markov Logic Networks \cite{richardson2006markov} and Probabilistic Logic Programs~\cite{fierens}.


For model reconciling explanations specifically, our framework makes some assumptions about human inferential capabilities. Notably, we presume that human users possess the reasoning capacity to process and understand the explanations provided by the AI agent. However, this does not imply that humans can reason with the same efficiency as AI agents---much like how humans can perform arithmetic operations correctly but are significantly slower than calculators. Rather, we assume that given sufficient time and a properly formulated explanation, humans can validate the logical correctness of the AI agent's decision-making process.

Another big assumption in model reconciliation is that the agent has direct access to the human's belief base. This assumption, while necessary for the operationalizing our framework, is often unrealistic in practice. Thus, a promising direction for future research involves approximating the human model through iterative interactions. Indeed, in a recent work \cite{tang2025approximating}, we showed how (propositional) probabilistic logic can be effectively used to learn a human model based on past, dialogue-based interactions. This line of work provides a fruitful ground for our framework to create more personalized and effective model reconciliation processes.

At the other end of the spectrum, a practical challenge of explainable decision-making frameworks involves how the explanations are communicated to human users.\footnote{Note that, in certain settings, it has been shown that explanations in the form of model reconciliation are well understood and preferred by human users~\cite{8673193,zahedi2019towards,vasLMRP,kumar2022vizxp,vasileiouplease}.} While we tackled the problem of generating explanations with formal guarantees about their quality, we did not directly explore the presentation of these explanations to users. There are several avenues to explore here. First, the emergence of large language models (LLMs) \shortcite{NEURIPS2020_1457c0d6} may offer an immediate solution to this problem. Specifically, one can use LLMs as translation engines between logic and natural language, and thus create hybrid systems that generate natural language explanations with formal guarantees (provided by logical framework). For example, we presented a system for tackling this idea in the domain of coursework planning \cite{vasileiou2025traceKR,vasileiou2025trace}. In future work, we plan to extend it for general problems expressible in propositional logic. It is also worth emphasizing that explanations can also be communicated in other forms, such as visualizations, something that has been shown to improve the users' understanding when compared to text-based explanations alone \cite{kumar2022vizxp,karthik2021radar}.

Finally, it is important to acknowledge the resurgence of hybrid approaches that combine symbolic and neural-based methods~\shortcite{d2009neural,garcez2023neurosymbolic}. These \textit{neuro-symbolic} AI systems aim to synergize the strengths of both paradigms: the robust learning capabilities and pattern recognition of neural networks with the interpretability and reasoning power of symbolic systems. We position the methods proposed in this work as complementary to the advancements in neuro-symbolic AI, potentially informing and enhancing the explainability aspects of future neuro-symbolic systems.

\section{Concluding Remarks}
\label{sec:conclusion}
In this paper, we attempted to bridge the gap between classical explanation models and the inherent uncertainty found in real-world scenarios. We started by describing a framework for generating \textit{probabilistic monolithic explanations} within uncertain knowledge bases (e.g.,~belief bases), and introduced the concepts of \textit{explanatory gain} and \textit{explanatory power} as quantitative measures to evaluate the effectiveness and relevance of explanations, thus offering a better characterization of explanation quality. Additionally, we presented an extension to the model reconciliation problem for generating \textit{probabilistic model reconciling explanations}, which addresses the need for reconciling model differences between an agent and a human model, specifically when the human model is not known with certainty.

Furthermore, we developed algorithms that leverage the duality between minimal correction sets (MCSes) and minimal unsatisfiable subsets (MUSes) and demonstrated their potential for generating probabilistic explanations. Our experimental evaluations across different benchmarks underscore the effectiveness of these algorithms, suggesting they are promising in practical scenarios.

\section*{Acknowledgments}

We thank the anonymous reviewers, whose suggestions improved the quality of our paper. Stylianos Loukas Vasileiou and William Yeoh are partially supported by the National Science Foundation (NSF) under award 2232055. Son Tran is partially supported by NSF grants 1914635 and
2151254 as well as ExpandAI grant 2025-67022-44266. The views and conclusions contained  in  this  document  are  those  of  the  authors  and  should  not  be  interpreted  as representing the official policies, either expressed or implied, of the sponsoring organizations, agencies, or the United States government.

\appendix

\section{Illustrative Example: Simple Package Delivery Problem}

In Section~\ref{sec:example}, we presented a simple probabilistic planning scenario~\cite{littman1997probabilistic} that can be used by our framework. This problem can be fully specified in the form of the probabilistic planning domain definition language (PPDDL) \cite{younes2004ppddl1}, as shown in Listings~\ref{lst:domain} and~\ref{lst:problem}. Given the PPDDL descriptions, we can then encode this problem into (propositional) probabilistic logic \cite{littman1997probabilistic}, similarly to the encoding for classical planning problems (cf. \citet{kautz1996encoding}). Specifically, for $l,l' \in \{room1, room2\}$ and $c \in \{A, B\}$, we can encode the belief base: \\

\xhdr{Initial States:} (Starting states and belief about the environment)
\begin{align*}
&(\texttt{robot}\textrm{-}\texttt{at}(room1)_0, \infty) \\
&(\neg \texttt{robot}\textrm{-}\texttt{at}(room2)_0, \infty) \\
&(\neg \texttt{package}\textrm{-}\texttt{delivered}_0, \infty) \\
&(\texttt{crowded}(A), w_1)\\
&(\texttt{crowded}(B), w_2)
\end{align*}

\xhdr{Goal State:} (Final state to be reached)
\begin{align*}
&(\texttt{package}\textrm{-}\texttt{delivered}_{t_n}, \infty) \\
\end{align*}

\xhdr{Action Preconditions:} (Fluents that must be true at timestep $t$ for the action to be executed)
\begin{align*}
& ( \texttt{move}(l,l',c)_t \rightarrow \texttt{robot}\textrm{-}\texttt{at}(l)_t, \infty) \\
&( \texttt{deliver}(l)_t \rightarrow \texttt{robot}\textrm{-}\texttt{at}(l)_t, \infty)
\end{align*}

\xhdr{Deterministic Action Effects:} (Fluents that become true at timestep $t+1$ if the action executed at timestep $t$)
\begin{align*}
& ( \texttt{deliver}(l)_t \rightarrow \texttt{package}\textrm{-}\texttt{delivered}_{t+1}, \infty)
\end{align*}

\xhdr{Probabilistic Action Effects:} (Fluents that become true at timestep $t+1$ with a certain probability if the action executed at timestep $t$)
\begin{align*}
&(\texttt{move}(l,l',c)_t \wedge \texttt{crowded}(c) \rightarrow \texttt{robot}\textrm{-}\texttt{at}(l')_{t+1} \wedge \lnot \texttt{robot}\textrm{-}\texttt{at}(l)_{t+1}, w_3) \\
&(\texttt{move}(l,l',c)_t \wedge \texttt{crowded}(c) \rightarrow \texttt{robot}\textrm{-}\texttt{at}(l)_{t+1}, w_4) \\
&(\texttt{move}(l,l',c)_t \wedge \lnot \texttt{crowded}(c) \rightarrow \texttt{robot}\textrm{-}\texttt{at}(l')_{t+1} \wedge \lnot \texttt{robot}\textrm{-}\texttt{at}(l)_{t+1}, w_5) \\
&(\texttt{move}(l,l',c)_t \wedge \lnot \texttt{crowded}(c) \rightarrow \texttt{robot}\textrm{-}\texttt{at}(l)_{t+1}, w_6) 
\end{align*}

\xhdr{Explanatory Frame Axioms:} (Fluents  do  not  change  between subsequent timesteps $t$ and $t+1$ unless they are effects of actions that are executed at timestep $t$)
\begin{align*}
&(\texttt{robot}\textrm{-}\texttt{at}(l)_t \wedge \lnot \texttt{robot}\textrm{-}\texttt{at}(l)_{t+1} \rightarrow \texttt{move}(l,l',A)_t \vee \texttt{move}(room1,room2,B)_t, \infty)\\
&(\lnot \texttt{robot}\textrm{-}\texttt{at}(l)_t \wedge \texttt{robot}\textrm{-}\texttt{at}(l)_{t+1} \rightarrow \texttt{move}(l',l,A)_t \vee \texttt{move}(l',l,B)_t, \infty)\\
&(\texttt{package}\textrm{-}\texttt{delivered}_t \rightarrow \texttt{package}\textrm{-}\texttt{delivered}_{t+1}, \infty) \\
&(\lnot \texttt{package}\textrm{-}\texttt{delivered}_t \wedge \texttt{package}\textrm{-}\texttt{delivered}_{t+1} \rightarrow \texttt{deliver}(l)_t, \infty)
\end{align*}

\xhdr{Action exclusions:} (Only one action can occur at each timestep $t$)
\begin{align*}
&(\neg \texttt{move}(l,l',A)_t \vee \neg \texttt{move}(l,l',B)_t, \infty)\\
&(\neg \texttt{move}(l,l',c)_t \vee \neg \texttt{deliver}(l)_t, \infty)
\end{align*}

Finally, we can extract a plan by finding an assignment of truth values that satisfies the belief base (i.e.,~for all timesteps $t= 0,...,n-1$, there will be exactly one action $a$ such that $a_t=True$).

Note that a formula's weight represents the log odds comparing worlds where the formula is true versus where it is false, assuming all else remains unchanged. However, when formulae share variables with each other, it is impossible to flip one formula's truth value without affecting others. Despite that, we can still determine appropriate weights through collective formula probabilities. By treating desired probabilities as empirical frequencies for maximum likelihood estimation (MLE)~\cite{della2002inducing}, we can derive optimal weights. This approach involves specifying how often each formula should be true, using these as observed frequencies, and applying standard MLE algorithms (e.g.,~\cite{richardson2006markov}) to compute weights.

\begin{lstlisting}[
  float=!h,
  caption={PPDDL Domain of Office Robot Delivery.},
  label={lst:domain},
  language=PDDL]
(define (domain Office-Robot-Delivery)
      (:requirements :strips :typing :probabilistic-effects)
      (:types
        location corridor)
      (:predicates
        (robot-at ?l - location)
        (connected ?l1 ?l2 - location ?c - corridor)
        (package-delivered)
        (crowded ?c - corridor)
      )
      
      (:action move
        :parameters (?from ?to - location ?c - corridor)
        :precondition (and 
                        (robot-at ?from) 
                        (connected ?from ?to ?c))
        :effect (and 
                  (when (crowded ?c)
                    (probabilistic 
                      p (and (robot-at ?to) (not (robot-at ?from)))
                      1-p (and (robot-at ?from))))
                  (when (not (crowded ?c))
                    (probabilistic 
                      p' (and (robot-at ?to) (not (robot-at ?from)))
                      1-p' (and (robot-at ?from)))))
      )
      
      (:action deliver
        :parameters (?l - location)
        :precondition (robot-at ?l)
        :effect (package-delivered)
      )
)
\end{lstlisting}

\begin{lstlisting}[
  float=!h,
  caption={PPDDL Problem of Office Robot Delivery.},
  label={lst:problem},
  language=PDDL]
  (define (problem office-delivery)
      (:domain Office-Robot-Delivery)
      
      (:objects
        room1 room2 - location
        A B - corridor)
        
      (:init
        (robot-at room1)
        (connected room1 room2 A)
        (connected room1 room2 B)
        (probabilistic p (crowded A))
        (probabilistic p' (crowded B))
      )
      
      (:goal (package-delivered room2))
)
\end{lstlisting}

\clearpage

\bibliography{sample}
\bibliographystyle{theapa}

\end{document}